\documentclass[11pt]{article}
\usepackage{graphicx}
\usepackage{enumerate}
\usepackage{natbib}
\usepackage{url} %
\usepackage[margin=1in]{geometry}
\usepackage{booktabs,colortbl}
\usepackage{amsmath, amssymb, outlines,mathtools,subcaption,shortcuts,amsthm,xcolor}
\usepackage{enumitem}
\usepackage{multirow,nicefrac}
\usepackage[capitalise]{cleveref}
\Crefname{assumption}{Assumption}{Assumptions}

\newtheorem{theorem}{Theorem}[section]
\newtheorem{corollary}{Corollary}[section]
\newtheorem{lemma}{Lemma}[section]
\newtheorem{proposition}{Proposition}[section]
\newtheorem{assumption}{Assumption}
\newtheorem{definition}{Definition}

\newcommand{\pr}{\mathbb{P}}

\newcommand{\expect}{\mathbb{E}}

\newcommand{\expectnk}{\hat{\expect}_{k}}

\newcommand{\hexpect}{\hat{\expect}}
\newcommand{\score}{\text{SC}}

\newcommand{\na}{\mathsf{NA}}

\newcommand{\et}{e^*}

\newcommand{\eo}{e_0}
\newcommand{\he}{\hat{e}}
\newcommand{\hek}{\hat{e}_k}
\newcommand{\rt}{r^*}
\newcommand{\ro}{r_0}
\newcommand{\hr}{\hat{r}}
\newcommand{\hrk}{\hat{r}_k}
\newcommand{\mut}{\mu^*}
\newcommand{\muo}{\mu_0}
\newcommand{\hmu}{\hat{\mu}}
\newcommand{\hmuk}{\hat{\mu}_k}
\newcommand{\mutb}{{\mu}^{*2}}
\newcommand{\tmut}{\tilde{\mu}^*}
\newcommand{\tmuo}{\tilde{\mu}_0}
\newcommand{\htmu}{\hat{\tilde{\mu}}}
\newcommand{\htmuk}{\hat{\tilde{\mu}}_k}
\newcommand{\tmutb}{\tilde{\mu}^{*2}}
\newcommand{\tmu}{\tilde{\mu}}

\newcommand{\rhoe}{\rho_{N, e}}
\newcommand{\rhor}{\rho_{N, r}}
\newcommand{\rhomu}{\rho_{N, \mu}}
\newcommand{\rhotmu}{\rho_{N, \tilde{\mu}}}

\newcommand{\ON}{\overline{N}}

\allowdisplaybreaks

\newcommand*\samethanks[1][\value{footnote}]{\footnotemark[#1]}
\title{On the Role of Surrogates in the Efficient Estimation of Treatment Effects with Limited Outcome Data}
\author{Nathan Kallus$^1$\thanks{Alphabetical order.}, 
\quad Xiaojie Mao$^2$\samethanks}
\date{\small
$^1$ Cornell Tech, Cornell University, New York, NY 10044, USA; \\
$^2$ School of Economics and Management, Tsinghua University, Beijing, 100084, China.
}

\begin{document}
\maketitle

\begin{abstract}
In many experimental and observational studies, the outcome of interest is often difficult or expensive to observe, reducing effective sample sizes for estimating average treatment effects (ATEs) even when identifiable. 
{We study how incorporating data on units for which only surrogate outcomes not of primary interest are observed can increase the precision of ATE estimation.
We refrain from imposing stringent surrogacy conditions, which permit surrogates as perfect replacements for the target outcome.
Instead, we supplement the available, albeit limited, observations of the target outcome  with abundant observations of surrogate outcomes, without any assumptions beyond unconfounded treatment assignment and missingness and corresponding overlap conditions.}
To quantify the potential gains, we derive the difference in efficiency bounds on ATE estimation with and without surrogates, both when an overwhelming or comparable number of units have missing outcomes. We develop robust ATE estimation and inference methods that realize these efficiency gains. We empirically demonstrate the gains by studying long-term-earning effects of job training.
\end{abstract}

\noindent%

\section{Introduction}
In many causal inference applications, it may be expensive, inconvenient or infeasible to measure the outcome of primary interest. 
Nevertheless, some auxiliary variables that are faster or easier to measure may be available. 
In clinical trials for AIDS treatment, the primary outcome is often mortality, which may take years of follow-up to fully reveal. 
But clinically relevant biomarkers like viral loads or CD4 counts can be measured quite rapidly \citep{fleming1994surrogate}.  
In comparative effectiveness research for long-term impact of therapies, e.g., long-term quality of life measures, many patients may drop-out so their responses are missing. But short-term outcomes, e.g., responses shortly after the therapy, may be well recorded \citep{post2010analysis}. 
In program evaluation for addiction prevention projects, accurately measuring the primary outcome, e.g., smoking behavior, may require costly chemical analysis of saliva samples for the presence of cotinine, and thus are available for only a limited number of participants. Yet self-report data are relatively inexpensive to collect \citep{pepe1992inference}. 
In offline conversion analysis, we wish to assess the effect of a digital marketing campaign on visitation to a brick-and-mortar location. And, while we can only observe visitation for individuals for whom we have cellphone geolocation data and who we can match to ad identifiers, we can observe digital ad clicks for all units.
We refer to these easy-to-obtain auxiliary variables as {surrogate outcomes} or simply {surrogates}, which are often informative about or correlate with the {primary outcome} of interest.

There has been considerable interest in using surrogates as a {replacement} for the missing primary outcome to reduce data collection costs in causal inference. For example, the U.S. Food and Drug Administration (FDA) launched the Accelerated Approval Program to allow for early approval of drugs based on clinically relevant surrogates, aiming to speed up clinical trials for drug approval \citep{FDA2016}. 
This program is spurred by the urgent need to determine the efficacy of new drugs quickly and economically. 
As stated by the National Center for Advancing Translational Sciences (NCATS) of the U.S. National Institutes of Health, many thousands of diseases known to affect humans do not have any approved treatment yet;  meanwhile, a novel drug can ``take 10 to 15 years and more than \$2 billion to develop'' \citep{NIH2019}.
Therefore, accelerating drug approval is of great value and urgency to both pharmaceutical companies and patients.
Using surrogates that can be measured more easily provides a promising way toward this goal.

However, one major challenge is that surrogates may not be perfectly indicative of the primary outcome, so a misuse may lead to severe or even disastrous consequences. 
For example, three drugs (encainide, flecainide, and moricizine) were approved by FDA based on early success of supressing ventricular arrhythmia (surrogate), but in later follow-up trials the drugs alarmingly {increased} mortality (primary outcome) \citep{fleming1996surrogate,echt1991mortality}. 
To resolve these problems, a wide variety of surrogacy criteria have been proposed to ensure that it is adequate to base causal inference solely on the surrogate without observations of the primary outcome.
 However, using these criteria to search for  valid surrogates  is still extremely challenging, since the criteria impose stringent assumptions that may 
 {often} be 
 violated in practice (see Related Literature in \Cref{sec: literature} and \Cref{sec: criteria}). For example, the popular statistical surrogacy condition \citep{prentice1989surrogate,athey2016estimating} requires the primary outcome to be conditionally independent of the treatment given surrogates, i.e., surrogates must fully explain away the dependence of the outcome on the intervention meant to affect it. 
Not only does this condition require full mediation of the treatment effect, it is also easily invalidated if there is any other common cause of both surrogates and the primary outcome, which may often be unavoidable even in ideal randomized trials (see \Cref{sec: criteria}). 
Thus this surrogacy condition and similarly many other criteria are very 
prone to violation
in practice. 

{In this paper, we refrain from imposing such surrogacy conditions. Consequently, the surrogate outcomes 
alone are insufficient as complete replacements for the target outcome.
Nonetheless, we continue to refer to them as surrogates as our proposed method does use them to predict the target outcome. Specifically, our paper views these outcomes as imperfect surrogates and uses them as {supplements}, rather than {replacements}, for the primary outcome.
}%
  We  consider {combining} surrogates with the primary outcome, and investigate how this proposal can improve the {efficiency} of treatment effect estimation. 
Such combination is possible because in practice {paired} observations of both the primary outcome and surrogates are often available for at least some units.
By incorporating a limited number of primary-outcome observations, we can avoid the aforementioned problems resulting from relying on surrogates alone, and circumvent stringent surrogacy conditions. 
Instead, we only assume standard causal inference assumptions and a typical missing data assumption that the primary outcome is {missing (conditionally) at random} (MAR), i.e., any interdependence between the primary outcome value and whether it is observed or not may be explained by other observed variables (i.e., pre-treatment covariates, treatment, surrogates).
Similar missingness conditions are also commonly assumed in previous literature that combine different datasets \cite[e.g.,][]{athey2016estimating,cheng2018efficient,zhang2019high}. 
Under only these standard assumptions, and in particular no overly restrictive surrogacy conditions,
we aim to investigate the role of surrogates in estimating treatment effects when the primary-outcome observations are limited. 

We first study the possible extent of benefit achievable by leveraging surrogate information.
Using the theory of semiparametric efficiency {\cite[e.g.,][]{bickel1993efficient,robins1995semiparametric,tsiatis2007semiparametric}},  
we derive the efficiency lower bound of estimating the average treatment effect (ATE) on the {primary} outcome (\Cref{thm: efficient-if}). This lower bound characterizes the fundamental statistical limit in estimating ATE under our assumptions, in that it is the {best} possible precision of ATE estimation that can be asymptotically achieved by {any} regular estimator.
By comparing the efficiency lower bounds both with and without the presence of surrogates, and bounds in several intermediary settings (\Cref{thm: efficiency-comparison,corollary: efficiency-loss}), we precisely quantify the efficiency gains from surrogates, namely, the benefit of surrogates in terms of allowing us to estimate treatment effects up to the {same} precision with {fewer} observations of the primary outcome. 

We find that using surrogates is particularly advantageous when (i) the primary outcome is missing for a large number of units, and (ii) the surrogates are reasonably {predictive} of the primary outcome, in that they can account for large variations of the primary outcome, but they need not determine them exactly or render them independent of treatment.
These theoretical results provide insightful guidelines for understanding when surrogates can yield significant benefits.
Moreover, we show that essentially the same efficiency lower bound (under appropriate reformulation) reigns across two different regimes: when the size of the unlabeled data is comparable to the size of the labeled data (\Cref{thm: efficient-if}), and when the former is much larger than the latter (\Cref{thm: vanishing-if,thm: normality2}). 
In the second regime, the commonly assumed overlap condition (\Cref{assump: overlap} condition \eqref{condition: overlap-2}) fails and the efficiency analysis under MAR setting becomes more challenging.
Our paper tackles this very practical setting when enormous amounts of cheap unlabeled data may be available.

We further propose an ATE estimator that can optimally leverage the efficiency gains from surrogates and achieves the efficiency lower bound. 
The proposed estimator involves some nuisance parameters that are of no intrinsic interest but need to be estimated first. By employing a cross-fitting technique \citep[e.g.,][]{chernozhukov2018double,zheng2011cross}, we can allow for any flexible machine learning estimators to be used for the nuisance parameters as long as they satisfy some generic convergence rate conditions.
We show that the proposed estimator converges to the true ATE value, even if only some but not all nuisance parameters are consistently estimated (\Cref{thm: DR}). 
If all nuisance parameters are indeed consistently estimated under generic rate conditions, then the proposed estimator is asymptotically normal centered at the true ATE value and its asymptotic variance attains the efficiency lower bound (\Cref{thm: normality,thm: normality2}).
Furthermore, we construct asymptotically valid confidence intervals based on a simple plug-in estimator for the asymptotic variance of our ATE estimator (\Cref{thm: conf-interval}). 
In summary, we propose an ATE estimator that can leverage the power of {flexible} machine learning estimators for nuisance estimation, is {robust} to nuisance estimation errors, achieves full asymptotic {efficiency} in leveraging surrogate information, and may be combined with {easy-to-use} inference.

Our paper is organized as follows. In \Cref{sec: setup}, we set up the problem of treatment effect estimation with surrogates when the primary outcome is not fully observed, and introduce our notation. 
In \Cref{sec: efficiency}, we derive the efficiency lower bound for ATE estimation in our setting, and compare it with bounds in other benchmark settings to characterize the efficiency gains from surrogates. 
We then construct an asymptotically efficient estimator and prove its asymptotic properties in \Cref{sec: estimation}. 
In \Cref{sec: extension}, we extend the efficiency and estimation results to the setting where the amount of unlabeled data is much larger than amount of labeled data. 
{In \Cref{sec: empirics}, we use our methods to study the effect of a job training intervention on earnings at a later follow up using data from a large-scale randomized controlled trial \citep{hotz2006evaluating,athey2016estimating} and we demonstrate the gains due to employing surrogates and due to our methods.}
We provide concluding remarks in \Cref{sec: conclusion}. In \Cref{sec: connection,sec: criteria,sec: more-extension}, we provide supplementary discussions about the statistical surrogacy condition, the connection of our work to some previous literature, and additional details that 
expand on our results from \Cref{sec: extension}, respectively.

\subsection{Problem Setup}\label{sec: setup}
Let $T \in \{0, 1\}$ denote a treatment indicator variable  (i.e., $T = 1$ means being treated with a therapy of interest, and $T = 0$ means control), $X  \in \mathcal{X} \subseteq \mathbb{R}^{d_x}$ denote baseline covariates  measured prior to treatment (e.g., patients' demographic characteristics and health measurements before treatment), and $Y \in \mathbb{R}$ denotes the  outcome variable of primary interest (e.g., patients' health outcome after treatment). Following the Neyman-Rubin potential outcome framework \citep{neyman1923applications,rubin2005causal}, we assume the existence of two potential outcomes $Y(1), Y(0)$ corresponding to the outcomes that would have been realized under each treatment option. We assume that the actual observed outcome is the potential outcome corresponding to the actual treatment, i.e., $Y = Y(T)$, which encapsulates the non-interference and
consistency assumptions in causal inference \citep{imbens2015causal}. Our goal is to estimate the {average treatment effect} (ATE):
\begin{align}\label{eq: ATE}
\delta^* = \xi_1^* - \xi_0^*,\quad\text{where}\quad \xi_t^*= \expect[Y(t)]~\text{for}~t=0,1. 
\end{align}
{If} we could observe $(X, T, Y)$ for all units, then we could estimate the ATE  by many existing methods \citep[e.g., ][]{imbens2015causal}.

In this paper, we consider 
a more challenging 
setting where the {primary outcome} $Y$ {cannot} be observed for all units, due to long follow-up, drop-out, budget constraints, etc. Nonetheless, we can observe for all units some {surrogates} $S \in \mathcal{S} \subseteq \mathbb{R}^{d_s}$ (i.e., intermediate outcomes) that may be informative about the primary outcome $Y$ (i.e., a long-term outcome). Since surrogates are measured after the treatment assignment, they may also be affected by the treatment. Thus we hypothesize the existence of two potential surrogate outcomes $S(1), S(0)$ analogously, and assume $S = S(T)$. We use $R \in \{0, 1\}$ to denote the indicator of whether the primary outcome $Y$ is observed. 

In summary, we can observe a {labeled} subset $\{(X_i, T_i, S_i, Y_i, R_i = 1): i \in \mathcal{I}^l\}$, and an {unlabeled} subset $\{(X_i, T_i, S_i, Y_i = \na, R_i = 0): i\textbf{} \in \mathcal{I}^u\}$, where $\na$ stands for ``not available'' (missing value), and $\mathcal{I}^l$ and $\mathcal{I}^u$ are the index sets for labeled data and unlabeled data respectively. We denote $N_l = |\mathcal{I}^l|$, $N_u = |\mathcal{I}^u|$ as the corresponding sample sizes for these two datasets, and $N = N_l + N_u$ as the total sample size. We represent the $i$th data point as $W_i = (X_i, T_i, S_i, Y_i, R_i)$, and assume that each data point in both datasets is given by coarsening an independent and identically distributed (i.i.d.) draw from a population $W^*=(X, T, S(0), S(1), Y(0), Y(1), R)$ whose distribution is characterized by a probability measure $\pr^*$. In particular, the coarsening map is given by 
\begin{equation}\label{eq:coarsening}\mathcal C:(X, T, S(0), S(1), Y(0), Y(1), R)\mapsto(X,T,S(T),R\times Y(T)+(1-R)\times(\na),R).\end{equation} We let $\pr$ denote the distribution on $W=\mathcal C(W^*)$ induced by $\pr^*$. Depending on the context, 
we may use $\E$ to denote expectation with respect to either $\pr^*$ or $\pr$. {The addition and multiplication operation involving a missing value ``$\na$'' in \Cref{eq:coarsening} can be understood as regular arithmic operations. For example, $R\times Y(T)+(1-R)\times(\na)$ equals $Y(T)$ if $R = 1$ and $\na$ if $R = 0$. }

\begin{assumption}[Unconfoundedness]\label{assump: unconfound}
For $t = 0, 1$,
\begin{align}
(Y(t), S(t)) \perp T &\mid X \label{condition: unconfound-1}.
\end{align}
\end{assumption}

\Cref{assump: unconfound} assumes that the treatment assignment is unconfounded in the combined population of labelled and unlabelled data. 
In \Cref{lemma: MAR-implication}, we will show that this condition is guaranteed if the treatment is unconfounded in {both} the labelled and unlabelled  subpopulation, {separately}, when additional missing-at-random assumptions are imposed. 
This condition requires that $X$ include all confounders that can affect the primary outcome and treatment simultaneously, or the surrogate and treatment simultaneously.
It is trivially satisfied by design in clinical trials where the treatment $T$ is assigned totally at random.

\begin{assumption}[Missing at random]\label{assump: mar-1}
For $t = 0, 1$, 
\begin{align}\label{condition: unconfound-2}
R \perp Y(t) \mid X, S(t), T.
\end{align}
\end{assumption}

In \Cref{assump: mar-1}, we assume that the primary outcome is {missing (conditionally) at random} (MAR), i.e., the indicator $R$ depends on only observed variables, including pre-treatment covariates $X$, the surrogates $S$, and the treatment $T$. 
This condition guarantees that the distribution of the primary outcome on the labeled data and unlabeled data are comparable after accounting for the observed variables, so that we can use the labeled data to infer information about the missing primary outcome in the unlabeled data.  
This condition is considerably weaker than the  {missing completely at random} (MCAR) condition typically assumed in previous semi-supervised inference literature \citep[e.g.,][]{cheng2018efficient,zhang2019high}, since MCAR does {not} allow the missingness of the primary outcome to depend on {any} other variable. 
\Cref{assump: mar-1} may be satisfied by design in a two-phase sampling scheme \citep[e.g.,][]{wang2009causal,cochran2007sampling}: in the first phase, relatively cheap measurements of $T, X, S$ are available for all units, and in the second phase, expensive measurements of the primary outcome $Y$ are collected for a validation subsample selected according to variables measured in the first phase. 
For example, we may want to oversample units who self-report no-smoking behavior for further chemical analysis, if we suspect more misreporting in this subpopulation.

We next define some important quantities for ATE estimation. 
We first define the regression function $\tmut$ of the primary outcome in the labeled dataset, conditional on treatment, covariates, and surrogates, and also the projection of $\tmut$ onto the whole population, conditional on only treatment and covariates:
\begin{align}\tmut(t, x, s) &= \expect[Y \mid T = t, X = x, S = s, R = 1] \label{eq: nuisance-mu-tilde},\\\mut(t, x) &= \expect[\tmut(T, X, S) \mid T = t, X = x].\label{eq: nuisance-mu}
\end{align}

We also define the propensity scores for treatment and labeling: 
\begin{align}\label{eq: propensity}
\begin{array}{c}
\et(x) = \pr(T = 1 \mid X = x), ~~ \et(x, s)= \pr(T = 1 \mid X = x, S = s),  \\ 
\rt(t, x, s) = \pr(R = 1 \mid T = t, X = x, S = s).
\end{array}
\end{align}
Although these quantities are useful for estimating the ATE, they are of no intrinsic interest by themselves, so we refer to them as {nuisance parameters}. We let $\eta^* = (\et, \rt, \tmut, \mut)$ be the collection of the {true} nuisances. We further assume the following overlap condition. 

\begin{assumption}[Strict Overlap]
\label{assump: overlap}
There exist $\epsilon \in (0, 1/2)$ such that almost surely we have 
\begin{align}
\epsilon &\le \et(X, S) \le 1 - \epsilon,  \label{condition: overlap-1} \\
\epsilon & \le  \rt(T, X, S) \le  1.
\label{condition: overlap-2}
\end{align}
\end{assumption}

This assumption states that units with any given values of the conditioning variables above have at least probability of $\epsilon$ to receive each treatment option, and to get their primary outcome measured.  
This overlap assumption is very common in causal inference and missing data literature \citep[e.g.,][]{imbens2015causal,little2019statistical}. 
Note condition \eqref{condition: overlap-2} implies that $\pr(R = 1) \ge \epsilon$, so the  unlabeled and labelled  data  necessarily have comparable sizes, i.e., $N_u \asymp N_l$ (unless all data is labeled).
In \Cref{sec: extension}, we will relax this condition and consider the setting where enormous cheap unlabeled data are available so that $N_u \gg N_l$. 

Below we show identification of the the  ATE parameter $\delta^*$.
\begin{lemma}\label{lemma: identification-1}
If \Cref{assump: mar-1,assump: unconfound,assump: overlap} hold, then 
\begin{align}\label{eq: identif-0}
\begin{aligned}
    \delta^* 
&= \Eb{\Eb{\Eb{Y \mid T = 1, R = 1, X, S}\mid X, T = 1}} \\
&\qquad\qquad\qquad\qquad - \Eb{\Eb{\Eb{Y \mid T = 0, R = 1, X, S}\mid X, T = 0}}.
\end{aligned}
\end{align}
\end{lemma}

In this paper, we focus on the efficient estimation of ATE  (i.e., $\delta^*$ in \cref{eq: ATE}) when the primary outcome $Y$ is missing for many units while surrogates $S$ can be fully observed for all. Notably, we only assume \Cref{assump: mar-1,assump: unconfound,assump: overlap} (and some straightforward variants) that are very typical in causal inference and missing data literature. 
In particular, we do not assume any strong surrogacy conditions such as the statistical surrogacy condition, $Y \perp T \mid S, X, R= 1$, which may impose  restrictions that can easily be violated in practice (see \Cref{sec: literature} and \Cref{sec: criteria} for more discussion).

\paragraph*{Notation.} We use $O, o, O_p, o_p$ to denote the nonstochastic and stochastic asymptotic orders, respectively. For nonstochastic sequences $a_N\geq0$ and $b_N>0$, $a_N = O(b_N)$ if $\limsup_{N \to \infty} a_N/b_N < \infty$ and $a_N = o(b_N)$ if $\lim_{N \to \infty} a_N/b_N = 0$. For a random variable sequence $Z_N$, we denote $Z_N = O_p(a_N)$ if for any positive constant $\varepsilon$, there exists a finite positive constant $M$ such that $\pr(|Z_N/a_N| > M) < \varepsilon$, and we denote $Z_N = o_p(a_N)$ if for any positive constant $\varepsilon$, $\pr(|Z_N/a_N| > \varepsilon) \to 0$ as $N \to \infty$.
We also use the notation $\asymp$ and $\gg, \ll$ for asymptotic orders (both stochastic and nonstochastic). For example, for nonstochastic asymptotic order, $a_N \asymp b_N$ if $a_N = O(b_N)$ and $b_N = O(a_N)$, $a_N \gg b_N$ if $b_N/a_N = o(1)$, and $a_N \ll b_N$ if $a_N/b_N = o(1)$.
For an appropriately measurable and integrable function $f$, we use $\|f\|$, $\|f\|_p$, $\|f\|_{\infty}$ to denote the $L_2$, $L_p$ and $L_\infty$ norms with respective to the measure $\pr$: $\|f\| = \left\{\expect\left[f^2(W)\right]\right\}^{1/2}$, $\|f\|_p = \left\{\expect\left[|f(W)|^p\right]\right\}^{1/p}$, and $\|f\|_{\infty} = \inf\{c \ge 0: \pr(|f(W)| \le c) = 1\}$.
 Throughout this paper, we use $\phantom\cdot^*$ to denote unknown population quantities like $\delta^*$ and $\eta^*$, and use $\hat{\phantom\cdot}$   denote estimators, i.e., $\hat{\delta}$.

\subsection{Related Literature}\label{sec: literature}
\paragraph*{Causal inference with surrogates.} Many different surrogate criteria have been proposed to ensure that the treatment effect  on a surrogate will reliably predict the treatment effect on the primary outcome. 
The statistical surrogacy criterion proposed by \cite{prentice1989surrogate} was the first such criterion, which requires the primary outcome to be conditionally independent of the treatment, given the surrogate. 
Since then, many other criteria have been proposed, such as the {principal surrogate} criterion \citep{frangakis2002principal}, {strong surrogate}
criterion \citep{lauritzen2004discussion}, {consistent surrogate} criterion \citep{chen2007criteria}, among many others. 
However, almost all of these criteria involve unidentifiable quantities, so they are unverifiable in practice. 
Moreover, many of them can easily run into a logical paradox  described by 
\cite{chen2007criteria}.
See \cite{vanderweele2013surrogate} for a comprehensive review of surrogate criteria and \Cref{sec: criteria} for a detailed discussion about the statistical surrogacy condition.

While the literature above mostly focus on a {single} surrogate,
\cite{price2018estimation,wang2020model} propose to estimate  transformations of multiple surrogates to optimally approximate the primary outcome using labelled experimental data.
Their optimal transformations can avoid the surrogate paradox described in \cite{chen2007criteria}. 
\cite{athey2016estimating}  consider identifying and estimating the average treatment effect with multiple surrogates in the setting where the primary outcome cannot be observed simultaneously with treatment variable and are instead observed in two separate datasets, connected only by the surrogates and covariates. 
This setting is practically very challenging, since the two datasets have no complete observations at all, with the primary outcome missing in one dataset and treatment missing in the other. 
To fuse these two incomplete datasets and have hope of relating the effect of treatments on downstream outcomes, they have to assume the {statistical surrogacy condition}, which, however, may be too strong in practice (see \cref{sec: criteria}).
\cite{athey2020combining,imbens2022long} use surrogates to combine experimental data with short-term observations and confounded observational data with long-term observations, the former using a latent unconfoundedness assumption and latter using multiple sequential surrogates as proxy variables. 
\cite{chen2021semiparametric} further study semiparametric inference of the average treatment effect in the settings of \cite{athey2016estimating,athey2020combining}. However, these works and our paper use different assumptions. In particular, these works leverage surrogates for identification under either the statistical surrogacy condition \citep{athey2016estimating} or the latent unconfoundedness condition \citep{athey2020combining}. 
In contrast, our work focuses on leveraging surrogates to improve efficiency in already-identified settings under missing-at-random assumptions. In \Cref{sec: empirics}, we consider an empirical study where the statistical surrogacy condition is very likely to fail, and the corresponding estimators in \cite{athey2016estimating,chen2021semiparametric} have high bias. In \Cref{sec: aside}, we further  compare our assumptions with the identification assumptions in \cite{athey2020combining}.

\cite{cheng2018efficient} study efficient ATE estimation when combining a small number of primary-outcome observations with many observations of the surrogates, without assuming any surrogate criteria like those mentioned above. 
Their setting is closest to ours, except that they focus on the case when the unlabeled dataset is much larger than the labeled dataset, i.e., $N_u \gg N_l$ and they assume that the primary outcome is MCAR. In contrast, our paper studies both $N_u \gg N_l$ and $N_u \asymp N_l$ and considers a more general MAR setting.
By studying both $N_u \gg N_l$ and $N_u \asymp N_l$, we discover that essentially the same efficiency lower bound governs both regimes. 
Moreover, \cite{cheng2018efficient} consider certain specialized estimators based on parametric regressions and kernel smoothing, 
while 
our proposed estimator can leverage flexible machine learning nuisance estimation.
See \Cref{sec: connection} for a more detailed comparison of our work with \cite{cheng2018efficient}.

Our paper studies a missing data setting where the primary outcome is either observed or completely missing, following many previous literature \citep[e.g., ][]{cheng2018efficient,athey2016estimating,wang2020model,price2018estimation}. This is different from the censored data setting in some surrogate literature \citep[e.g., ][]{prentice1989surrogate,lin1997estimating,ghosh2008semiparametric,parast2017evaluating}. In the latter literature, the primary outcome is typically a time-to-event outcome subject to right (or interval) censoring. So even when the primary outcome is not perfectly observed, we at least know a range of its value. Since the primary outcome is not perfectly observed, additional surrogate observations can also be beneficial. It is interesting to extend our results to this important censored data setting in the future.

\paragraph*{Semi-supervised inference.} Our paper is related to the growing body of  literature on parameter estimation and inference in the  semi-supervised setting where a small labeled dataset is enriched with a large unlabeled dataset. 
A stream of research has investigated how to the use unlabeled data to aid in the estimation of a wide variety of parameters, including regression coefficients \citep{azriel2016semi,chakrabortty2018efficient,hou2021surrogate}, population mean and average treatment effect \citep{zhang2019semi,zhang2019high,chakrabortty2022general,zhang2021double}, 
quantiles and quantile treatment effect \citep{chakrabortty2022semi,chakrabortty2022general}, etc. 
Nearly all of these literature implicitly or explicitly assume that labels are MCAR. 
Our paper relaxes this assumption by allowing the labeling process to depend on pre-treatment covariates, the treatment, and even the surrogates. Moreover, while we consider partially labelled outcomes, we also focus on the use of surrogates as a source of extra information.
Interestingly, when viewing the surrogates in our paper as empty, our results also recover results in existing semi-supervised inference literature, for example, \cite{zhang2019high}. See \Cref{sec: connection} for details.  

\paragraph*{Measurement error problems with a validation sample.} 
We can also view our problem as a measurement error problem: abundant  mismeasurements of the primary outcome (i.e., the surrogate observations) are available, while accurate measurements (i.e., the primary outcomes itself) are  observed only on a small {validation sample} (i.e., the labeled dataset). 
In similar settings, many methods have been proposed to leverage observations with measurement noise to improve the efficiency of estimating regression coefficients \citep[e.g., ][]{pepe1994auxiliary,pepe1992inference,reilly1995mean,engel1991increasing,carroll1991semiparametric,chen2000unified} or solutions to estimation equations \citep[e.g.,][]{chen2008improving,chen2003information,chen2005measurement,chen2008semiparametric}.
Some literature also cast this type of problem as a missing data problem where the variables of primary interest are missing for all units not in the validation sample \citep[e.g.,][]{yu2006revisit,chen2004semiparametric}.
Our paper builds on the missing data framework to study the efficiency of estimating treatment effects in presence of surrogates. Thus our paper is closely related to the broader literature on semiparametric inference with missing data or more general data coarsening \citep{robins1995semiparametric,robins1994estimation,vanderLaan2003,tsiatis2007semiparametric}.
In contrast to the missing data literature that commonly assume the proportion of complete observations to be bounded away from $0$, our paper allows the complete-case proportion to vanish to $0$ in order to model the setting with enormous amounts of unlabeled surrogate data.

\section{Efficiency Analysis}\label{sec: efficiency}

In this section we derive the efficiency lower bounds and efficient influence functions in a sequence of models ranging from no surrogate observations to full outcome observations on all data points, crucially including our primary setting of interest as a practical middle ground (see \cref{table: setting3}). This serves to quantify both the information gain from surrogate observations relative to no surrogate observations and the gap remaining relative to full outcome observations. 

\subsection{Efficiency Analysis in the Presence of Surrogates}\label{sec: efficiency main setting}

We first derive the semiparametric efficiency lower bound for ATE estimation in our primary setting of interest as described in \cref{sec: setup}. 

\begin{theorem}\label{thm: efficient-if}
Let $\mathcal M$ be the set of {all} distributions $\pr$ on $W$ induced by the coarsening map $\mathcal C$ in \cref{eq:coarsening} applied to any distribution $\pr^*$ on $W^*$ satisfying \cref{assump: unconfound,assump: overlap,assump: mar-1}. 
The semiparametric efficiency lower bound for $\delta^*$ under  model $\mathcal M$ is $V^* = \expect[\psi^2(W; \delta^*, \eta^*) ]$ where 
\begin{align}\label{eq: efficient-IF}
        &\psi(W; \delta^*, \eta^*) =   \mut(1, X) - \mut(0, X) - \delta^* + \frac{T - \et(X)}{\et(X)\prns{1-\et(X)}}(\tmut(T, X, S) - \mut(T , X)) \nonumber \\
 &\quad+ \frac{TR}{\et(X)\rt(1, X, S)}(Y - \tmut(1, X, S)) - \frac{(1 - T)R}{(1 - \et(X))\rt(0, X, S)}(Y - \tmut(0, X, S)) 
\end{align}
{Moreover, the efficiency bound remains the same if either of $\et$ or $\rt$ or both are known.}
\end{theorem}

\Cref{thm: efficient-if} reveals the fundamental statistical limit in estimating $\delta^*$ with surrogates under {\Cref{assump: unconfound,assump: overlap,assump: mar-1}}: for any regular estimator $\hat{\delta}$, the variance of the limiting distribution of $\sqrt{N} (\hat{\delta} - \delta^*)$ must be no smaller than $V^*$.
In other words, $V^*$ is the best possible precision we can aim to achieve asymptotically among all regular estimators. 
The function $\psi(W; \delta^*, \eta^*)$ is the efficient influence function for $\delta^*$, which will be used to construct efficient estimators for $\delta^*$ in \Cref{sec: estimation}. Notably,  we show that the efficiency bound does not change if the propensity scores are known. This is because the efficient influence function can be shown to be orthogonal to the parts of tangent space corresponding to the propensity scores. 

Notably, the efficiency bound here corresponds to the model that only assumes {\Cref{assump: unconfound,assump: overlap,assump: mar-1}}, but not any strong surrogacy condition.
To study the role of surrogates in the efficient estimation of ATE, 
we next consider the efficiency bound in a few other settings.

\subsection{Efficiency Analysis in Other Settings}\label{sec: efficiency-settings}

\begin{table}
    \begin{subtable}[t]{.24\linewidth}%
    \centering%
    \begin{tabular}{|cccc|c|}
    \hline 
        $X$ & $T$ & $S$ & $Y$ & $R$ \\
    \hline 
        \checkmark & \checkmark & ? & \checkmark & \multirow{3}{*}{$1$}   \\
        \vdots & \vdots & \vdots & \vdots &    \\
        \checkmark & \checkmark & ? & \checkmark &    \\
    \hline 
         \checkmark & \checkmark & ? & ? & \multirow{3}{*}{$0$}  \\
         \vdots & \vdots & \vdots & \vdots &   \\
         \checkmark & \checkmark & ? & ? &   \\
    \hline 
    \end{tabular}
    \caption{Setting I}\label{table: setting1}
  \end{subtable}
  \begin{subtable}[t]{.24\linewidth}%
    \centering%
    \begin{tabular}{|cccc|c|}
    \hline 
        $X$ & $T$ & $S$ & $Y$ & $R$ \\
    \hline 
        \checkmark & \checkmark & \checkmark & \checkmark & \multirow{3}{*}{$1$}   \\
        \vdots & \vdots & \vdots & \vdots &    \\
        \checkmark & \checkmark & \checkmark & \checkmark &    \\
    \hline 
         \checkmark & \checkmark & ? & ? & \multirow{3}{*}{$0$}  \\
         \vdots & \vdots & \vdots & \vdots &   \\
         \checkmark & \checkmark & ? & ? &   \\
    \hline 
    \end{tabular}
    \caption{Setting II}\label{table: setting2}
  \end{subtable}
  \begin{subtable}[t]{.24\linewidth}%
    \centering%
        \begin{tabular}{|cccc|c|}
    \hline 
        $X$ & $T$ & $S$ & $Y$ & $R$ \\
    \hline 
        \checkmark & \checkmark & \checkmark & \checkmark & \multirow{3}{*}{$1$}   \\
        \vdots & \vdots & \vdots & \vdots &    \\
        \checkmark & \checkmark & \checkmark & \checkmark &    \\
    \hline 
         \checkmark & \checkmark & \checkmark & ? & \multirow{3}{*}{$0$}  \\
         \vdots & \vdots & \vdots & \vdots &   \\
         \checkmark & \checkmark & \checkmark & ? &   \\
    \hline 
    \end{tabular}
    \caption{Setting III}\label{table: setting3}
  \end{subtable}
    \begin{subtable}[t]{.24\linewidth}%
    \centering%
    \begin{tabular}{|cccc|c|}
    \hline 
        $X$ & $T$ & $S$ & $Y$ & $R$ \\
    \hline 
        \checkmark & \checkmark & \checkmark & \checkmark & \multirow{3}{*}{$1$}   \\
        \vdots & \vdots & \vdots & \vdots &    \\
        \checkmark & \checkmark & \checkmark & \checkmark &    \\
    \hline
         \checkmark & \checkmark & \checkmark & \checkmark & \multirow{3}{*}{$0$} \\
         \vdots & \vdots & \vdots & \vdots &   \\
         \checkmark & \checkmark & \checkmark & \checkmark &   \\
    \hline 
    \end{tabular}
    \caption{Setting IV}\label{table: setting4}
  \end{subtable} 
\caption{Illustrations for the observed data in Setting I to Setting IV. Here ``\checkmark'' stands for an observed value, and ``?'' stands for a missing value. 
}\label{table: four-settings}
\end{table}

To quantify the benefit of surrogates in estimating ATE, we compare the efficiency lower bounds in following different settings. 
\begin{definition}[Four different settings]\label{def: settings}
\begin{description}
 \item[Setting I: no surrogate (\Cref{table: setting1}).] We observe $(X, T, Y)$ for $R = 1$ and observe $(X, T)$ for $R = 0$;
 \item[Setting II: surrogate only on labeled data (\Cref{table: setting2}). ] We observe $(X, T, S, Y)$ for $R = 1$ and observe $(X, T)$ for $R = 0$;
 \item[Setting III: surrogate on all data (\Cref{table: setting3}).] We observe $(X, T, S, Y)$ for $R = 1$ and observe $(X, T, S)$ for $R = 0$;
 \item[Setting IV: fully labeled data (\Cref{table: setting4}).] We observe $(X, T, S, Y)$ for all units.
 \end{description} 
\end{definition}

From setting I to setting IV in \Cref{def: settings}, more information is increasingly observed. 
 Setting I corresponds to one extreme where no surrogates are observed at all, and setting IV corresponds to the other extreme where all variables (including the primary outcome) are always completely observed. 
In the intermediate setting II, we observe surrogates only for units whose primary outcome is already observed, and setting III corresponds to our primary problem setup in \Cref{sec: setup}, where surrogates are always observed. {Note that the joint distribution of the variables $(X, T, S(1), S(0), Y(1), Y(0), S, Y)$ is taken to be the same in all four settings, even though some of these variables are not fully observed or even entirely missing in some settings. In particular, the functions $\tmut, \mut, \et, \rt$ in \Cref{eq: propensity,eq: nuisance-mu,eq: nuisance-mu-tilde} are well-defined and identical in the four settings.}

Each of these settings can be described by different choices for the coarsening map $\mathcal C$. {For example, the coarsening map for setting I is $\mathcal{C}: W^* \mapsto (X,T, \na, R \times Y (T) + (1 - R) \times (\na), R)$, and the coarsening maps for other settings can be defined analogously.} To compare the efficiency gains of the additional information in each setting, we can consider the efficiency bound corresponding to the {same} $\pr^*$ as {different} coarsening maps $\mathcal C$ are applied, each corresponding to one of the above settings.
Each map induces a model given by the distributions $\pr$ induced by all $\pr^*$ that satisfy \cref{assump: unconfound,assump: overlap,assump: mar-1}.
Crucially, we will need that in each setting we have {identifiability},
meaning that if $\pr^{*'}$ and $\pr^*$ induce the same data distributions, $\pr'=\pr$, under $\mathcal C$, then they also induce the same ATE, $\delta^*$, so that $\delta^*$ {is} a valid function of $\pr$. This ensures we are in fact considering the same estimand in each of the models. In our primary setting (i.e., setting III), restricting to $\pr^*$  satisfying \cref{assump: unconfound,assump: overlap,assump: mar-1} is enough to ensure identifiability. In settings I and II, since surrogates $S$ are not observed for some units,
we need to further assume that whether the primary outcome is observed or not, i.e., indicator variable $R$, does not depend on surrogates.

\begin{assumption}[Missing at random, cont'd]\label{assump: MAR2}
For $ t= 0, 1$, $R \perp S(t) \mid {T = t}, X$. 
\end{assumption}

With this additional assumption, the ATE parameter is identifiable in all four settings, so we can  compare the efficiency of estimating the same ATE in these different settings. 
\begin{lemma}\label{lemma: identification-2}
If \Cref{assump: unconfound,assump: overlap,assump: mar-1,assump: MAR2} all hold, then the ATE parameter $\delta^*$ is identified in all four settings in \Cref{def: settings}. 
\end{lemma}

In the following lemma, we summarize some additional implications of \Cref{assump: MAR2}.
\begin{lemma}\label{lemma: MAR-implication}
 If \Cref{assump: mar-1,assump: MAR2} hold, then \Cref{assump: unconfound} holds if and only if 
 \begin{align}\label{eq: unconfound-both-data}
  (Y(t), S(t)) \perp T \mid X, R = i, ~~ i \in \braces{0, 1}.
  \end{align}
Moreover, when \Cref{assump: MAR2} holds,  $\rt(t, x, s) =\rt(t, x) \coloneqq \pr(R = 1 \mid T  = t, X = x)$ and $\mut(t, x) = \expect[Y \mid T = t, X = x,R = 1]$. 
\end{lemma} 

In \Cref{lemma: MAR-implication}, \cref{eq: unconfound-both-data}
shows that under the missing-at-random assumptions in \Cref{assump: MAR2,assump: mar-1}, 
the treatment 
unconfoundedness over the {combined} population of the labelled and unlabelled data in \Cref{assump: unconfound} is equivalent to unconfoundedness over the two subpopulations respectively. 
Moreover, \Cref{lemma: MAR-implication} shows that \Cref{assump: MAR2} can also simplify two nuisances that appear in the efficient influence function in \Cref{thm: efficient-if}. 
This is very beneficial because the simplified nuisances are easier to estimate. 
For example, the nuisance function $\mut(t, x) = \expect[Y \mid T = t, X = x,R = 1]$ can be directly estimated by running regressions. 
In contrast, estimating the nuisance function $\mut$ in \cref{eq: nuisance-mu} requires first estimating another nuisance $\tmut$ in \cref{eq: nuisance-mu-tilde} and then further projecting the estimated nuisances.

In the following theorem, we derive efficiency lower bounds for ATE in the four settings in \Cref{def: settings}. We impose \Cref{assump: MAR2} even in settings III and IV where it is not needed for identification, else the four settings would not be comparable.

\begin{theorem}\label{thm: efficiency-comparison}
Under \Cref{assump: unconfound,assump: overlap,assump: MAR2,assump: mar-1}, the efficiency lower bounds for $\delta^*$ in setting $j$ is $V^*_{j} = \expect[\psi^2_j(W; \delta^*, \eta^*)]$ for $j = \text{I}, \dots, \text{IV}$, where 
\begin{align*}
\psi_\text{I}(W; \delta^*, \eta^*) &= \psi_\text{II}(W; \delta^*, \eta^*) = \mut(1, X) - \mut(0, X) - \delta^*  \\
&\qquad\quad + \frac{TR}{\et(X)\rt(1, X)}(Y - \mut(1, X)) - \frac{(1 - T)R}{(1 - \et(X))\rt(0, X)}(Y - \mut(0, X)), \\
\psi_\text{III}(W; \delta^*, \eta^*) 
    &= \mut(1, X) - \mut(0, X)  - \delta^*  + \frac{T - \et(X)}{\et(X)\prns{1  - \et(X)}}(\tmut(T, X, S) - \mut(T , X))
    \\&\phantom{=}+ \frac{TR}{\et(X)\rt(1, X)}(Y - \tmut(1, X, S)) - \frac{(1 - T)R}{(1 - \et(X))\rt(0, X)}(Y - \tmut(0, X, S)),  \\
\psi_\text{IV}(W; \delta^*, \eta^*) 
    &=  \mut(1, X) - \mut(0, X) - \delta^* + \frac{T - \et(X)}{\et(X)\prns{1 - \et(X)}}(Y - \mut(T, X)). 
\end{align*}
{Moreover, the efficiency bounds remain the same if either of $\et$ or $\rt$ or both are known.}
\end{theorem}

\Cref{thm: efficiency-comparison} proves that the efficiency  bound for setting III 
is identical to the bound in \Cref{thm: efficient-if}, meaning that the additional \Cref{assump: MAR2} has no impact on the efficiency bound. {This is because the \Cref{assump: MAR2} only imposes restrictions on the conditional distribution of $R$ given $S, T, X$ while 
the efficient influence function derived in \Cref{thm: efficient-if} is orthogonal to the part of tangent space corresponding to that conditional distribution (see also remarks below \Cref{thm: efficient-if}). We also prove that the efficient influence functions in other settings are also orthogonal to parts of the tangent spaces corresponding to propensity scores, so that the resulting efficiency lower bounds are again invariant to the knowledge of the propensity scores.}
Moreover, even though we have access to surrogates for at least a subset of units in setting II and IV, their efficiency lower bounds do not depend on surrogates $S$.  
This means that surrogates cannot improve the efficiency of ATE estimation if surrogates are observed only when the primary outcome is already observed.
Indeed, for units whose primary outcome is already observed, surrogates can provide no extra information for ATE, especially considering that we do not restrict the relationship between surrogates and the primary outcome at all.
In contrast, for units whose primary outcome is missing, the observed surrogates do provide extra information, because under \Cref{assump: unconfound,assump: overlap,assump: mar-1}, we can learn the relationship between surrogates and the primary outcome based on the labeled data and extrapolate it to the unlabeled data to impute the missing primary outcome.

\begin{corollary}\label{corollary: efficiency-loss}
Suppose \Cref{assump: unconfound,assump: overlap,assump: MAR2,assump: mar-1} hold. Then: 
\begin{enumerate}
\item The efficiency gain from observing the surrogates on all units is measured by 
\begin{align*}
\textstyle
V_\text{I}^* - V^*_\text{III} 
    &= \expect\bigg[\sum_{t\in\{0,1\}}\frac{1 - \rt(t, X)}{(t\et(X) + (1-t)(1-\et(X)))\rt(t, X)}\var[\tmut(t, X, S(t))\mid X]\bigg].
\end{align*}
{(Note $\var[\tmut(t, X, S(t))\mid X] = \var[\Eb{Y(t) \mid X, S(t)}\mid X]$ for $t = 0, 1$.)} 
\item The information loss due to not fully observing the primary outcome 
is measured by 
\begin{align*}
V^*_\text{III}  - V^*_\text{IV} = \expect\bigg[\sum_{t\in\{0,1\}}\frac{1 - \rt(t, X)}{(t\et(X) + (1-t)(1-\et(X)))\rt(t, X)}\var[Y(t) \mid X, S(t)]\bigg].
\end{align*}
\item Observing additional surrogates on the labeled data alone provides no improvement, that is, $V_\text{I}^* = V^*_\text{II}$.
\end{enumerate}
\end{corollary}

\Cref{corollary: efficiency-loss} quantifies the  optimal efficiency gain from surrogates, and the efficiency gap to the ideal setting where the primary outcome is fully observed. 
It shows that the efficiency benefits of surrogates depend on two factors: the {predictiveness} of the surrogates with respect to the primary outcome 
and the extent of {missingness} of the primary outcome.

\textbf{Predictiveness of the surrogates.} The efficiency gain due to surrogates (i.e., $V_\text{I}^* - V^*_\text{III}$) positively depends on the term {$\var\left[\tmut(t, X, S(t))\mid X\right] = \var[\Eb{Y(t) \mid X, S(t)}\mid X]$ for $t \in \{0, 1\}$ (see \Cref{lemma: mu-tilde-fun}). 
This means that the efficiency gain due to surrogates depends on the variations of the primary outcome that can be explained by surrogates but {not} by pre-treatment covariates.} Similarly, the efficiency loss compared to the ideal setting (i.e., $V^*_\text{III}  - V^*_\text{IV}$) positively depends on $\var\left[Y(t) \mid X, S(t)\right]$ for $t = 0, 1$, i.e., the residual variations of the primary outcome that cannot be explained by either the surrogates or the pre-treatment covariates. 
This means that the more predictive the surrogates are, the more efficiency improvement can be achieved by leveraging the surrogates (i.e., larger $V_\text{I}^* - V^*_\text{III}$), and the closer the efficiency bound is to the ideal limit with fully observed primary outcome (i.e., smaller $V_\text{III}^* - V^*_\text{IV}$).   
At one extreme, if $\var\left[Y(t)\mid X,S(t)\right]=0$, i.e., outcomes are given by an (unknown) deterministic function of surrogates and covariates, then observing surrogates is equivalent to observing the primary outcome, and thus $V_\text{III}^* = V^*_\text{IV}$. At the other extreme, if $Y(t)\perp S(t)\mid X$, then surrogates have no predictive power at all and we have $\var[\tmut(t, X, S(t))\mid X]=0$, so there is no benefit to observing surrogates and $V_\text{III}^* = V^*_\text{I}$. In between these extremes, we have $V_\text{I}^*<V_\text{III}^* < V^*_\text{IV}$.

\textbf{Missingness of the primary outcome.} 
{Both quantities in \Cref{corollary: efficiency-loss} increase with the odds of not labeling the outcome, \ie,  $(1-\rt(1, X))/\rt(1, X)$ and $(1-\rt(0, X))/\rt(0, X)$, so they decrease with the labeling propensity scores $\rt(1, X)$ and $\rt(0, X)$.}
This means that when the primary outcome is less missing (i.e., overall higher labeling propensity scores), the efficiency gains from additionally observing surrogates or the primary outcome both decrease. Indeed, if the primary outcome is already observed for most of the units, then the room for extra efficiency gain from observing surrogates (or, from observing more primary outcomes, for that matter) is small.  

The efficiency analysis above provides important guidelines on when leveraging surrogates can improve the efficiency of ATE estimation. It shows that surrogates are particularly beneficial for ATE estimation when (1) surrogates can account for large variations of the primary outcome that cannot be explained by the pre-treatment covariates, and (2) the primary outcome for a large number of units is missing. 

{In \Cref{sec: missing-pattern}, we further extend the analyses of this subsection to allow for  additional missingness patterns. Specifically, in the setting II with partially observed surrogates, the missingness patterns for the surrogates and the primary outcome are identical. In \Cref{sec: missing-pattern}, we further allow the number of surrogate observations to be anywhere between those in setting I and setting II or between those in setting II and setting III.

\subsection{Aside: Other Target Populations}\label{sec: aside}

In the above we considered our estimand to be the ATE on the whole population described by $\pr^*$. 
To identify this estimand, we require the  treatment assignment to be  unconfounded on the whole population (\Cref{assump: unconfound}). 
According to \Cref{lemma: MAR-implication}, under the missing-at-random assumptions in \Cref{assump: MAR2,assump: mar-1}, the whole-population unconfoundedness \Cref{assump: unconfound} amounts to unconfoundedness on both the labelled ($R = 1$) and unlabelled ($R = 0$) subpopulations, separately. 

We could easily consider the ATE on other target populations, for example, $\expect[Y(1)-Y(0)\mid R=i]$ for either $i=0,1$, that is, the ATE on the unlabeled  or labeled  subpopulation. We may be interested in $\delta_1^* = \expect[Y(1)-Y(0)\mid R=1]$ if, for example, the unlabeled data is collected from an auxiliary source to augment a small study already involving the population of interest. 
Or, we may be interested in $\delta_0^* = \expect[Y(1)-Y(0)\mid R=0]$ if the unlabelled data are easier to collect and more representative of the population of interest. 
For these alternative estimands, we only need the treatment assignment to be unconfounded for the target population of interest.

When the target is the ATE on the labelled population $\delta_1^*$, we only need  unconfoundedness on the labelled population, \ie,  $Y(t) \perp T \mid X, R = 1$.
Moreover, we only need strict overlap assumption on the treatment assignment (namely, \Cref{condition: overlap-1} in \Cref{assump: overlap}). 
In particular, 
the missing-at-random assumptions in \Cref{assump: MAR2,assump: mar-1} are not required. 
In this case, the surrogates are not useful since we already fully observe the primary outcome in the labelled population of interest. 
The semiparametric efficiency analysis of $\delta_1^*$ immediately follow from restricting the analysis in \cite{hahn1998role} to the labelled subpopulation. 

When the target is the ATE on the unlabelled population $\delta_0^*$, we only need unconfoundedness on the unlabelled population.  
In the following theorem, we show the identification of $\delta^*_0$ and its semiparametric efficiency bound. 

\begin{theorem}\label{thm: effect-unlabelled}
If $\prns{Y(t), S(t)} \perp T \mid X, R = 0$ and \Cref{assump: mar-1,assump: overlap} hold, then 
\begin{align}\label{eq: effect-unlabelled}
\begin{aligned}
\delta_0^*  &= \Eb{\Eb{\Eb{Y \mid S, X, T = 1, R = 1} \mid X, T = 1, R = 0} \mid R = 0} \\
&\quad\quad - \Eb{\Eb{\Eb{Y \mid S, X, T = 0, R = 1} \mid X, T=0, R = 0} \mid R = 0}.
\end{aligned}
\end{align}
The corresponding semiparametric efficiency bound is $V_0^* = \Eb{\psi^2_0(W; \delta_0^*)}$, where 
\begin{align*}
\textstyle
\psi_0(W; \delta_0^*) 
   &= \frac{1-R}{\Prb{R=0}} \prns{\mut_0(1, X) - \mut_0(0, X) - \delta_0^*} \\
   &+ \frac{1-R}{\Prb{R=0}}\frac{T-\et(0,X)}{\et(0, X)(1-\et(0, X))}\prns{\tmut(T, X, S) - \mut_0(T, X)} \\
&+\frac{R}{\Prb{R=0}}\frac{\Prb{R=0\mid S, X, T}}{\Prb{R=1\mid S, X, T}}{
\frac{T-\et(0,X)}{\et(0, X)(1-\et(0, X))}\prns{Y - \tmut(T, X, S)}},
\end{align*}
and  
$\mut_0(t, x)= \expect[\tmut(T, X, S) \mid X = x, T = t, R = 0]$, $\et(0, X) = \Prb{T=1\mid R=0, X}$.
\end{theorem}

Based on the efficient influence function in \Cref{thm: effect-unlabelled}, we can easily adapt our estimation method in \Cref{sec: estimation} to construct efficient estimators for the parameter $\delta_0^*$. 

{Below, we show that the conclusions in  \Cref{thm: effect-unlabelled} also hold under an alternative set of identification assumptions. We then relate these assumptions to those in \cite{athey2020combining}.} 
{
\begin{proposition}\label{prop: latent-unconf}
If $Y(t) \perp T \mid X, S(t), R = 1$, $S(t) \perp T \mid X, R = 0$ and $Y(t) \perp R \mid X, S(t)$ and \Cref{assump: overlap}, then the conclusions in \Cref{thm: effect-unlabelled} still hold. 
\end{proposition}
}

{
In contrast to \Cref{thm: effect-unlabelled}, which assumes a full unconfoundedness assumption on only the unlabeled subpopulation, 
\Cref{prop: latent-unconf} assumes two conditions related to the confoundedness on the labeled and unlabeled sub-populations separately. These alternative assumption are closely related to the assumptions in \cite{athey2020combining} that combine experimental and observational data to estimate long term treatment effects. 
In their setting, the observational data record observations of both short-term and long-term outcomes, but the experimental data record observations of only short-term outcomes. 
We can re-interpret our labeled ($R = 1$) data as their observational data, our unlabeled data ($R = 0$) as their experimental data, and our surrogates $S$ and primary outcome $Y$ as their short-term and long-term outcomes, respectively. Then the assumptions in \Cref{prop: latent-unconf} recover the identification assumptions in \cite{athey2020combining}. Specifically, the condition $Y(t) \perp T \mid X, S(t), R = 1$ corresponds to their ``latent unconfoundedness'' condition.  
The condition $S(t) \perp T \mid X, R = 0$ corresponds to their unconfoundedness condition on the experimental data, and the condition $Y(t) \perp R \mid X, S(t)$ corresponds to their  external validity condition. 
As a result, \Cref{thm: effect-unlabelled} also applies to the average treatment effect over the experimental population in the setting of \cite{athey2020combining}. This efficiency analysis for the problem of \cite{athey2020combining} matches that in the Online Causal Inference Seminar discussion of that paper by the present authors\footnote{\url{https://sites.google.com/view/ocis/past-talks/summer-2021-talks}} and the subsequent \cite{chen2021semiparametric} 
(see their Theorem B.2). 
}

{
In this paper, we are more interested in the conditions in \Cref{thm: effect-unlabelled} than the setting of \cite{athey2020combining}, since our aim is to study the efficiency gains from surrogate observations when the treatment is unconfounded and identification is not the primary problem. 
Although the assumptions in  \Cref{prop: latent-unconf} happen to justify the same identification formula and  they can be re-interpreted as conditions in \cite{athey2020combining}, they are not the focus of our paper. 
Specifically, \cite{athey2020combining} need the short-term outcomes to achieve identification with a confounded observational study. In contrast, we directly assume unconfoundedness on the labeled population, so we do not even need surrogates to achieve identification, but instead consider surrogates from an efficiency perspective. 
}

\section{Treatment Effect Estimator}\label{sec: estimation}

In this section we develop our treatment effect estimator that efficiently leverages surrogates and achieves the bound derived in \cref{sec: efficiency main setting}.
We first show how to construct the estimator in \Cref{sec: est-construction}. Then we establish the asymptotic guarantees for our estimator in \Cref{sec: est-asymptotic}.

\subsection{Constructing an Efficient Estimator}\label{sec: est-construction}

Our analysis in \cref{sec: efficiency main setting} not only provides the efficiency bound for ATE estimation, but also guides us directly in the construction of an efficient estimator.
In particular, \cref{thm: efficient-if} suggests one {hypothetical} estimator {if} the nuisance parameters $\eta^* = (\et, \rt, \tmut, \mut)$ were known: specifically, the efficient influence function itself in \cref{eq: efficient-IF} gives the estimator $\hat\delta_0$ that solves the following estimating equation: 
\begin{align}\label{eq: infeasible-est}
\frac{1}{N}\sum_{i = 1}^N \psi(W_i; \hat\delta_0, \eta^*) = 0.
\end{align}
It is then easy to verify by the Central Limit Theorem that 
$\sqrt{N}(\hat{\delta}_0 - \delta^*) \overset{d}{\to} \mathcal{N}(0, V^*)$, which validates the efficiency of $\hat\delta_0$.

However, in practice, we do {not} know the nuisance parameters, so the estimator $\hat{\delta}_0$ is {infeasible}. Instead, our approach will be to construct some nuisance parameter estimators $\hat{\eta} = (\he, \hr, \hmu, \htmu)$ first, and then plug them into \cref{eq: infeasible-est} in place of $\eta^*$.
We could estimate $\eta^*$ using parametric models (e.g., generalized linear models), but this could risk model misspecification and lead to inconsistent estimates. 
This is particularly a concern when either covariates $X$ or surrogates $S$ are rich, which should normally be regarded as a good thing as it can make \cref{assump: unconfound} more defensible as well as increase surrogates' predictiveness and hence the efficiency gains from surrogate observations.
Hence, we prefer flexible machine learning estimators 
 that avoid restrictive parametric  assumptions on the nuisance parameters to avoid misspecification error.
For example, estimating $\et, \rt$ amounts to learning conditional probabilities from binary classification data, and estimating $\mut, \tmut$ is essentially learning two regression functions. For both tasks many successful machine learning methods exist \citep[e.g.,][]{breiman2001random,chen2016xgboost,goodfellow2016deep}.

Although flexible machine learning estimators are less prone to model misspecification, we must be careful that their slow convergence and possible biases do not impact our estimator badly so that the resulting {feasible} ATE estimator is still root-$N$ consistent and asymptotically normal, just like the {hypothetical} estimator in  \cref{eq: infeasible-est}. 
Luckily, the efficient influence function we derive in \cref{eq: efficient-IF} has a special multiplicative bias structure that makes it insensitive to errors in $\eta^*$.

\begin{lemma}\label{lemma: multiplicative-bias}
{There exists a universal  constant $c_0 > 0$ that only depends on $\epsilon$, such that}
for any $\eta_0 = (e_0, r_0, \tilde\mu_0, \mu_0)$ with $e_0, r_0$ satisfying \cref{assump: overlap} and any $\delta$,  
\begin{align*}
\begin{aligned}
      &\abs{\Eb{\psi(W; \delta, \eta_0)-\psi(W; \delta, \eta^*)}} 
        \\
       &\qquad \le c_0\big(\|e_0 - \et\|_2\|\mu_0 - \mut\|_2 + \|r_0 - \rt\|_2\|\tmu_0 - \tmut\|_2 + \|e_0 - \et\|_2\|\tmu_0 - \tmut\|_2\big).
\end{aligned}
\end{align*}
\end{lemma}

\Cref{lemma: multiplicative-bias} suggests that replacing $\eta^*$ with $\hat\eta$ in \cref{eq: infeasible-est}, our estimate remains consistent even if some nuisance estimates are inconsistent (see \Cref{thm: DR} below). 
More crucially, \Cref{lemma: multiplicative-bias} suggests that if all nuisance estimators are consistent but converge slowly, then the overall error in using $\hat\eta$ in place of $\eta^*$ in \cref{eq: infeasible-est} will converge as the {product} of the slow rates. If this product is faster than the $O_p(N^{-1/2})$ convergence of $\hat\delta_0$ itself, e.g., each rate is $o_p(N^{-1/4})$, then our {feasible} estimator will asymptotically behave the same as the {infeasible} one (see \Cref{thm: normality} below for the  formal statements).
The property shown in \Cref{lemma: multiplicative-bias} implies the Neyman orthogonality property that plays a central role in the recent debiased machine learning literature \citep[e.g., ][]{chernozhukov2018double,newey2018cross}. 

To construct our estimator, we further employ cross-fitting in nuisance estimation.
We divide the data into multiple folds, use data in all but one fold to estimate nuisances, and apply the estimated nuisance only to the hold-out fold. 
This technique prevents each nuisance estimator from overfitting to the data where it is evaluated, and eschews stringent Donsker conditions on the nuisance estimators, which has been widely used in semiparametric estimation \citep[e.g., ][]{chernozhukov2018double,zheng2011cross}.

\begin{definition}[Cross-fitted Estimator]\label{def: estimator}
Let $K$ be a fixed positive interger. Take $K$-fold random partitions $\{\mathcal{I}_k^l\}_{k = 1}^K$ and $\{\mathcal{I}_k^u\}_{k = 1}^K$ of the labeled and unlabeled index sets $\mathcal{I}^l$ and $\mathcal{I}^u$ respecitvely. Then $\{\mathcal{I}_k = \mathcal{I}_k^l \cup \mathcal{I}_k^u\}_{k = 1}^K$ constitutes a K-fold random partition of the whole index set $\{1, \dots, N\}$. 
For each $k = 1, \dots, K$, we define $\mathcal{I}_k^c = \{1, \dots, N\}\setminus \mathcal{I}_k$ and use all but the $k$th fold data
to train machine learning estimators for the nuisance parameters:  $\hat{\eta}_k = \hat{\eta}(\{W_i\}_{i \in \mathcal{I}_k^c})$. The final ATE estimator is $\hat\delta$ that solves the following equation: 
\begin{align}\label{eq: DML-est}
\textstyle
\frac{1}{K}\sum_{k = 1}^K \expectnk\bracks{\psi(W; \hat\delta, \hat\eta_k)} = \frac{1}{K}\sum_{k = 1}^K \frac{1}{\abs{\mathcal{I}_k}}\sum_{i \in \mathcal{I}_k} \psi(W_i; \hat\delta, \hat\eta_k) = 0,
\end{align}
where $\expectnk$ denotes the sample average over the $k^{\text{th}}$ fold. This estimator can be also written as
\begin{align*}
\textstyle
\hat{\delta}   
    &= \frac{1}{K}\sum_{k = 1}^K\expectnk\bigg[\hmuk(1, X) - \hmuk(0, X) + \frac{T - \hek(X)}{\hek(X)\prns{1 - \hek(X)}}(\htmuk(T, X, S) - \hmuk(T , X))
    \\&\qquad\quad+ \frac{TR}{\hek(X)\hrk(1, X, S)}(Y - \htmuk(1, X, S)) - \frac{(1 - T)R}{(1 - \hek(X))\hrk(0, X, S)}(Y - \htmuk(0, X, S)) \bigg].
\end{align*}
\end{definition}

\subsection{Asymptotic Properties of the Estimator}\label{sec: est-asymptotic}

In this section we establish the insensitivity of our estimator to the nuisance estimation errors. Namely, we establish both a double robustness property as well as efficiency. We then proceed to use our results to also construct valid confidence intervals.

Our results will depend on the asymptotic behavior of our nuisance estimates, $\hat\eta_k$. To state our results, we use the next assumption to define both the limit point of the estimates and the convergence rate. Note it is only an ``assumption'' once we specify a certain limit point and rate. In particular, the below allows the nuisance estimators to be misspecified in that the limit point $\eta_0$ need not be equal to $\eta^*$.

\begin{assumption}[Nuisance Estimator Convergence Rate]\label{assump: convergence}
For $k = 1, \dots, K$, the nuisance estimators $\hat{\eta}_k = (\hek, \hrk, \hmuk, \htmuk)$ converge to their limit $\eta_0 = (\eo, \ro, \muo, \tmuo)$ in mean sqaured error at the following rates:
\begin{align*}
\|\hek - \eo\|_2 = {O}_p(\rhoe), ~ \|\hrk - \ro\|_2 = {O}_p(\rhor), ~ \|\hmuk - \muo\|_2 = {O}_p(\rhomu), ~ \|\htmuk - \tmuo\|_2 = {O}_p(\rhotmu). 
\end{align*}
Furthermore, the propensity score estimators and their asymptotic limits are almost surely bounded: $\hek(X), \eo(X) \in [\epsilon, 1 - \epsilon]$ and $\hrk(X), \ro(X) \in [\epsilon, 1]$ with probability $1$. 
\end{assumption}

We further assume the following boundedness on the variance of the primary outcome.
\begin{assumption}[Bounded Moments]\label{assump: bounded}
There exist constants $C > 0,\,q > 2$ such that 
\begin{align*}
&\|\var\{Y \mid X, S, T\}\|_{\infty} \le C, ~~ \|\var\{Y \mid X, T\}\|_{\infty} \le C, \\
 &\|\var\{\tmut (T, X, S) \mid T, X\}\|_{\infty} \le C, ~~ \|Y(1)\|_q \vee \|Y(0)\|_q \le C.
\end{align*}
\end{assumption}

Our next result establishes formally the doubly robust property of our estimator $\hat\delta$.

\begin{theorem}[Double Robustness]\label{thm: DR}
Given \Cref{assump: unconfound,assump: overlap,assump: convergence,assump: bounded,assump: mar-1}, if we further assume that $\rhoe, \rhor, \rhomu, \rhotmu$ are all $o(1)$, $(\tmuo - \tmut)(\ro - \rt) = 0$, $(\tmuo - \tmut)(\eo - \et) = 0$, $(\muo - \mut)(\eo - \et) = 0$, 
and the asymptotic bias $\|\tmuo - \tmut\|$ and $\|\muo - \mut\|$ of the outcome regressions are almost surely bounded by the positive constant $C$, 
then $\hat{\delta} \overset{p}{\to} \delta^*$  as $N \to \infty$.
\end{theorem}
\Cref{thm: DR} states that the proposed estimator converges to the true ATE, as long as all nuisance estimators converge to a limit point and at least one of the limit points, but not necessarily both, in each pair of {$(\tmuo, \ro)$, $(\muo, \eo)$, and $(\tmuo, \eo)$} is equal to the corresponding true value. 
Thus the consistency of our estimator does not require all nuisance parameters to be correctly estimated, nor the knowledge of which one is correctly estimated. 
This means that our estimator is robust to misspecification errors of estimating some nuisance parameters, as long as  the rest are consistently estimated. This property is called ``double robustness'' in causal inference literature \citep[e.g., ][]{scharfstein1999adjusting,kang2007demystifying}.

Our next result formalizes the notion that slow convergence rates in nuisance estimation multiply, causing the effect of estimating nuisances to be negligible in analyzing $\hat\delta$ and its first-order behavior to be similar to $\hat\delta_0$ that uses the true nuisances. 
\begin{theorem}[Asymptotic Normality]\label{thm: normality}
Under assumptions in \Cref{thm: DR}, if we  assume $\max\{\rhor\rhotmu, \rhoe\rhotmu, \rhoe\rhomu\}  = o(N^{-1/2})$, and that all nuisance components are correctly specified so that $\tmuo - \tmut = \ro - \rt = \muo - \mut = \eo - \et = 0$, then as $N \to \infty$,
\begin{align*}
\sqrt{N}(\hat{\delta}  - \delta^*) \overset{d}{\to} \mathcal{N}(0, V^*),
\end{align*}
where $V^*$ is the efficiency lower bound in \Cref{thm: efficient-if}.
\end{theorem}

\Cref{thm: normality} further shows that if all nuisance estimators converge to the truth at sufficiently fast rate, then the proposed estimator $\hat{\delta}$ converges at rate $O_p(N^{-1/2})$, and it is asymptotically normal with the efficiency lower bound $V^*$ as its limiting variance. The rate requirement is lax and can be satisfied even if all nuisance estimators converge to true values at $o_p(N^{-1/4})$ rates, i.e., much slower than the parametric rate $O_p(N^{-1/2})$. Notably, we do not restrict the nuisance estimators to Donsker or bounded entropy classes \citep{van2000asymptotic}, thereby permitting flexible machine learning methods. Moreover, the product rate condition allows estimators converging at faster rate to compensate for those converging at slower rate. 
For example, if we have strong domain knowledge about the labeling process and treatment assignment process (e.g., in a randomized experiment with two-phase sampling design) so that we can estimate the labeling propensity score $\rt$ and treatment propensity score $\et$ at very fast rate, then we can allow for very flexible regression estimators $\hmu, \htmu$ that converge at slow rates. 
Therefore, the estimation errors of machine learning nuisance estimators may not undermine the asymptotic behavior of our ATE estimator and it can still achieve the efficiency bound of \cref{thm: efficient-if} similarly to the infeasible estimator $\hat\delta_0$.

In the next result we propose a way to consistently estimate the efficient variance, $V^*$, which immediately lends itself to confidence interval construction.

\begin{theorem}[Confidence Interval]\label{thm: conf-interval}
Under the assumptions in \cref{thm: normality},
\begin{align*}
\textstyle 
\hat{V} = \frac{1}{K}\sum_{k = 1}^K\expectnk[\psi^2(W; \hat{\delta}, \hat{\eta}_k)] \overset{p}{\to} V^* \text{ as } N \to \infty. 
\end{align*}
Consequently, the following $(1 - \alpha) \times 100\%$ confidence interval 
\begin{align}\label{eq: CI}
\operatorname{CI} = (\hat{\delta} - \Phi^{-1}(1 - \alpha/2){(\hat{V}/N)}^{1/2}, ~~ \hat{\delta} + \Phi^{-1}(1 - \alpha/2){(\hat{V}/N)}^{1/2})
\end{align}
with $\Phi$ as the cumulative density function of standard normal distribution satisfies that 
\[
    \pr(\delta^* \in \operatorname{CI}) \to 1 - \alpha \text{ as } N \to \infty.
\]
\end{theorem}
\Cref{thm: conf-interval} shows that under the same conditions, we can consistently estimate the efficiency lower bound by forming the sample analogue of $\expect[\psi^2(W; {\delta^*}, {\eta^*})]$ with cross-fitting nuisance estimators and the proposed estimator $\hat{\delta}$. 
The resulting confidence interval in \cref{eq: CI} asymptotically achieves correct coverage probability. 
Also, since the proposed estimator $\hat{\delta}$ asymptotically achieves the smallest possible variance, the confidence interval in \cref{eq: CI} tends to be shorter than confidence intervals based on less efficient estimators.

\section{Extension: Very Large Unlabeled Data}\label{sec: extension}
In this section, we consider the setting where the size of unlabeled data is much larger than the size of the labeled data, i.e., $N_u \gg N_l$.
This setting is practically relevant since the number of units being followed 
in a study with labeled outcome may often be much smaller than the massive amount of unlabeled data cheaply collected fro existing 
 databases, such as from electronic medical records \citep{cheng2018efficient}.
Despite its practical relevance, this setting cannot be directly accommodated by our efficiency and estimation theory in \Cref{sec: efficiency,sec: estimation}.
This is because  previous results all hinge on the overlap condition \eqref{condition: overlap-2} in \Cref{assump: overlap}, 
which implies that the 
marginal labeling probability is positive, $\pr(R = 1) \geq \epsilon$, and thus $N_u \asymp N_l$ with high probability. This rules out the setting with many more unlabeled data, that is, $N_u \gg N_l$ and $\pr(R = 1) = 0$. 
In fact, in the very-many-unlabeled-data setting, we can show that $\rt(T, X, S) = 0$ almost surely (\Cref{lemma: no-positivity}), so previous efficiency lower bounds based on $1/\rt$ are invalid.
Despite these drastic differences, we will show that essentially the same efficiency results and estimation strategy actually still apply in the very-many-unlabeled-data setting, as long as we change the perspective  and scaling appropriately. 

\subsection{Efficiency Analysis}

{To accommodate this  setting,
we change the efficiency considerations mainly in two aspects. 
First, instead of using $\rt$, we characterize efficiency in terms of the following density ratio: 
\begin{align}
\label{eq: density-ratio}
\lambda^*(S, X, t) 
&\coloneqq \frac{f^*(S, X \mid T = t)}{f^*(S, X \mid T = t, R  = 1)}, ~~ \text{for } t = 0, 1, 
\end{align}
where $f^*(\cdot \mid T = t, R = \cdot)$ and $f^*(\cdot \mid T = t)$ are the conditional density functions of $S$ and $X$ corresponding to the target distribution $\pr$.
In particular, in the limit that $\Prb{R = 1} = 0$ (meaning that the unlabeled data dominates the whole data) we further have 
\begin{align}\label{eq: density-ratio-0}
 \lambda^*(S, X, t) = \lambda^*_0(S, X, t) \coloneqq \frac{f^*(S, X \mid T = t, R = 0)}{f^*(S, X \mid T = t, R  = 1)}. 
 \end{align} 
Notably, this density ratio is well-defined and bounded as long as the distribution of $S, X$ on the labeled and unlabeled data overlap sufficiently, 
even if $N_l \ll N_u$ and $\Prb{R = 1} = 0$.\footnote{{We use the conditional distribution given $R = 1$ to denote the distribution from which the labeled data is sampled. 
It is well-defined 
even though $\Prb{R = 1} = 0$. In \Cref{sec: est-vanish}, we will rationalize this choice by an observation model where the chance of labeling a data point is strictly positive but it converges to $0$ as the sample size grows to $\infty$.}}}

Second, we will characterize  convergence rates of  ATE estimators in terms of the labeled data size $N_l$ instead of the total sample size $N$. 
This is crucial since,
in the current setting, the size of the labeled data becomes the bottleneck for accurate ATE estimation, noting that the primary outcome observed in the labeled data is the primary source of information on the ATE, but its sample size, $N_l$, is on a different scale than the total sample size, $N$.  

Since $N_l \ll N_u$, from the perspective of the behavior as $N_l\to\infty$, the size of the unlabeled dataset appears infinitely larger and thus its distribution, i.e., the distribution of $(X, T, S)$ given $R = 0$, appears virtually known. 
Moreover, in the asymptotic limit, the labeled dataset is negligible, and the combined dataset is virtually identical to the unlabeled dataset ($N_u/N\to1$).
Thus the unconditional distribution of $(X, T, S)$ is in the limit identical to  their conditional distribution  given $R = 0$. 
 Therefore, the unconditional distribution of $(X, T, S)$ can also be viewed as known from the perspective of labeled data. 
The semiparametric efficiency lower bound for ATE from this perspective is formalized in the following theorem.

\begin{theorem}\label{thm: vanishing-if}
Consider the data consisting of 
i.i.d. draws 
from the unknown conditional distribution of $(X, T, S, Y)$ given $R = 1$.   
Suppose \Cref{assump: unconfound,assump: mar-1} hold, $\lambda^*(S, X, t) < \infty$ almost surely for $t = 0, 1$, and $\Prb{T  = 1 \mid R = 1, X, S} \in (\epsilon', 1-\epsilon')$ almost surely for some $\epsilon' \in (0, 1/2)$. 
Then the semiparametric efficiency lower bound for the ATE parameter with respect to a known unconditional distribution of $(X, T, S)$ 
is
$\tilde{V}^* =  \expect[\tilde{\psi}^2(W; \delta^*, \tilde{\eta}^*) \mid R = 1]$, where the new nuisance functions are $\tilde{\eta}^* = (\et, \lambda^*, \tmut, \mut)$ and
\begin{align}\label{eq: vanishing-if}
\begin{aligned}
        \tilde{\psi}(W; \delta^*, \tilde{\eta}^*) 
      &= \frac{T\lambda^*(S, X, T)}{\et(X)}\frac{\Prb{T=1}}{\Prb{T=1\mid R=1}}(Y - \tmut(1, X, S))  \\
      &\quad - \frac{(1 - T)\lambda^*(S, X, T)}{1 - \et(X)}\frac{\Prb{T=0}}{\Prb{T=0\mid R=1}}(Y - \tmut(0, X, S)).
\end{aligned}
\end{align}
\end{theorem}
{\Cref{thm: vanishing-if} does not assume the full   overlap condition in \Cref{assump: overlap}  that excludes the very-many-unlabeled-data setting. Instead, it only assumes a  treatment overlap condition on the labeled data, which can be shown to be weaker than \Cref{assump: overlap} (\Cref{sec: proof-extension} \Cref{lemma: overlap-t})}. 
Moreover, \Cref{thm: vanishing-if} 
 assumes the MAR assumptions in 
\Cref{assump: mar-1},
 which strictly generalizes the results in \cite{cheng2018efficient} under the more restrictive MCAR condition. 
This generalization is possible mainly because we formulate the efficiency lower bound in terms of the density ratio, as opposed to the inverse labeling propensity score formulation that is more commonly used but ill-defined in the current setting. 

{
One may wonder the connection between the efficiency bound in \Cref{thm: vanishing-if} and those in \Cref{sec: efficiency} when the overlap condition in \Cref{assump: overlap} holds. This connection is revealed by following proposition that rescales the efficiency bound in \Cref{thm: efficient-if}.}

\begin{proposition}\label{prop: efficient-bound-rescale}
Suppose \Cref{assump: unconfound,assump: mar-1,assump: overlap} holds and let $V^*$ and $\tilde V^*$ be the semiparametric efficiency lower bounds given in \Cref{thm: efficient-if,thm: vanishing-if}, respectively. Then for any asymptotically efficient estimator $\hat\delta$ such that $\sqrt{N}(\hat\delta - \delta^*) \overset{d}{\to} \mathcal{N}(0, V^*)$ as $N \to \infty$,  we have $\sqrt{N_l}(\hat\delta - \delta^*) \overset{d}{\to} \mathcal{N}(0, \pr(R=1)V^*)$, where $\pr(R=1)V^*$ has the following decomposition:  
\begin{align*}
\textstyle 
\pr(R = 1)V^* 
      &= \tilde V^* + \pr(R = 1)\expect\left[\prns{\mut(1, X) - \mut(0, X) - \delta^*}^2\right]  \\
        &+\pr(R = 1)\expect\bigg[\prns{\frac{T - \et(X)}{\et(X)\prns{1 - \et(X)}}(\tmut(T, X, S) - \mut(T , X))}^2\bigg].
\end{align*}
\end{proposition}

According to \Cref{prop: efficient-bound-rescale}, the efficiency bound in \Cref{thm: efficient-if}, when rescaled according to the size of labeled data, can be decomposed into the efficiency bound in \Cref{thm: vanishing-if} and some additional terms. These additional terms quantify the intrinsic estimation uncertainty due to not knowing the distribution of $(X, T, S)$. Notably, the additional terms vanish as $\Prb{R = 1}\to0$, so the efficiency bound in \Cref{thm: vanishing-if} can be viewed as a limit of the efficiency bound in \Cref{thm: efficient-if}. This means that essentially the same efficiency results reign in both the regime in \Cref{sec: efficiency} and the very-large-unlabeled-data setting in this section.

We can also extend  \Cref{thm: vanishing-if,prop: efficient-bound-rescale} to the ATE parameter on the unlabelled population $\delta^*_0$ in \Cref{sec: aside}. In fact, the same efficiency lower bound applies to  $\delta^*_0$. In the current setting, where unlabelled dataset dominates the combined dataset, the unlabeled and combined population distributions become identical in the limit. 
Consequently, the corresponding ATEs $\delta^*_0$ and $\delta^*$ are also identical and share the same asymptotic efficiency bound. 
The  extensions are formally stated in \Cref{sec: effect-unlabeled-vanishing}.
As  mentioned in \Cref{sec: aside}, the $\delta_0^*$ parameter can be reinterpreted as  long-term treatment effect for the experimental population in the setting of \cite{athey2016estimating}. Therefore, our efficiency analysis  indirectly provides the corresponding efficiency bound when the experimental dataset is much larger than the observational dataset. This complements the analyses in \cite{chen2021semiparametric}, which focus exclusively on two datasets of comparable size. 

{
Although the condition of exactly knowing the distribution of $(X, T, S)$ in \Cref{thm: vanishing-if} seems idealized, we will confirm in the next subsection that this indeed characterizes the role of unlabeled data in an asymptotic sense from the perspective of labeled data. In particular, we will show that essentially the same estimator in \Cref{def: estimator} can still attain the efficiency bound under appropriate conditions.  
}

\subsection{Asymptotically Efficient Estimator}\label{sec: est-vanish}

Our efficiency analysis above is asymptotic, where asymptotically we have $N_l\ll N_u$ and hence $\pr(R = 1) = 0$. In terms of estimation from an actual finite sample, however, we do have {some} labeled data, albeit much less than unlabeled data. Therefore, the observation model of having $N$ i.i.d. draws from $\pr$ is inappropriate as it implies that we observe {no} outcome data with probability 1, which would make estimation impossible. 

{We consider a different observation model that allows for the marginal probability of labeling, which we denote as $\pi_N$, to vary with the total sample size $N$. We  require that $\pi_N>0,\,\pi_N \to 0$ and that the expected labeled sample size $\ON_l = \pi_N N$ grows to $\infty$ as $N \to \infty$.  
We then consider the observation model where we have $N$ i.i.d. draws, for each of which with probability $\pi_N$ we sample the observation from the fixed conditional distribution of $(X, T, S, Y)$ given $R=1$ and with probability $1-\pi_N$ we sample from the fixed conditional distribution given $R=0$. That is, we define $\pr^{(N)}(\mathcal{E}) = \pr(\mathcal{E} \mid R = 1)\pi_N + \pr(\mathcal{E} \mid R = 0)\prns{1 - \pi_N}$ for any event $\mathcal{E}$ measurable with respect to $(X, T, S,Y)$ and we observe $N$ i.i.d. draws from $\pr^{(N)}$. 
This means that the data distribution could change with the sample size $N$ only because of the  labling probability $\pi_N$, but 
the conditional distributions given $R = 1$ and $R = 0$ do not change with the sample size $N$.
This specification has the same spirit as the observation model in \cite{zhang2021double}. It aptly models the regime of interest:  
the labeled data are always available in finite sample, and asymptotically its absolute size grows to infinity despite the fact that its relative size as a fraction of $N$ vanishes as $N \to \infty$.
In fact, this specification is very general: it can accommodate both the current setting by letting $\pi_N  \to 0$, and the original setting (see \Cref{sec: setup}) by letting $\pi_N=\pr(R=1)>0$.}
 
 {Under this observation model, we have a fixed density ratio $\lambda^*_0$ as given in \Cref{eq: density-ratio-0}, but we have a finite-sample  counterpart $\lambda^*_N$ for the limiting density ratio $\lambda^*$ in \Cref{eq: density-ratio}:
 \begin{align}
  &\lambda^*_N(S, X, t) = (1 - \pi_{N, t})\lambda_0^*(S, X, t) + \pi_{N, t},\label{eq: lambda*} \\ 
 \text{ where }& \pi_{N, t} = \pr^{(N)}(R = 1 \mid T = t) = \frac{\Prb{T = t \mid R = 1}\pi_N}{\Prb{T = t \mid R = 1}\pi_N + \Prb{T = t \mid R = 0}(1-\pi_N)}. \nonumber 
 \end{align}
The finite-sample density ratio $\lambda^*_N$ is well-defined and bounded whenever the labeled and unlabeled data distributions sufficiently overlap so that  $\lambda^*_0$ is well-defined and bounded.}

 Our observation model also  
   induces the following $N$-dependent labeling  propensity score  
\begin{align}
\rt_N(t, X, S) 
  &\coloneqq \pr^{(N)}(R = 1 \mid T = t, X, S)  = \frac{\pi_{N, t}}{\lambda^*_N(S, X, t)}, \label{eq: prop-N}
\end{align}
and similarly a  treatment propensity score $
\et_N(X) \coloneqq \pr^{(N)}(T = 1 \mid X)$.  
When the density ratio $\lambda^*_0(S, X, t) < \infty$ almost surely and $\pi_{N, t} > 0$ for any finite $N$, the induced labeling propensity score typically satisfies $r^*_N(t, X, S) > 0$. This means that although in the limit $\rt(t, X, S) = 0$ almost surely, we have $r^*_N(t, X, S) > 0$ for any finite $N$.

{In \Cref{def: estimator}, we propose an ATE estimator $\hat\delta$ based on the semiparametric efficient influence function in \Cref{thm: efficient-if}. This ATE estimator needs 
 nuisance estimators $\{\hat e_k, \hat r_k, \hat \mu_k, \hat{\tilde{\mu}}_k\}_{k=1}^K$ for the unknown functions $(\et, \rt, \mut, \tmut)$ in the setting with $N_l \asymp N_u$. According to \Cref{prop: efficient-bound-rescale}, the efficiency bound in \Cref{thm: efficient-if} is closely related to the efficiency bound in \Cref{thm: vanishing-if} for the very-large-unlabeled-data setting with $N_l \ll N_u$. It is natural to consider whether this estimator in \Cref{def: estimator} also applies to the $N_l \ll N_u$ setting. Moreover, since the efficiency bound in \Cref{thm: vanishing-if} involves the density function $\lambda^*$, we may alternatively estimate the density ratio in order to estimate the average treatment effect. An estimator  following this idea and \Cref{def: estimator} is provided in the definition below. }

\begin{definition}[Revised Estimator]\label{def: estimator2}
We take the $K$-fold random partitions as in \Cref{def: estimator}, and analogously construct nuisance estimators $\hat{\tilde{\eta}}_k  = (\he_k, \hat{\lambda}_k, \htmuk, \hmuk), k = 1, \dots, K$ for the nuisance functions $\tilde\eta^* = (\et, \lambda^*, \tmut, \mut)$. 
We use $\hat{\pi}_N = N_l/N > 0$ to estimate the proportion of labeled data $\pi_N$ and use $\hat \nu_1, \hat\nu_0$ to estimate the probability ratios $\nu_1^* \coloneqq \pr^{(N)}\prns{T=1}/\pr\prns{T=1 \mid R = 1}$ and $\nu_0^* \coloneqq \pr^{(N)}\prns{T=0}/\pr\prns{T=0 \mid R =1 }$ respectively. 
The revised ATE estimator is   
\begin{align*}
\hat\delta^{\op{rev}}
    &= \frac{1}{K}\sum_{k = 1}^K\hexpect_{k} \bigg\{ \hmuk(1, X) - \hmuk(0, X) + \frac{T - \hek(X)}{\hek(X)\prns{1 - \hek(X)}}(\htmuk(1, X, S) - \hmuk(1 , X))  \\
    &\quad + 
    \frac{TR}{\hek(X)}\frac{\hat \nu_1\hat{\lambda}_k(S, X, T)}{\hat{\pi}_N}(Y - \htmuk(1, X, S)) - \frac{(1-T)R}{1 - \hek(X)}\frac{\hat \nu_0\hat{\lambda}_k(S, X, T)}{\hat{\pi}_N}(Y - \htmuk(0, X, S))\bigg\}.
\end{align*}
\end{definition}

The estimator $\hat\delta^{\op{rev}}$ is almost identical to the estimator $\hat\delta$ in \Cref{def: estimator}, except that it replaces the inverse of the estimated labeling propensity score $1/\hrk(t, X, S)$ in \Cref{def: estimator} by the estimated density term $\hat\nu_t\hat\lambda_k(S, X, t)/\hat \pi_N$. 
In both $\hat\delta$ and $\hat\delta^{\op{rev}}$, the preliminary nuisance estimators $\hek, \hmuk, \htmuk$ can be straightforwardly obtained by regressing the treatment $T$ and the outcome $Y$ on suitable predictors like 
the covariates $X$ and the short-term outcomes $S$, using the full data or the labeled data. 
The estimators $\hat r_k$ and $\hat\lambda_k$ are more tricky, considering that the labeling variable $R$ is very imbalanced in the current $N_l \ll N_u$ setting. 
In this subsection, we hypothesize some generic nuisance estimators $\hat r_k$ and $\hat\lambda_k$ and investigate  high-level conditions needed for the resulting estimators $\hat\delta$ and $\hat\delta^{\op{rev}}$ to be asymptotically normal and efficient. We will further study the construction of $\hat r_k$ and $\hat\lambda_k$ in \Cref{sec: nuisance-vanish} and \Cref{sec: simulation}.

{Again, we need to specify the convergence rates of the nuisance estimators. }

\begin{assumption}[Nuisance Estimator Convergence Rates, Cont'd]\label{assump: nuisance-rate-vanish}
For $k = 1, \dots, K$, the nuisance estimators $(\hek, \hmuk, \htmuk, \hrk, \hat\lambda_k)$ converge to $(\et_N,  \mut, \tmut, \rt_N, \lambda^*_N)$ at the following rates:
\begin{align*}
\begin{array}{c}
\|\hek - \et_N\|_2 = {O}_p(\rho_{N, e}), ~ \|\hmuk - \mut\|_2 = {O}_p(\rho_{\bar{N}_l, \mu}), ~ \|\htmuk - \tmut\|_2 = {O}_p(\rho_{\bar{N}_l, \tilde{\mu}}), \\
\left\|{\rt_N}/{\hrk} - 1\right\|_2 = {O}_p(\rho_{\bar{N}_l, r}), ~ \|\hat\lambda_k - \lambda^*_N\|_2 = {O}_p(\rho_{\bar{N}_l, \lambda}),
\end{array} 
\end{align*}
where $\bar{N}_l = \pi_N N$ is the expected size of the labeled data. 
Furthermore, the treatment propensity score estimator is almost surely bounded: $\hek(X)\in [\epsilon, 1 - \epsilon]$, and the labeling propensity score estimator $\hrk(t, S, X) > 0$ almost surely for $t \in \{0, 1\}$.
\end{assumption}

{
  \Cref{assump: nuisance-rate-vanish} is similar to \Cref{assump: convergence} except in three aspects. 
  First, the target of the propensity score estimators $\hek, \hrk$ are set as $\et_N, \rt_N$, since  these are the propensity scores induced by our observation model as we discuss above. 
  Second, the error of the labeling propensity score estimator $\hrk$ is quantified in terms of $\|{\rt_N}/{\hrk} - 1\|_2$ rather than $\|\hrk - \rt_N\|$ to  properly characterize the convergence when $\rt_N$ itself vanishes to $0$. Third, the  error rates in estimating $\tmu, \tmut, \rt_N, \lambda^*_N$ are indexed by the expected size of the labeled data $\bar{N}_l$ rather than the full sample size $N$, since their estimation are all limited by the size of labeled data.  
  In contrast, the error rate in estimating $\et_N$ is still indexed by $N$ since $\hek$ can be obtained by regressing the treatment variable on the full data. 
}

The following theorem derives the asymptotic distributions of  the estimators in \Cref{def: estimator,def: estimator2}  when $N_l \ll N_u$, under suitable high-level nuisance estimation conditions.

\begin{theorem}\label{thm: normality2}
Let $\hat\delta$ and $\hat\delta^{\op{rev}}$ be the estimators in \Cref{def: estimator,def: estimator2} respectively and $\tilde V^*$ be the  efficiency bound in \Cref{thm: vanishing-if}. 
Suppose assumptions in \Cref{thm: vanishing-if} and \Cref{assump: nuisance-rate-vanish} hold, the expected proportion of labeled data $\pi_N$ is strictly positive for any finite $N$ while $\pi_N \to 0$, the expected size of labeled data $\bar{N}_l = \pi_N N \to \infty$ as $N \to \infty$, and mild moment regularity conditions given in \Cref{sec: more-extension} \Cref{assump: moment} also hold. 
\begin{enumerate}
\item If $\rho_{\bar{N}_l, r}\rho_{\bar{N}_l, \tilde{\mu}} = o(\bar{N}_l^{-1/2})$, $\rho_{N, e}\rho_{\bar{N}_l, {\mu}} = o(\bar{N}_l^{-1/2})$, and $\rho_{N, e}\rho_{\bar{N}_l, {\tmu}} = o(\bar{N}_l^{-1/2})$, then $\sqrt{N_l}(\hat\delta - \delta^*) \overset{d}{\to} \mathcal{N}(0, \tilde V^*)$ as $N \to \infty$.
\item  If $\max\{\rho_{\bar{N}_l, \lambda}, \abs{\hat \nu_1 - \nu_1^*}, \abs{\hat \nu_0 - \nu_0^*}, |{\frac{r^*_N}{\hat r_N}-1}|\}\rho_{\bar{N}_l, {\tmu}} = o(\bar{N}_l^{-1/2})$, $\rho_{N, e}\rho_{\bar{N}_l, {\mu}} = o(\bar{N}_l^{-1/2})$, and $\rho_{{N}, e}\rho_{\bar{N}_l, \tilde{\mu}} = o(\bar{N}_l^{-1/2})$,  then $\sqrt{N_l}(\hat\delta^{\op{rev}} - \delta^*) \overset{d}{\to} \mathcal{N}(0, \tilde V^*)$ as $N \to \infty$. 
\end{enumerate}
\end{theorem}

  \Cref{thm: normality2} shows that in the $N_l \ll N_u$ regime and under some nuisance rate conditions, the estimators in \Cref{def: estimator,def: estimator2} are  asymptotically equivalent and their convergence rates are  $O_p({N}_l^{-1/2})$ rather than $O_p({N}^{-1/2})$. In particular, they are both asymptotically normal and  attain the semiparametric efficiency bound in \Cref{thm: vanishing-if}. The rate conditions for the nuisance estimators used to construct  $\hat\delta$ is analogous to the conditions in \Cref{thm: normality}. The main difference is that the  product  rate conditions are accordingly weakened to $o(\bar{N}^{-1/2}_l)$ rather than $o(N^{-1/2})$. 
  Nevertheless, the conditions may be still non-trivial, especially for the rate $\rho_{\bar{N}_l, r}$ in estimating the vanishing labeling propensity $\rt_N$.  
  The revised estimator $\hat\delta^{\op{rev}}$ involves  estimating the density ratio $\lambda^*$ and  some additional nuisance parameters $\pi_N, \nu^*_0, \nu_1^*$, which according to \cref{eq: prop-N} is equivalent to estimating the labeling propensity score $\rt_N$. So the first product rate conditions in  statements $1$ and $2$ play equivalent roles. Since the estimators $\hat\nu_1, \hat\nu_0, \hat \pi_N$ only need to estimate certain probabilities by sample frequencies, their convergence rates are typically already $O_p(\bar{N}_l^{-1/2})$, so the primary requirement of the conditions is $\rho_{\bar{N}_l, \lambda}\rho_{\bar{N}_l, {\tmu}} = o(\bar{N}_l^{-1/2})$. 
  Notably, the second and third conditions in  both statements that involve the treatment propensity score error rate $\rho_{N, e}$ should be  easy to hold,  because the estimators $\hek$'s are constructed using the full sample data of size $N$ and thus should achieve a convergence rate much faster than $\bar{N}_l^{-1/2}$. Therefore, the major bottlenecks are  the first conditions related to $\rho_{\bar{N}_l, r}$ and $\rho_{\bar{N}_l, \lambda}$, which we will further discuss in \Cref{sec: nuisance-vanish}. 

  \Cref{thm: normality2} validates the efficiency analysis in \Cref{thm: vanishing-if}: when $N_l / N \to 0$, our estimator achieves the efficiency bound in \cref{thm: vanishing-if} assuming a known unconditional distribution of $(T, X, S)$. 
  This justifies our intuition that the unconditional distribution of $(T, X, S)$ can be viewed as known from the perspective of much smaller labelled data.  
These results thus reveal the whole spectrum of efficiency in estimating ATE with surrogates, and feature a smooth transition from the regime $N_l \asymp N_u$ to the regime $N_l \ll N_u$.
In Appendix \Cref{thm: efficiency-comp-extension}, we further extend \Cref{thm: efficiency-comparison,corollary: efficiency-loss} to the regime $N_l \ll N_u$,
and we again confirm that more predictive surrogates result in bigger efficiency gains.

While the estimators $\hat\delta$ and $\hat\delta^{\op{rev}}$ have desirable asymptotic properties under suitable conditions, they involve estimating a vanishing propensity score or a density ratio function, which poses a new challenge in the current $N_l \ll N_u$ regime. Before delving into the details of estimating these quantities, we remark that this problem does not occur if the primary outcome is MCAR.  In this special setting, $R$ is independent with all other variables, so the labeling propensity score $\rt_N$ is identical to the marginal probability $\pi_N$ and thus can be easily estimated by $\hat \pi_N = N_l/N$. Accordingly, the density ratio $\lambda^*$ and the parameters $\nu_1^*, \nu_0^*$ are always $1$ so they do  not need estimation. Consequentally, the estimators $\hat\delta, \hat\delta^{\op{rev}}$ are equivalent and they reduce to the following form: 
\begin{align*}
    \hat{\delta}^{\op{rev}} 
    =\frac{1}{K}\sum_{k = 1}^K\expect_{k}\bigg\{&\hmuk(1, X) - \hmuk(0, X) +  \frac{T - \hek(X)}{\hek(X)\prns{1 - \hek(X)}}(\htmuk(T, X, S) - \hmuk(T , X))\nonumber \\
    &+  \frac{TR}{\hek(X)\hat{\pi}_N}(Y - \htmuk(1, X, S)) - \frac{(1 - T)R}{(1 - \hek(X))\hat{\pi}_N}(Y - \htmuk(0, X, S))\bigg\}.
\end{align*}
In \Cref{sec: connection}, we further discuss how this estimator connects to those in \cite{cheng2018efficient,zhang2019high}, and how our result generalizes those in previous literature.

\subsection{Nuisance Estimation}\label{sec: nuisance-vanish}
{The previous subsection proposes to use the ATE estimators $\hat\delta$ in \Cref{def: estimator} and $\hat\delta^{\op{rev}}$ in \Cref{def: estimator2} and derive their asymptotic properties in the $N_l \ll N_u$ regime. However, these two estimators need to first estimate a vanishing propensity score function $\rt_N(T, X, S)$ and a density ratio function $\lambda_N^*(S, X, T)$ respectively.
The estimation of propensity score $\rt_N(T, X, S)$ involves running binary regressions with very imblanaced label data. The estimation of density ratio $\lambda^*_N$ is typically even  harder to implement than running  regressions \cite[e.g., ][]{sugiyama2012machine,sugiyama2012density}, and it also faces the imbalanced data problem. 
In this subsection, we tackle the estimation of these two functions in the $N_l \ll N_u$ regime by 
extending the offset logistic regression method  of \cite{zhang2021double}.}

One of the most widely used approaches to  propensity score estimation is logistic regression, which models the log odds ratio of the propensity score. Under our observation model, the log odds ratio of the labeling propensity score $\rt_N$ in \Cref{eq: prop-N} can be written as 
\begin{align*}
\log\prns{\frac{\rt_N(t, X, S)}{1 - \rt_N(t, X, S)}} = -\log\prns{\lambda^*_0(S, X, t)} + \log\prns{\frac{\pi_{N, t}}{1-\pi_{N, t}}}, ~~ t\in\{0, 1\}.
\end{align*}
This formulation connects the estimation of the labeling propensity score and the estimation of density ratio. It also motivates the following offset logistic regression model: 
\begin{align*}
 \log\prns{\frac{\rt_N(t, X, S)}{1 - \rt_N(t, X, S)}} = \omega_t^\top X + \beta_t^\top S + \log\prns{\frac{\pi_{N, t}}{1-\pi_{N, t}}}, ~~ t\in\{0, 1\}, 
 \end{align*} 
where the linear function $\omega_t^\top X + \beta_t^\top S$ for some unknown coefficient vectors $\omega_t, \beta_t$ can be viewed as a model for the negative logarithm of the density ratio function $\lambda_0^*$. This linear function together with the offset term\footnote{{Directly following \cite{zhang2021double} would use $\log(\pi_{N, t})$ as the offset term, but we use $\log(\frac{\pi_{N, t}}{1-\pi_{N, t}})$ so the linear function $\omega_t^\top X + \beta_t^\top S$ more directly models  the density ratio.}} $\log({\pi_{N, t}}/({1-\pi_{N, t}}))$ provides a model for the log odds ratio of the propensity score.
In this model, the offset term accounts for the vanishing labeling probability, while the linear function part models the density ratio that remains invariant regardless of the sample size.  
In practice, the  probability  $\pi_{N, t}$ in the offset term is unknown but it can be easily estimated by $\hat \pi_{N, t} = \sum_{i=1}^N \indic{R_i=1, T_i=t}/\sum_{i=1}^N\indic{T_i=t}$. Then the unknown  coefficients $\omega_t, \beta_t$ can be estimated by running a logistic regression\footnote{{For example, the $\texttt{glm}$ function used to fit logistic regressions in the $\texttt{R}$ language can take an $\texttt{offset}$ term whose coefficient is coerced to be $1$.}} with an additional offset term $\log(\hat \pi_{N, t}/(1-\hat \pi_{N, t}))$ of a coefficient $1$.
Once we obtain  coefficient estimators $\hat\omega_t$ and $\hat\beta_t$, we can follow \cref{eq: lambda*,eq: prop-N} to  construct estimators for the density ratio function $\lambda^*_0$ and $\lambda^*$ and the labeling   
 propensity score function $\rt_N$:
\begin{equation}\begin{aligned}\label{eq: offset-estimator}
\hat\lambda_0(S, X, t) &= \exp(-\hat \omega_t^\top X - \hat\beta_t^\top S),  ~
 \hat \lambda(S, X, t) = (1-\hat \pi_{N, t})\hat\lambda_0(S, X, t) + \hat \pi_{N, t},  \\ 
\hat r(t, X, S) &= {\hat \pi_{N, t}}/{\hat \lambda(S, X, t)}.
\end{aligned}
\end{equation}

\cite{zhang2021double} provides a general theory for the offset logistic regression estimator with both low dimensional  and high dimensional covariates, assuming correct model specification. 
We can directly apply their theory to the estimation of labeling propensity score. 
For low dimensional $(X, S)$, their 
 theorem 4.1 guarantees that the offset logistic regression coefficient estimator converges to the truth at a $O(\bar{N}_l^{-1/2})$ rate. For high dimensional $(X, S)$, we may assume that the true coefficients are sparse and impose an additional lasso regularization when fitting the offset logistic regression. According to their Theorem 4.2, the resulting coefficient estimator can converge to truth at a $O((s\log d/\bar{N}_l)^{1/2})$ rate, where $s$ is the total number of nonzero coefficients and $d$ is the total dimension of $X$ and $S$. The convergence rates of the coefficient estimators then easily translate into the convergence rates of the resulting propensity score estimator and density ratio estimator. In other words, the theory in \cite{zhang2021double} shows that $\rho_{\bar{N}_l, r}, \rho_{\bar{N}_l, \lambda}$ are $O(\bar{N}_l^{-1/2})$  in the low dimensional regime and $O((s\log d/\bar{N}_l)^{1/2})$ in the high dimensional sparse regime.
In both cases, the conditions $\rho_{\bar{N}_l, r}\rho_{\bar{N}_l, \tilde{\mu}} = o(\bar{N}_l^{-1/2})$ and $\rho_{\bar{N}_l, \lambda}\rho_{\bar{N}_l, \tilde{\mu}} = o(\bar{N}_l^{-1/2})$  in \Cref{thm: normality2} are plausible as long as the regression estimators $\htmuk, k = 1, \dots, K$ are consistent (and the sparsity level $s$ is moderate).

  While the offset logistic regression in \cite{zhang2021double} is a parametric regression model, we consider extending to accommodate more general nonlinear function classes. Specifically, we can model the negative logarithm of density ratio $-\log \lambda^*_0(S, X, t)$ by a  function $f_t(S, X)$ within a general function class $\mathcal{F}$. This gives the following  offset model:
  \begin{align}\label{eq: offset-general}
    \log\prns{\frac{\rt_N(t, X, S)}{1 - \rt_N(t, X, S)}} = f_t(S, X) + \log\prns{\frac{\pi_{N, t}}{1-\pi_{N, t}}}, ~~ t\in\{0, 1\}. 
    \end{align}
    
    Then we can fit this model by maximizing the corresponding likelihood, which is equivalent to fitting a binary regression with an offset term by minimizing the cross-entropy loss. 
    The offset logistic regression is a special example with $\mathcal{F}$ as a class of linear functions  of $S$ and $X$. 
    We may also consider other more flexible nonlinear function classes, such as the linear sieves \citep{chen2007large}, reproducing kernel Hilbert space \citep{smola1998learning}, boosted trees \citep{friedman2001greedy}, neural networks \citep{goodfellow2016deep} and so on.
     These function classes may be more expressive and can  better approximate the (log) density ratio function. 
    Theoretically analyzing the convergence rates of the resulting estimators with general function approximation is an interesting theoretical question for future study. In \Cref{sec: simulation}, we empirically test the class of boosted trees by fitting an offset gradient boosted machine and demonstrate that it can achieve reasonably good performance.   

  Finally, we remark that when using the offset regression above to estimate the labeling propensity score and the density ratio, the resulting estimator $\hat\delta$ in \Cref{def: estimator} and estimator $\hat\delta^{\op{rev}}$ in \Cref{def: estimator2} become identical. Indeed, for the labeling propensity score estimator in \cref{eq: offset-estimator}, we have $1/\hat r(t, X, S) = {\hat\lambda(S, X, t)}/{\hat \pi_{N, t}}$.
  So the resulting estimator $\hat\delta$ is equivalent to the estimator $\hat\delta^{\op{rev}}$ where we use $\hat\lambda(S, X, t)$ to estimate the density ratio function $\lambda^*_N$ and $\hat \pi_N/\hat \pi_{N, t}$ to estimate the parameter\footnote{By Bayes' rule, we can easily show that $\nu_t^* = \pi_N/\pi_{N, t}$, so $\hat \pi_N/\hat \pi_{N, t}$ is a reasonable estimator for $\nu_t^*$.} $\nu_t^*$. This shows the close connection between estimating density ratio and estimating inverse propensity score.

\section{Numerical Studies}\label{sec: empirics}

{
In this section, we demonstrate the performance of our proposed estimators using both real data and numerical simulations. \Cref{sec: numerics-real} uses a real world dataset and focuses on validating the results in the $N_l \asymp N_u$ setting, while \Cref{sec: simulation} uses extensive simulations to validate the results in the $N_l \ll N_u$ setting. 
}.

\subsection{Real-Data Experiment}\label{sec: numerics-real}
In this part, we use experimental data for the Greater
Avenues to Independence (GAIN) job training program, a  job assistance
program designed in the late 1980s for low-income people  in California.
We employ the dataset analyzed in \cite{athey2016estimating}, which  contains results from a large-scale randomized experiment in four counties in California (Alameda, Los Angeles, and San Diego and Riverside). 
For
each experiment participant, this dataset records a binary treatment variable indicating whether
being treated by the GAIN program or not, quarterly earnings after treatment
assignments, and other covariate information.
We use the Riverside data to illustrate the performance of our proposed estimator and other benchmarks in estimating the average treatment effect in long-term earnings. 
We provide additional results for Los Angeles data and San Diego data in \Cref{sec: numeric-more} 
(Alameda dataset is very small and thus omitted).

We construct a labelled dataset and unlabelled dataset based on 
the Riverside data ($N = 5445$, with $4405$ treated units and $1040$ control units). 
We draw a fraction $r \in \braces{0.1, 0.3, 0.5}$ of units from the Riverside data {completely at random} as the labelled data ($R=1$) and use the rest as the unlabelled data ($R=0$). 
In our analysis, the primary outcome is the long-term earning in the $36$th quarter after the treatment assignment, and  surrogates are quarterly earnings up to $s$ quarters after the treatment, where $s \in \braces{8, 16, 24, 32}$.
We also consider additional covariates including age, gender, education, and ethnicity. 
The surrogates and covariates are observed on both labelled and unlabelled data, but the primary outcome is only observed on the labelled data. 

\begin{figure}
\centering 
\includegraphics[width=\textwidth]{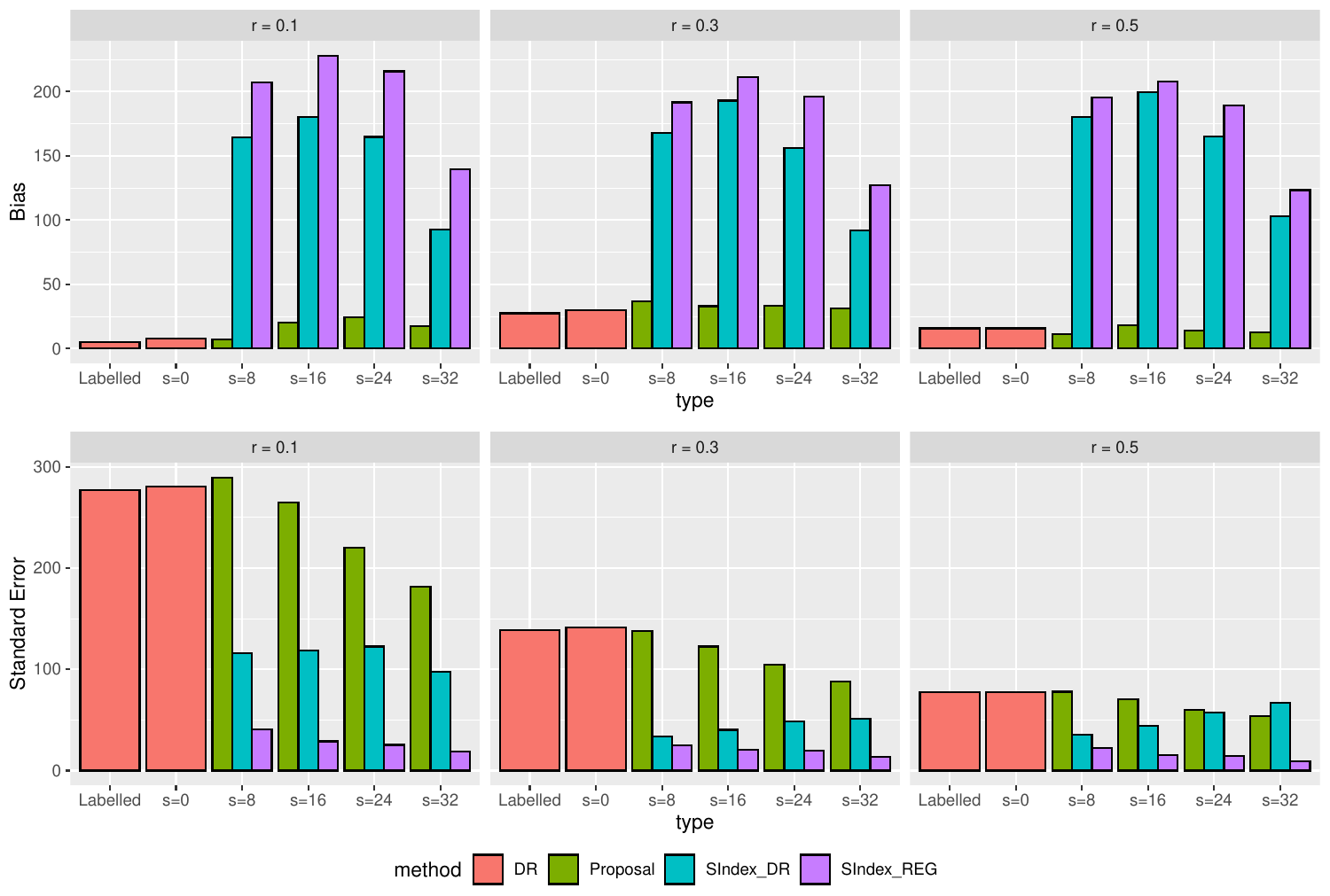}
\caption{Bias and standard error of different estimators over $120$ repetitions of experiments based on Riverside data. All nuisances are estimated by random forests.}
\label{fig: error-river-rf}
\end{figure}

\begin{figure}
\centering 
\includegraphics[width=\textwidth]{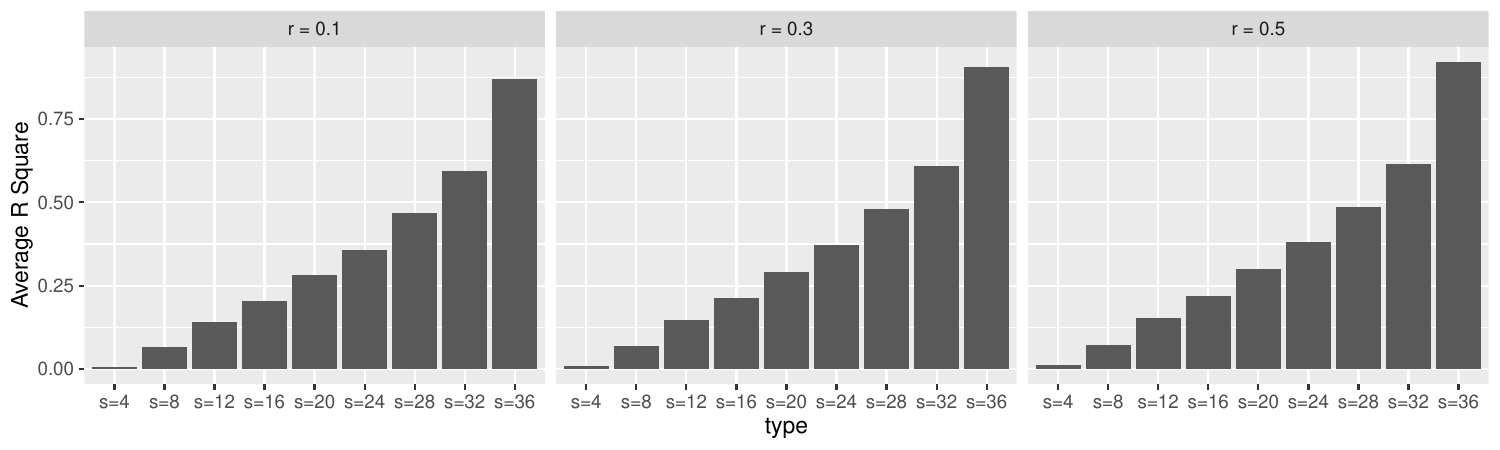}
\caption{Cross-validated R squares of random forest regressions with surrogates relative to baseline random forest regressions with only covariates, both restricted to the treated units. The R squares are averaged over $120$ repetitions of the experiments. 
Results for the control units are very similar and thus omitted. }
\label{fig: R-square}
\end{figure}

We apply five types of estimators to  
estimate the average treatment effect in the primary outcome: our proposed estimator in \Cref{def: estimator} (denoted as ``Proposal''), the surrogate index estimator based on regression imputation proposed in \cite{athey2016estimating} (denoted as ``SIndex\_REG''), the semiparametrically efficient surrogate index estimator proposed in \cite{chen2021semiparametric} (denoted as ``SIndex\_DR''), the doubly robust estimator that uses both datasets but ignores the surrogates, corresponding to the influence function $\psi_I$ in \Cref{thm: efficiency-comparison}  (denoted as ``DR'' with type $s = 0$), and the doubly robust estimator that only uses the labelled dataset (denoted as ``DR'' with type ``Labelled'').
We estimate the nuisance functions in these estimators by fitting random forests, gradient boosting, and LASSO respectively, all cross-fitted with $K = 5$ folds. 
The SIndex\_DR estimator is implemented by the longterm R package developed by \cite{chen2021semiparametric}, where all hyperparameters in nuisance estimation are automatically tuned by cross-validation. 
All other estimators (including our proposal) are implemented using R packages ranger (for fitting random forests), gbm (for fitting gradient boosting), and glmnet (for fitting  LASSO). The hyperparameters in random forests and boosting are set as default values without any tuning, and those in LASSO are tuned by cross-validation.
Note that in our data generating process, both the treatment assignments and the primary outcome missingness are completely at random, so we directly use sample frequencies as estimates for the corresponding propensity nuisances.
These propensity score estimates are trivially consistent. 
In this section, we focus on estimation results with random forest nuisance estimators. 
Additional results for other nuisance estimators are provided in the appendix.

The upper row in \Cref{fig: error-river-rf} shows the bias  of different treatment effect estimators with nuisances estimated by random forests, over $120$ repetitions of the experiments on the Riverside data.  
We observe that the surrogate index  estimators always have high bias. 
This bias is not primarily caused by nuisance estimation, since according to the theory in \cite{chen2021semiparametric}, the efficient surrogate index estimator (SIndex\_DR) should be robust to the nuisance estimation bias given that the propensities are well estimated (its bias is indeed lower than the bias of SIndex\_Reg). 
Instead, the high bias may be plausibly attributed to  the violation of the statistical surrogacy condition assumed for the surrogate index methods. 
As the number of surrogates grows, the violation of the statistical surrogacy condition is alleviated, so the bias of surrogate index estimators drops, but it 
still remains very high, posing serious challenges for statistical inference. 
In contrast, our proposed estimator has very low bias across all settings.

The lower row in \Cref{fig: error-river-rf} shows the standard errors of different  estimators.
We note that the standard errors of our proposed estimator decrease with the number of surrogates.
This is because with more surrogates we can better predict the primary outcome, as we validate in \Cref{fig: R-square}.
Moreover, the amount of standard error reduction due to introducing surrogates is overall higher when the missingness of the primary outcome is more severe (smaller $r$). 
These empirical observations support our qualitative conclusions from the efficiency analysis in \Cref{corollary: efficiency-loss}: the efficiency gains from leveraging surrogates improve for more predictive  surrogates and more missing primary outcome. 
Interestingly, the standard errors of surrogate index estimators may not decrease with the number of surrogates. 

In \cref{sec: numeric-more}, we provide additional results for other types of nuisance estimators (Gradient Boosting and LASSO) and  results for Los Angeles data and San Diego data.
These additional results show similar patterns of bias and standard error.

\subsection{Simulation Experiments}\label{sec: simulation}

  In this part, we simulate the very-large-unlabeled-data setting described in \Cref{sec: extension}. Specifically, we generate samples of total size $N = \{2000, 4000, 8000, 16000, 32000, 64000\}$, and randomly draw a vanishing proportion $\pi_N = N^{-1/4}$ of them as the labeled data $(R = 1)$ while viewing the remaining $1-\pi_N$ of them as the unlabeled data $(R = 0)$.

  We simulate covariates $X \in \R{6}$ for the labeled and unlabeled subsamples according to two different multivariate normal distributions:
  \begin{align*}
  X \mid R = 1 \sim \mathcal{N}\prns{\mu_{1}, \sigma^2_{1}I_{6 \times 6}},
   ~~ X \mid R = 0 \sim \mathcal{N}\prns{\mu_{2}, \sigma^2_{2}I_{6 \times 6}},
  \end{align*}
  where $\mu_{1} = (1, 1, 1, 1, 1, 1)^\top$, $\mu_{2} = (1/2, 1/2, 1/2, 3/2, 3/2, 3/2)^\top$, $\sigma_{1}^2 = 1$,  $\sigma_{2}^2 = 1/2$ and $I_{6 \times 6}$ is a $6 \times 6$ identity matrix. 
  The treatment variable is sampled according to a logistic regression model: given $X = x$, the probability of $T = 1$ is given by $1/(1+\exp({\sum_{j=1}^6 \eta_j x_j}))$, where $(\eta_1, \dots, \eta_6) = (1, -1/2, -1/2, -1/2, -1/2, 1)$. 
  We also simulate  potential surrogate outcomes $S(0) \in \R{5}, S(1) \in \R{5}$ from two different normal distributions: for $j = 1, \dots, 5$ and $t \in \{0, 1\}$, $S_j(t)$'s are independently drawn from the distribution $\mathcal{N}((-1)^{t+1}, 1)$. Furthermore, we simulate the potential target outcome $Y(1) \in \R{}, Y(0) \in \R{}$ as follows:
  \begin{align*}
  &Y(t) = (-1)^{t+1} +  \frac{(-1)^t}{2}\frac{\sum_{j=1}^5 S_j(t)}{5} + \sum_{j=1}^6 \alpha_j X_j^2 + \sum_{j=1}^6 \beta_j X_j  + \varepsilon, ~~ \varepsilon \sim \mathcal{N}(0, 1), \\
  \text{where } & (\alpha_1, \alpha_2, \alpha_3, \alpha_4, \alpha_5, \alpha_6) = (1, 0, 1, 0, 1, 0) \text{ and } (\beta_1, \beta_2, \beta_3, \beta_4, \beta_5, \beta_6) = (0, 1, 0, 1, 0, 1).
  \end{align*}
  We observe  $(X, R, T)$ and $S = S(T)$ on both the labeled and unlabeled data, while observing $Y = Y(T)$ only on the labeled data. It is easy to verify that this simulation corresponds to a missing-at-random setting where the missing indicator $R$ is dependent with the covariates $X$, but conditionally independent with all other variables given $X$. The density ratio $\lambda^*_0$ is a second order polynomial function of $X$, and the induced labeling propensity score $\rt_N$ satisfies the offset logistic regression model in \Cref{sec: nuisance-vanish}.

  We carry out $5$-fold cross-fitting estimation of the nuisance functions   with two different types of models: parametric models (parametric)  and gradient boosting (GB). For parametric models, we use second order polynomial   regressions to estimate $\tmut, \mut$, use linear logistic regression to estimate $\et_N$, and  use offset logistic regressions with second order polynomials to estimate $\rt_N, \lambda^*_0$. These parametric models are all correctly specified\footnote{We also tried misspecified linear models without the squared terms of $X$. We found them perform  much worse than the methods shown in this section due to model misspecification,  so we omitted them for brevity.}. For gradient boosting, we apply the $\texttt{gbm}$   package in $\texttt{R}$ to fit gradient boosted trees with $1000$ trees and learning rate $0.05$ to estimate $\tmut, \mut, \et_N$,
  and 
   use the same specification with a logit link and the offset in \cref{eq: offset-general} to estimate $\rt_N, \lambda^*_0$. The offset term can be easily incorporated, since the $\texttt{gbm}$ function in $\texttt{R}$ can also take an $\texttt{offset}$ term with coefficient $1$. 
   To estimate the long-term average treatment effect, 
    we apply the estimator $\hat\delta$ in \Cref{def: estimator} with these  parametric and GB nuisance  estimates, which is equivalent to the estimator $\hat\delta^{\op{rev}}$ in \Cref{def: estimator2} according to our discussion at the end of \Cref{sec: nuisance-vanish}. For reference, we also calculate the values of $\hat\delta$ with the true values of all nuisances and refer to this as ``oracle''. 

  We first demonstrate that the offset logistic regression and its gradient boosting extension can indeed effectively estimate the labeling propensity score and the density ratio, when the relative proportion $\pi_N$  of the  labeled data vanishes but its absolute size grows. In \Cref{fig: vanish-nuisance-error}, we show  the five-fold cross-validation errors  of the offset logistic regression and offset gradient boosting over $1000$ replications of experiments. The error in estimating the labeling propensity score $\rt_N$ is measured by $\|\rt_N/\hat r - 1\|$ and the error in estimating the  density ratio $\lambda^*_0$ is measured by $\|\log \hat\lambda - \log \lambda^*_0\|$. We observe that both errors decrease as the  sample size $N$  grows. The offset logistic regression model is a correctly specified parametric model and  consistently achieves lower estimation errors, which agrees with the theory in \cite{zhang2021double}.  The offset gradient boosting, as a flexible nonparametric model, does not use the knowledge of the true functional form of nuiances. Its estimation errors are higher but still properly decrease with $N$.

  \Cref{table: ate-0.25} summarizes the results of ATE estimator $\hat\delta$ in \Cref{def: estimator} from $1000$ replications of the experiments, where the plug-in nuisance values are either the truth (oracle) or estimates given by parametric models (Parametric) and gradient boosting (GB). We observe that the biases of all estimators are very small, while the ATE estimators using estimated nuisance values  have higher standard deviations than that using the true nuisance values. 
  This means that nuisance estimation may  result in higher  variance in finite-sample ATE estimation. However, the difference drops with the sample size $N$, verifying that the impact of nuisance estimation is asymptotically negaligible. We also estimate the standard errors of the ATE estimators based on the efficient influence function in \cref{eq: vanishing-if} and the cross-fitted nuisance estimates, and construct the corresponding $95\%$ confidence intervals. \Cref{table: ate-0.25} reports the average length and the coverage frequency of the confidence intervals. We observe that all  coverage is close to the nominal level, showing that the efficient influence function in \cref{eq: vanishing-if} well characterizes the asymptotic behavior of the ATE estimator. 

  In \Cref{sec: numeric-more}, we further show results for $\pi_N = N^{-1/3}$  and $\pi_N = 2.5 N^{-1/2}$ respectively. The proportions of labeled data vanish at faster rates in these two settings\footnote{The scaling factor $2.5$ in $\pi_N = 2.5 N^{-1/2}$ is set merely to ensure the existence of at least $100$ labeled data points  to fit  gradient boosting. }, resulting in smaller labeled data. As a result, the performance of all methods somewhat degrade. However, the qualitative conclusions remain the same.

\begin{figure}
\centering
\includegraphics[width=\textwidth]{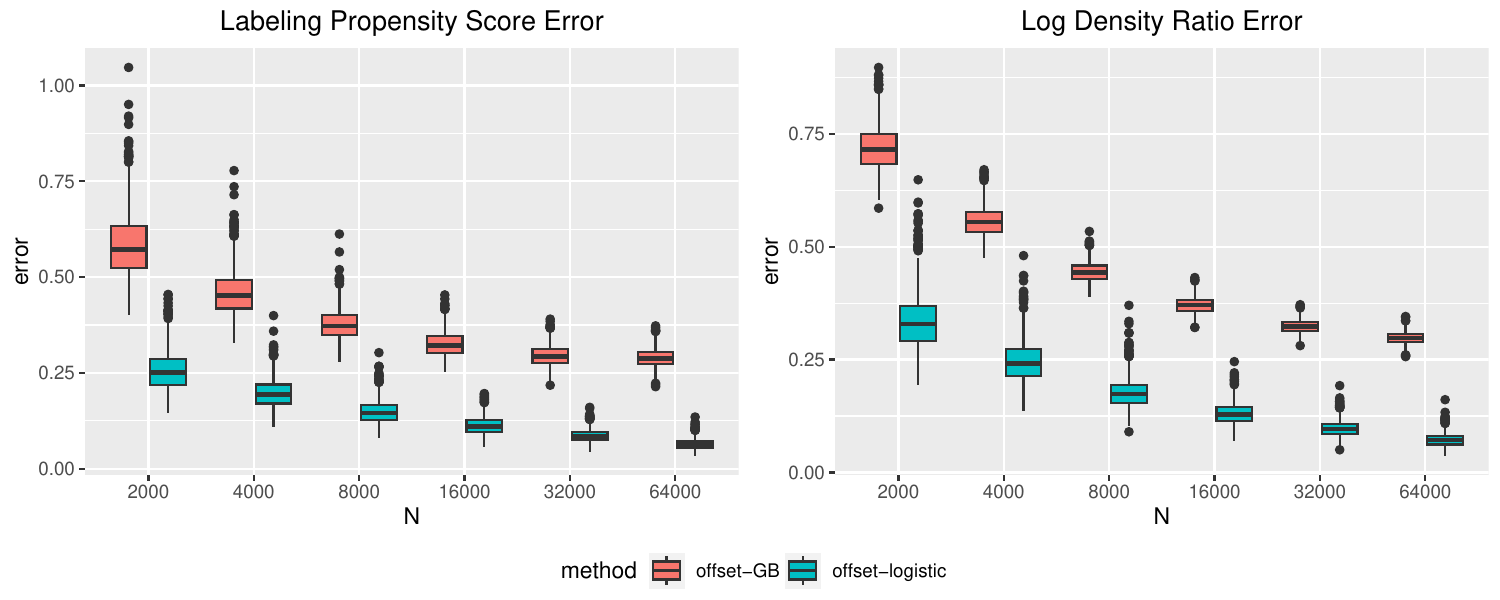}
\caption{Cross-validation errors in estimating the labeling propensity score $\rt_N$ and the logarithm of the density ratio $\lambda^*_0$ when the sample size $N$ growsand the proportion of labeled data is $\pi_N = N^{-1/4}$. The errors are based on $1000$ replications of the experiments.}\label{fig: vanish-nuisance-error}
\end{figure}

\begin{table}
\centering 
\begin{tabular}{cccccccc}
\toprule 
\multirow{2}{*}{Measure}  & \multirow{2}{*}{Nuisance Est.} & \multicolumn{6}{c}{ $N$}  \\
& & $2000$ & $4000$ & $8000$ & $16000$ & $32000$ & $64000$ \\
\cmidrule(lr){1-2}\cmidrule(lr){3-8}
\multirow{3}{*}{Bias} & Oracle & 0.0037 & 0.0041 & 0.0048 & 0.0014 & 0.0012 & 0.0048 \\
 & Parametric & 0.0028 &  0.0042 & 0.0035 & 0.0004 & 0.0013 & 0.0047 \\\vspace{0.25cm}
& GB & 0.0067 & 0.0258 & 0.0001 & 0.0032 & 0.0009 & 0.0008 \\
\multirow{3}{*}{Standard  Deviation} &Oracle & 0.2821 & 0.2283 & 0.1809 & 0.1429 & 0.1096 & 0.0900 \\
 & Parametric & 0.3275 & 0.2507 & 0.1891 & 0.1467 & 0.1105 & 0.0908 \\\vspace{0.25cm}
& GB &  0.5635 & 0.3669 & 0.2377 & 0.1695 & 0.1210 & 0.0937 \\
\multirow{3}{*}{CI Length} & Oracle & 1.0776 & 0.8842 & 0.6933 & 0.5516 & 0.4345 & 0.3395 \\
 & Parametric & 1.2303 & 0.9554 & 0.7208 & 0.5644 & 0.4396 & 0.3423 \\\vspace{0.25cm}
& GB & 2.0629 & 1.3704 & 0.9017 & 0.6435 & 0.4660 & 0.3459 \\
\multirow{3}{*}{CI Coverage} & Oracle &  0.959 & 0.959 & 0.957 & 0.948 & 0.963 & 0.940 \\
& Parametric & 0.943 & 0.945 & 0.950 & 0.944 & 0.963 &  0.940 \\
& GB & 0.943 & 0.946 & 0.945 & 0.942 & 0.953 & 0.942 \\
\bottomrule
\end{tabular}
\caption{Results of ATE estimation  with true nuisance values (oracle) or nuisances estimated by parametric models (Parametric) and gradient boosting (GB).}\label{table: ate-0.25}
\end{table}

\section{Conclusion}\label{sec: conclusion}
We study the estimation of average treatment effect with only a limited number of primary outcome observations but  abundant observations of surrogates. 
Particularly, we avoid stringent surrogacy conditions that are prone to violation in practice and only assume standard causal inference and missing data assumptions. 

We investigated the role of surrogates by comparing the efficiency lower  bounds of ATE  with and without presence of surrogates, and also bounds in some intermediary cases. 
We find that efficiency gains from optimally leveraging surrogates crucially depend on how well surrogates can predict the primary outcome and also the fraction of missing outcome data.
These results provide valuable insights on when leveraging surrogates can be beneficial. 
We also show that the efficiency results are valid in two regimes: when the size of surrogate observations is comparable to the size of primary-outcome observations   (i.e., $N_u \asymp N_l$), and when the former is much larger than the other (i.e., $N_u \gg N_l$).
The second regime violates the overlap condition commonly assumed in the literature and was thus understudied in the past, even though it is highly relevant in modern data collection. 
Our analysis shows that the second regime can be viewed as a limiting case of the first regime, which reveals the intimate connection between these two regimes. 

Moreover, we propose ATE estimators that can employ any flexible machine learning method for nuisance parameter estimation. We provide strong statistical guarantee for the proposed estimators by showing that they are robust to nuisance estimation bias, and they asymptotically achieve the semiparametric efficiency lower bounds under high-level rate conditions for the machine learning nuisance estimators. 
We further develop consistent estimators for the efficiency lower bounds and construct asymptotically valid confidence intervals for ATE. 
In summary, our methods provide a principled approach to optimally leverage surrogate observations when only a limited number of primary-outcome observations are available and without using strong surrogacy assumptions. 

\section*{Acknowledgments}
The authors thank
the Associate Editor and three anonymous reviewers for their insights and suggestions, which have led to significant improvement of this paper. \\

\noindent\emph{Conflict of interest}: We have no conflict of interest to disclose.

\section*{Funding}
Nathan Kallus acknowledges that this material is based upon work supported by the National Science Foundation under Grant No. 1846210. Xiaojie Mao is supported in part by National Natural Science Foundation of China (grant numbers 72201150, 72322001, and 72293561) and National Key
R\&D Program of China (grant number 2022ZD0116700).

\section*{Data availability}
The California GAIN dataset analyzed in \Cref{sec: numerics-real} contains sensitive individual data and cannot be shared publicly. It may be shared upon request. 
The data analyzed in \Cref{sec: simulation} are simulated according to the processes described in that section. The code script used to generate the simulated data is available at \url{https://github.com/CausalML/Efficient_estimation_surrogate}.

\bibliographystyle{plainnat}
\bibliography{semisupervised}

\newpage
\appendix

This appendix is organized as follows. In \Cref{sec: criteria}, we review the statistical surrogacy criterion in \cite{prentice1989surrogate} and discuss its limitations. In \Cref{sec: connection}, we compare our paper with some existing literature in terms of the assumptions and estimation methods. 
\Cref{sec: missing-pattern} extends the efficiency comparisons in 
\Cref{sec: efficiency-settings} by analyzing some additional missing data patterns. \Cref{sec: more-extension} provides some supplementary materials for \Cref{sec: more-extension}. In particular, it extends the efficiency comparisons in \Cref{thm: efficiency-comparison} to the $N_l \ll N_u$ regime. It also studies the ATE on the unlabelled population in this regime. \Cref{sec: att} extends our theory to the average treatment effect on the treated parameter. All proofs are included in \Cref{sec: proof}. Finally, \Cref{sec: numeric-more} presents some additional experimental results related to \Cref{sec: numerics-real,sec: simulation}. 

\appendix 
\section{Statistical Surrogacy Condition}\label{sec: criteria}
In this section, we review the definition of statistical surrogacy condition proposed by \cite{prentice1989surrogate}. Throughout this section, we implicitly condition on pre-treatment variables $X$ in all distributional statements. For example, $Y \perp T \mid S$ stands for $Y \perp T \mid S, X$.

 \cite{prentice1989surrogate} suggested a valid surrogate $S$ satisfy that a test of the null of no effect of the treatment $T$ on surrogate $S$ should serve as a valid test of the null of no effect of treatment $T$ on outcome $Y$. They formalized this by the following ``statistical surrogate'' condition. 

\begin{definition}[Statistical Surrogate]\label{def: stat-surrogate}
$S$ is said to
be a surrogate for the effect of $T$ on $Y$ if (i) $Y \perp T \mid S$; (ii) $S$ and $Y$ are correlated.
\end{definition}

To justify this condition, \cite{prentice1989surrogate} considered a time-to-event primary outcome with surrogates sampled from a stochastic process.
For simplicity, we now adapt their argument to a single-time measurement case. 
Note that under the statistical surrogacy condition, we can easily show that  
\[
    F(y \mid t) = \int F(y \mid t, s)dF(s \mid t) =  \int F(y \mid s)dF(s \mid t),
\]
where $F(y \mid t), F(y \mid t, s), F(s \mid t)$ are conditional cumulative distribution functions for the corresponding random variables. This equation shows that under the statistical surrogacy condition, $T$ is dependent with $Y$ only if $T$ is dependent with $S$. 
See also \cite{freedman1992statistical} for a similar argument for binary outcome. 
However, this type of argument is based purely on the statistical relationship rather causal relationship among the treatment, surrogate, and the primary outcome. 
Thus, the causal implication of this argument is not immediately straightforward. 

In the language of causal diagram \citep{pearl2009causality}, the statistical surrogacy condition is often characterized by \Cref{fig:subim1} \citep{vanderweele2013surrogate,athey2016estimating}. In this diagram, $T$ has no direct effect on $Y$, and $S$ has an effect on $Y$. As a result, $T$ can have an effect on $Y$ only if $T$ has an effect on $S$. 
Also, no direct effect of $T$ on $Y$ implies that $T$ is independent of $Y$ given $S$, namely the condition (i) in the definition of statistical surrogate. 
However, this relationship may be invalidated by any unmeasured confounder between the surrogate and the primary outcome (i.e., the variable $U$ in \Cref{fig:subim2}):
since $S$ is a collider on the causal path $T \rightarrow S \leftarrow U \rightarrow Y$, conditioning on $S$ can induce spurious dependence between $T$ and $Y$,  even though there is no direct effect of $T$ on $Y$ \citep{elwert2014endogenous}.
In other words, no direct effect of the treatment $T$ on the primary outcome $Y$ does not necessarily ensure conditional independence between the treatment $T$ and primary outcome $Y$ given surrogates $S$, if there exists any unmeasured confounder between surrogates $S$ and the primary outcome $Y$.

The following proposition, adapted from Proposition 3 in \cite{athey2016estimating}, reiterates the implication of \Cref{fig:subim1} in language of potential outcomes, and further elucidates the causal assumptions underlying the statistical surrogacy condtion. We denote $Y(t, s)$ as the potential outcome that would have been realized if treatment $T$ had been set to $t$, and surrogate oucomes $S$ had been set to $s$.

\begin{proposition}\label{prop: sufficient-stat-surrogacy}
$S$ satisfies condition (i) in \Cref{def: stat-surrogate} if the following conditions hold:
\begin{enumerate}[label=(\roman*)]
\item $Y(t, s) = Y(t', s)$ for any $t, t' \in \{0, 1\}$ and $s \in \mathcal{S}$; \label{cond: no-direct};
\item $T \perp (Y(0, s), Y(1, s))_{s \in \mathcal{S}}$;  \label{cond: T-Y-unconfound}
\item $S(t) \perp \{Y(t, s)\}_{s \in \mathcal{S}} \mid T = t$ for any $t \in \{0, 1\}$. \label{cond: S-Y-unconfound}
\end{enumerate}
\end{proposition}

\Cref{prop: sufficient-stat-surrogacy} above shows that no direct effect of treatment on the primary outcome (condition \ref{cond: no-direct}), and no unmeasured confounding either between treatment and the primary outcome (condition \ref{cond: T-Y-unconfound}) or between surrogates and primary outcome (condition \ref{cond: S-Y-unconfound}) together ensure statistical surrogacy condition. 
Conditions \ref{cond: T-Y-unconfound}\ref{cond: S-Y-unconfound} are also commonly assumed in mediation analysis that aims to decompose the total effect of treatment $T$ into the direct effect not through post-treatment variable $S$ and the effect mediated by $S$ \citep[\eg, ][]{imai2011unpacking}.
Here condition \ref{cond: T-Y-unconfound} may be satisfied by design in randomized trials where the treatment assignment $T$ is under perfect control. 
However, surrogates $S$ and their relationship to the primary outcome are generally  not manipulatable, so \ref{cond: no-direct} and \ref{cond: S-Y-unconfound} are often (if not always) violated even in perfect randomized trials.  

\begin{figure}[t]
\begin{subfigure}{0.5\textwidth}
\includegraphics[width=0.9\linewidth, height=4cm]{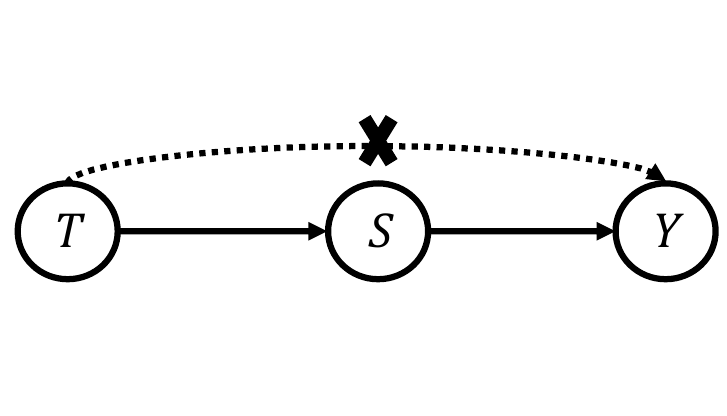} 
\caption{}
\label{fig:subim1}
\end{subfigure}
\begin{subfigure}{0.5\textwidth}
\includegraphics[width=0.9\linewidth, height=4cm]{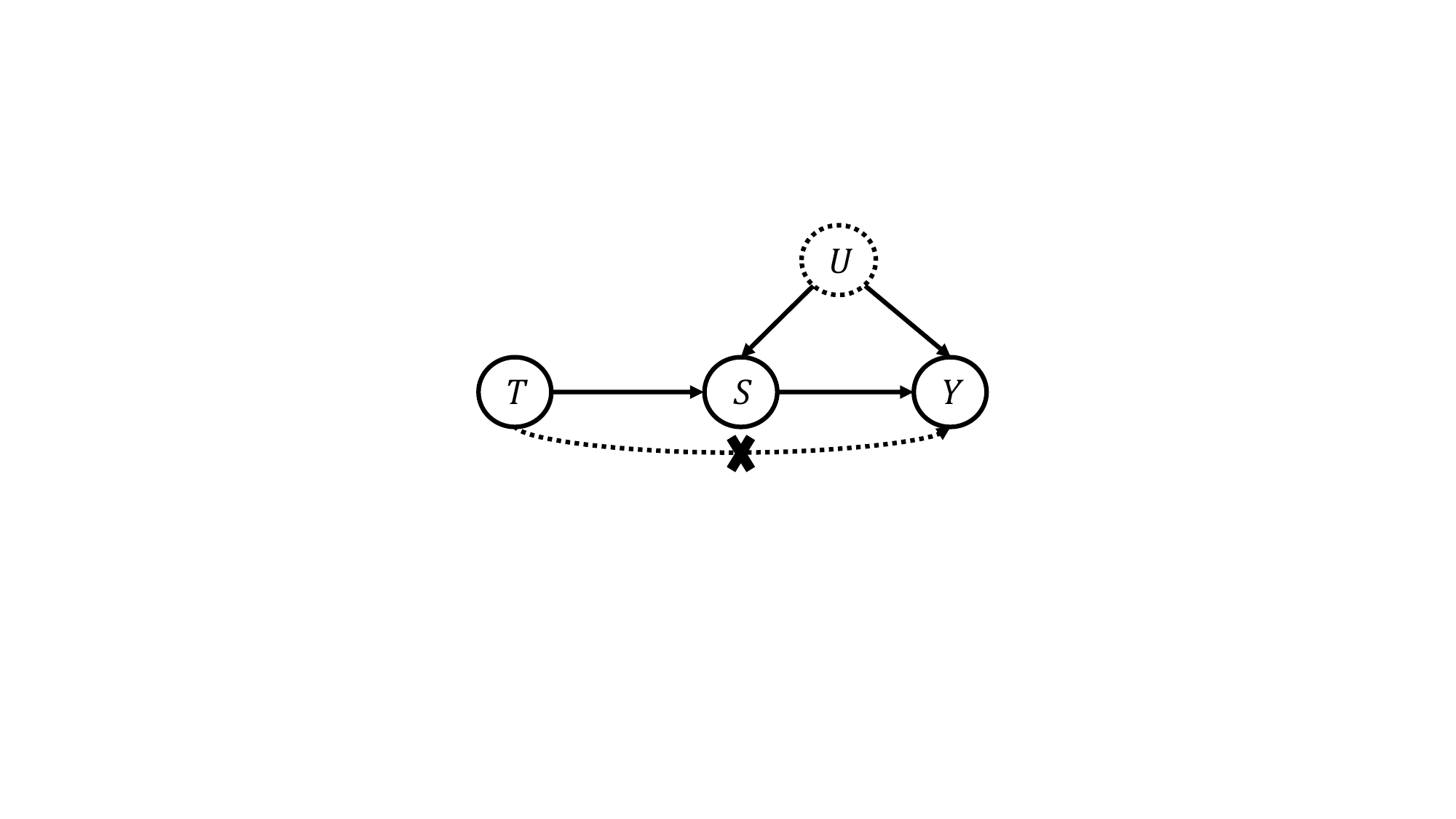}
\caption{}
\label{fig:subim2}
\end{subfigure}
\caption{Causal diagrams illustrating the statistical surrogacy condition: (a) statistical surrogacy condition holds; (b) statistical surrogacy condition can be violated in presence of unmeasured confounder $U$.}
\end{figure}

The discussions above also reveal that it is perhaps misleading to follow the quite common practice of interpreting statistical surrogates as variables that block all causal pathways between the treatment and primary outcome (i.e., no-direct-effect assumption characterized by condition \ref{cond: no-direct} in \Cref{prop: sufficient-stat-surrogacy}). Actually, the no-direct-effect condition is neither sufficient nor necessary for the conditional independence between the treatment and the primary outcome given surrogates (i.e., condition (i) in \Cref{def: stat-surrogate}), since there may exist unmeasured confounders between the surrogates and the primary outcome (i.e., condition \ref{cond: S-Y-unconfound} in \Cref{prop: sufficient-stat-surrogacy} is violated).  
For example, section 5.2 in \cite{frangakis2002principal} provide counter-examples 
to show that statistical surrogates may not satisfy no-direct-effect condition and vice versa.

\section{Comparisons with Previous Literature}\label{sec: connection}
\begin{table}
      \begin{subtable}[t]{.24\linewidth}%
    \centering%
        \begin{tabular}{|cccc|c|}
    \hline 
        $X$ & $T$ & $S$ & $Y$ & $R$ \\
    \hline 
        \checkmark & \checkmark & \checkmark & \checkmark & \multirow{3}{*}{$1$}   \\
        \vdots & \vdots & \vdots & \vdots &    \\
        \checkmark & \checkmark & \checkmark & \checkmark &    \\
    \hline 
         \checkmark & \checkmark & \checkmark & ? & \multirow{3}{*}{$0$}  \\
         \vdots & \vdots & \vdots & \vdots &   \\
         \checkmark & \checkmark & \checkmark & ? &   \\
    \hline 
    \end{tabular}
    \caption{Our paper and \cite{cheng2018efficient}.}\label{table: papers1}
  \end{subtable}
    \begin{subtable}[t]{.24\linewidth}%
    \centering%
    \begin{tabular}{|cccc|c|}
    \hline 
        $X$ & $T$ & $S$ & $Y$ & $R$ \\
    \hline 
        \checkmark & \checkmark & ? & \checkmark & \multirow{3}{*}{$1$}   \\
        \vdots & \vdots & \vdots & \vdots &    \\
        \checkmark & \checkmark & ? & \checkmark &    \\
    \hline 
         \checkmark & \checkmark & ? & ? & \multirow{3}{*}{$0$}  \\
         \vdots & \vdots & \vdots & \vdots &   \\
         \checkmark & \checkmark & ? & ? &   \\
    \hline 
    \end{tabular}
    \caption{\cite{zhang2019high}.}\label{table: papers2}
  \end{subtable}
  \begin{subtable}[t]{.24\linewidth}%
    \centering%
    \begin{tabular}{|cccc|c|}
    \hline 
        $X$ & $T$ & $S$ & $Y$ & $R$ \\
    \hline 
        \checkmark & ? & \checkmark & \checkmark & \multirow{3}{*}{$1$}   \\
        \vdots & \vdots & \vdots & \vdots &    \\
        \checkmark & ? & \checkmark & \checkmark &    \\
    \hline 
         \checkmark & \checkmark & \checkmark & ? & \multirow{3}{*}{$0$}  \\
         \vdots & \vdots & \vdots & \vdots &   \\
         \checkmark & \checkmark & \checkmark & ? &   \\
    \hline 
    \end{tabular}
    \caption{The setting for estimation in \cite{athey2016estimating}.}\label{table: papers3}
  \end{subtable}
    \begin{subtable}[t]{.24\linewidth}%
    \centering%
    \begin{tabular}{|cccc|c|}
    \hline 
        $X$ & $T$ & $S$ & $Y$ & $R$ \\
    \hline 
        \checkmark & \checkmark & \checkmark & \checkmark & \multirow{3}{*}{$1$}   \\
        \vdots & \vdots & \vdots & \vdots &    \\
        \checkmark & \checkmark & \checkmark & \checkmark &    \\
    \hline 
         \checkmark & \checkmark & \checkmark & \checkmark & \multirow{3}{*}{$0$} \\
         \vdots & \vdots & \vdots & \vdots &   \\
         \checkmark & \checkmark & \checkmark & \checkmark &   \\
    \hline 
    \end{tabular}
    \caption{The setting for the efficiency analysis in \cite{athey2016estimating}.}\label{table: papers4}
  \end{subtable} 
\caption{Illustrations for the observed data in our paper and \cite{cheng2018efficient,zhang2019high,athey2016estimating} respectively. Here ``\checkmark'' stands for an observed value, and ``?'' stands for a missing value.}
\end{table}

\subsection{Comparison with \cite{cheng2018efficient}}
\cite{cheng2018efficient} consider the same data configuration as our paper (\Cref{table: papers1}), but they assume that the primary outcome is missing completely at random. 

Recall that our estimator under MCAR setting reduces to 
\begin{align}\label{eq: DML-est3}
    \hat{\delta}^{\op{rev}}&=\frac{1}{K}\sum_{k = 1}^K\expectnk\bigg\{\frac{T}{\hek(X)}(\htmuk(1, X, S) - \hmuk(1 , X)) - \frac{1 - T}{1 - \hek(X)}(\htmuk(0, X, S) - \hmuk(0 , X))\nonumber \\
    &+ \hmuk(1, X) - \hmuk(0, X) + \frac{TR}{\hek(X)\hat{\pi}_N}(Y - \htmuk(1, X, S)) - \frac{(1 - T)R}{(1 - \hek(X))\hat{\pi}_N}(Y - \htmuk(0, X, S))\bigg\}.
\end{align}

This  estimator and the estimator in \cite{cheng2018efficient} both asymptotically achieve the efficiency lower bound  in \Cref{thm: vanishing-if} with $\lambda^*(X) = 1$. 
The estimator in \cite{cheng2018efficient} is valid only under MCAR setting,
while our estimator can be straightforwardly extended to MAR setting, if augmented with a density ratio estimator (\Cref{def: estimator2}).
Moreover, the estimator in \cite{cheng2018efficient} imposes parametric assumptions on the nuisances and relies on computationally intensive resampling methods to construct confidence intervals.  
In contrast, our estimator can leverage the power of any flexible machine learning nuisance estimator under generic rate conditions, and its confidence interval can be easily constructed using a straightforward plug-in estimator for standard errors (\Cref{thm: conf-interval}).
Furthermore, \cite{cheng2018efficient} focuses on the setting of $N_l \ll N_u$, while our analysis accommodates both $N_l \ll N_u$ and $N_l \asymp N_u$, and reveals the the whole spectrum of efficiency limits across two regimes. 

\subsection{Comparison with \cite{zhang2019high}} 
\cite{zhang2019high} focus on the efficiency improvement from unlabeled data, without studying possible efficiency gains from incorporating surrogates (\Cref{table: papers2}, or equivalently the setting I in \Cref{table: setting1}).

This setting can be viewed as a special case of our problem: we can view $S$ as an empty set of random variables and thus $\tmut(T, X, s) = \mut(T, X)$ for any $s \in \mathcal{S}$. Consequently, our estimator in \cref{eq: DML-est3} corresponding MCAR primary outcome reduces to the following form:  
\begin{align}
\label{eq: DML-est4}
&\frac{1}{N}\sum_{i = 1}^N \left[\hat{\mu}_{k(i)}(1, X_i) - \hat{\mu}_{k(i)}(0, X_i)\right] \nonumber \\
&\qquad\qquad+ \frac{1}{N_l}\sum_{i \in \mathcal{I}_l} \left[\frac{T_i}{\hat{e}_{k(i)}(X_i)}(Y_i - \hat{\mu}_{k(i)}(1, X_i)) - \frac{1 - T_i}{(1 - \hat{e}_{k(i)}(X_i))}(Y_i - \hat{\mu}_{k(i)}(0, X_i))\right],
\end{align}
where $k(i)$ is the fold that the $i$th observation belongs to. 
This estimator  recovers the semi-supervised ATE estimator in \cite{zhang2019high}.

\subsection{Comparison with \cite{athey2016estimating}}\label{sec: Athey}
In \cite{athey2016estimating}, they assumed the statistical surrogacy condition that $Y \perp T \mid X, S, R = 1$, namely the observed primary outcome and the treatment on the labeled data are independent given the pre-treatment covariates and surrogates.
This assumption is crucial for the identification of treatment effects in the  setting considered by \cite{athey2016estimating}: the treatment and primary outcome are observed on separate datasets, but surrogates are always observed (\Cref{table: papers3}). 
Their setting is different and more challenging than our setting: in our setting the treatment is always observed (\Cref{table: papers1}), but in their setting the treatment is missing on the labelled data.
Although the statistical surrogacy condition seems inevitable in their setting to fuse the two separate datasets without any complete observation, the causal assumptions underlying this statistical surrogacy condition may be too strong to hold in practice, as we discussed in  \Cref{sec: criteria}.

\section{{Different Missingness Patterns}}\label{sec: missing-pattern}
{In \Cref{sec: efficiency-settings}, we consider four different settings with increasing amount of observed information (see \Cref{table: four-settings}). In particular, 
in setting I 
the surrogate variables are completely missing, in setting II the  surrogate variables are observed if and only if the primary outcome is observed (i.e., the missingness patterns of the surrogate variables and the primary outcome are identical), in setting III the  surrogate variables are fully observed. 
Here we further consider two additional settings with partially observed surrogate variables: one is when the surrogate variables are observed only for a part of units whose primary outcome is observed (which can be viewed as an intermediate setting between setting I and setting II, thus named as setting I-II), and the other is when the surrogate variables are observed for units whose  primary outcome may  not be observed (which can be viewed as an intermediate setting between setting II and setting III, thus named as setting II-III). These two additional settings are illustrated in \Cref{table: additional-setting}(a) and \Cref{table: additional-setting}(b) respectively, where we introduce the variable $R_S$ to indicate the observation of the surrogate variables $S$. Obviously, we have $R = 1$ if $R_S = 1$ in setting I-II while $R_S = 1$ if $R = 1$ in setting II-III.
}

\begin{table}
    \begin{subtable}[t]{.48\linewidth}%
    \centering%
    \begin{tabular}{|cc|cc|cc|}
    \hline 
        $X$ & $T$ & $S$ & $R_S$ & $Y$ & $R$ \\
    \hline 
        \checkmark & \checkmark & \checkmark & 1 & \checkmark & 1 \\
        \vdots & \vdots & \vdots & \vdots & \vdots & \vdots  \\
        \checkmark & \checkmark & \checkmark & 1 & \checkmark & 1 \\
        \cline{3-4} 
        \checkmark & \checkmark & ? & 0 & \checkmark & 1 \\
        \vdots & \vdots & \vdots & \vdots & \vdots & \vdots  \\
        \checkmark & \checkmark & ? & 0 & \checkmark & 1 \\
         \cline{5-6} 
         \checkmark & \checkmark & ? & 0 & ? & 0 \\
         \vdots & \vdots & \vdots & \vdots & \vdots & \vdots  \\
         \checkmark & \checkmark & ? & 0 & ? & 0 \\
         \checkmark & \checkmark & ? & 0 & ? & 0 \\
         \vdots & \vdots & \vdots & \vdots & \vdots & \vdots  \\
         \checkmark & \checkmark & ? & 0 & ? & 0 \\
         \hline 
    \end{tabular}
    \caption{Setting I-II}\label{table: setting1}
  \end{subtable}
  \begin{subtable}[t]{.48\linewidth}%
    \centering%
    \begin{tabular}{|cc|cc|cc|}
    \hline 
        $X$ & $T$ & $S$ & $R_S$ & $Y$ & $R$ \\
    \hline 
        \checkmark & \checkmark & \checkmark & 1 & \checkmark & 1 \\
        \vdots & \vdots & \vdots & \vdots & \vdots & \vdots  \\
        \checkmark & \checkmark & \checkmark & 1 & \checkmark & 1 \\
        \checkmark & \checkmark & \checkmark & 1 & \checkmark & 1 \\
        \vdots & \vdots & \vdots & \vdots & \vdots & \vdots  \\
        \checkmark & \checkmark & \checkmark & 1 & \checkmark & 1 \\
         \cline{5-6} 
         \checkmark & \checkmark & \checkmark & 1 & ? & 0 \\
         \vdots & \vdots & \vdots & \vdots & \vdots & \vdots  \\
         \checkmark & \checkmark & \checkmark & 1 & ? & 0 \\
        \cline{3-4} 
         \checkmark & \checkmark & ? & 0 & ? & 0 \\
         \vdots & \vdots & \vdots & \vdots & \vdots & \vdots  \\
         \checkmark & \checkmark & ? & 0 & ? & 0 \\
         \hline 
    \end{tabular}
    \caption{Setting II-III}\label{table: setting2}
  \end{subtable}
\caption{Illustrations for the observed data in two additional settings. Here ``\checkmark'' stands for an observed value, and ``?'' stands for a missing value. 
}\label{table: additional-setting}
\end{table}

{To enable the use of surrogate variables in these two settings, we need to additionally assume that $R_S$ is also missing at random. Thus we further impose the following assumption  in addition to 
\Cref{assump: mar-1,assump: MAR2}. 
\begin{assumption}\label{assump: mar3}
Suppose that $R_S \perp S(t) \mid X, T$ for any $t = 0, 1$. 
\end{assumption}
}

{It is easy to verify that \Cref{assump: mar3} implies $R_S \perp S \mid X, T$. 
Below, we derive the efficiency bound for setting I-II and setting II-III respectively.}

{\begin{theorem}\label{thm: efficiency-intermediate}
Under assumptions in \Cref{thm: efficiency-comparison} and \Cref{assump: mar3}, the semiparametric efficiency bounds for $\delta^*$ under the setting I-II and the setting II-III are $V^*_{I-II} = \Eb{\psi^2_{\text{I-II}}(W; \delta^*, \eta^*)}$ and $V^*_{II-III} = \Eb{\psi^2_{\text{II-III}}(W; \delta^*, \eta^*)}$ respectively, where 
\begin{align*}
&\psi_{\text{I-II}}(W; \delta^*, \eta^*) = \psi_{\text{I}}(W; \delta^*, \eta^*) = \psi_{\text{II}}(W; \delta^*, \eta^*) \\
&\qquad = \mut(1, X) - \mut(0, X) - \delta^*  + \frac{TR}{\et(X)\rt(1, X)}(Y - \mut(1, X)) - \frac{(1 - T)R}{(1 - \et(X))\rt(0, X)}(Y - \mut(0, X)), \\
&\psi_{\text{II-III}}(W; \delta^*, \eta^*) = \mut(1, X) - \mut(0, X) - \delta^* \\
&\qquad + \frac{TR}{\et(X)\rt(1, X)}(Y - \tmut(1, X, S)) - \frac{(1-T)R}{(1-\et(X))\rt(0, X)}(Y - \tmut(0, X, S))  \\
&\qquad + \frac{TR_S}{\et(X)\rt_S(1, X)}\prns{\tmut(1, X, S) - \mut(1, X)}  - \frac{(1-T)R_S}{(1-\et(X))\rt_S(0, X)}\prns{\tmut(0, X, S) - \mut(0, X)},
\end{align*}
where $\tmut(t, x, S) = \Eb{Y \mid T = t, X = x, S = s, R = 1}$, $\mut(t, x) = \Eb{\tmut(T, X, S) \mid T = t, X, R_S = 1}$, $\rt(t, x) = \Prb{R = 1 \mid T = t, X = x}$, and $\rt_S(t, x) = \Prb{R_S = 1 \mid T = t, X = x}$. 
\end{theorem}
}

{From the theorem above, we can observe that the efficiency bounds in the settings I, II and the setting I-II are all the identical. This further supports our conclusion in \cref{sec: efficiency-settings} that observing surrogate variables only when the primary outcome is already observed cannot improve any efficiency. 
Below, we further compare the efficiency bound in setting II-III with those in settings II and III respectively, in order to demonstrate the benefit of observing surrogate variables when the primary outcome is not observed. 
}

{\begin{theorem}\label{thm: efficiency-comparison-intermediate}
  Under the assumptions in \Cref{thm: efficiency-intermediate}, we have 
  \begin{align*}
    &V^*_{II} - V^*_{II-III} =  \Eb{\frac{\rt_S(1, X) - \rt(1, X)}{\et(X)\rt(1, X)\rt_S(1, X)}\op{Var}\bracks{\tmut(1, X, S(1)) \mid X}} \\
        &\qquad\qquad\qquad\qquad\qquad + \Eb{\frac{\rt_S(0, X) - \rt(0, X)}{(1 - \et(X))\rt(0, X)\rt_S(0, X)}\op{Var}\bracks{\tmut(0, X, S(0)) \mid X}}, \\
    &V^*_{II-III} - V^*_{III} = \Eb{\frac{1 - \rt_S(1, X)}{\et(X)\rt_S(1, X)}\op{Var}\bracks{\tmut(1, X, S(1)) \mid X}} \\
        &\qquad\qquad\qquad\qquad\qquad + \Eb{\frac{1 - \rt_S(0, X)}{(1 - \et(X))\rt_S(0, X)}\op{Var}\bracks{\tmut(0, X, S(0)) \mid X}}. 
  \end{align*}
\end{theorem}
}

{Recall that compared to the setting II, the  setting II-III has more  surrogate observations. \Cref{thm: efficiency-comparison-intermediate} shows that the additional surrogate observations lead to larger efficiency gains when the surrogates are more predictive of the primary outcome (i.e., higher $\op{Var}\bracks{\tmut(1, X, S(1)) \mid X}$ and $\op{Var}\bracks{\tmut(0, X, S(0)) \mid X}$) or when more surrogate observations are available (i.e., higher $\rt_S(1, X)$ and $\rt_S(0, X)$). 
Moreover, \Cref{thm: efficiency-comparison-intermediate} establishes  the efficiency gap of the setting II-III relative to the setting III with fully observed surrogates.
This efficiency gap is larger when the surrogates are more predictive of the primary outcome, or when surrogates are more missing (i.e., lower $\rt_S(1, X)$ and $\rt_S(0, X)$). 
\Cref{thm: efficiency-comparison-intermediate} together with \Cref{corollary: efficiency-loss} characterizes the efficiency gains from different size of surrogate observations.  
}

\section{Supplements to \Cref{sec: extension}}\label{sec: more-extension}
\subsection{Regularity Assumption for \Cref{thm: normality2}}
In this part, we give a  supplementary assumption  for \Cref{thm: normality2}.

\begin{assumption}\label{assump: moment}
There exist positive constants $q > 2$ and $C$ such that for $t = 0, 1$,
\begin{align*}
\begin{array}{c}
\left\{\expect\left[|Y - \tmut(T, X, S)|^q\mid R = 1\right]\right\}^{1/q} \le C, ~~~  \|\lambda^*\|_q \le C,\\
\|\tmut(t, X, S) - \mut(t , X)\|_q \le C, ~~~ \|\mut(t , X)\|_q  \le C.
\end{array}
\end{align*}
\end{assumption}
Moment conditions in \Cref{assump: moment} are mild, and they are mainly used in verifying the Lyapunov condition in Lindberg-Feller Central Limit Theorem in the proof of \Cref{thm: normality2}.

\subsection{Efficiency Comparison}
We now provide the efficiency lower bounds for other settings in \Cref{sec: efficiency} when $N_l \ll N_u$. Note that setting IV is the ideal setting with fully labeled data, so the regime of $N_l \ll N_u$ degenerates.
Therefore, we only need to study setting I and setting II. {The following theorem extends \Cref{thm: efficiency-comparison}, which also assumes the additinal \Cref{assump: MAR2} to ensure the identification of $\delta^*$ in settings I and II.}

\begin{theorem}\label{thm: efficiency-comp-extension}
{Consider the following two settings:
\begin{enumerate}[label=\Roman*.]
\item We only observe the labeled data, i.e., i.i.d. samples from the conditional distribution of $(X, T, Y)$ given $R = 1$, and we know the unconditional distribution of $(X, T)$;
\item We only observe the labeled data, i.e., i.i.d. samples from the conditional distribution of $(X, T, S, Y)$ given $R = 1$, and we know the unconditional distribution of $(X, T)$;
\end{enumerate}
We further assume assumptions in \Cref{thm: vanishing-if} and \Cref{assump: MAR2}. 
Then the efficiency lower bounds for two settings above are $\tilde{V}^*_j = \expect[\tilde{\psi}^2_j(W; \delta^*, \tilde{\eta}^*) \mid R = 1]$ for $j = I, II$, where 
\begin{align*}
    \tilde{\psi}_I(W; \delta^*, \tilde{\eta}^*) = \tilde{\psi}_{II}(W; \delta^*, \tilde{\eta}^*)  = \frac{T\lambda^*(X, 1)}{\et(X)}(Y - \mut(1, X)) - \frac{(1 - T)\lambda^*(X, 0)}{1 - \et(X)}(Y - \mut(0, X)),
\end{align*}
and $\lambda^*(X, T) = f^*(X \mid T = t)/f^*(X \mid T = t, R = 1)$ is the density ratio function of the covariates $X$. Then the efficiency gains from surrogates are quantified by 
\begin{align*}
\tilde{V}^*_I - \tilde{V}^* &= \tilde{V}^*_{II} - \tilde{V}^*\\&= \expect\bigg[\frac{\lambda^{*2}(X, 1)}{\et(X)}\frac{\prns{\mathbb{P}\prns{T=1}}^2}{\prns{\mathbb{P}\prns{T=1 \mid R = 1}}^2}\var\{\tmut(1, X, S(1))\mid X\} \\
&\qquad\qquad + \frac{\lambda^{*2}(X, 0)}{1 - \et(X)}\frac{\prns{\mathbb{P}\prns{T=0}}^2}{\prns{\mathbb{P}\prns{T=0 \mid R = 1}}^2}\var\{\tmut(0, X, S(0))\mid X\} \mid R = 1\bigg].
\end{align*}
}
\end{theorem} 

\Cref{thm: efficiency-comp-extension} shows that the efficiency gains from surrogates increase with the variations of the primary outcome explained by the surrogates beyond the pre-treatment covariates, i.e., $\var\{\tmut(t, X, S(t))\mid X\}$ for $t = 0, 1$. This means that surrogates that are more predictive of the primary outcome can result in larger efficiency improvement, which is in line with the findings in \Cref{corollary: efficiency-loss}.

\subsection{Average Treatment Effect on the Unlabelled Population}\label{sec: effect-unlabeled-vanishing}
In \Cref{thm: effect-unlabelled}, we derived the efficiency lower bound for the ATE on the unlabelled population $\delta_0^*$ under the overlap condition in \Cref{assump: overlap}. In this part, we extend the theory to the setting with very large unlabeled data (i.e., $N_l \ll N_u$). 

The corollary below extends \Cref{thm: vanishing-if} to the parameter $\delta^*_0$. This corollary shows that $\delta^*_0$ and $\delta^*$ share the same semiparametric efficiency lower bounds.  Note that currently the  unlabelled
dataset dominates the combined dataset, so the average effects $\delta_0^*$ and $\delta^*$ on the unlabeled and combined population distributions become
identical in the limit. It is thus not surprising that they have the same semiparametric efficiency lower bounds. 

\begin{corollary}\label{corollary: effect-unlabelled-vanishing}
    Under the assumptions in \Cref{thm: vanishing-if}, the semiparametric efficiency lower bound for the average treatment effect parameter on the unlabelled population $\delta_0^*$ with respect to a known unconditional distribution of $(X, T, S)$  
is identical to the efficiency bound $\tilde V^*$ in \Cref{thm: vanishing-if}. 
\end{corollary}

Furthermore, the corollary below extend the \Cref{prop: efficient-bound-rescale}. It connects the efficiency bounds for $\delta_0^*$ when the the size of unlabelled data is much larger than the size of the labelled data and when their sizes are comparable. The bound in the former setting again can be viewed as the limit of the bound in the latter setting. 

\begin{corollary}\label{corollary: effect-unlabelled-vanishing-rescale}
    Let $V_0^*$ and $\tilde V^*$ be the semiparametric efficiency lower bounds given in \Cref{thm: effect-unlabelled} and \Cref{thm: vanishing-if} respectively. For any asymptotically efficient estimator $\hat\delta_0$ such that $\sqrt{N}(\hat\delta_0 - \delta^*_0) \overset{d}{\to} \mathcal{N}(0, V^*_0)$ as $N \to \infty$,  we have $\sqrt{N_l}(\hat\delta_0 - \delta^*_0) \overset{d}{\to} \mathcal{N}(0, \pr(R=1)V^*_0)$. Moreover, 
    \begin{align*}
    \pr(R = 1)V^*_0 
      &= \tilde V^* + \frac{\pr(R = 1)}{\pr(R = 0)}\expect\left[\prns{\mut_0(1, X) - \mut_0(0, X) - \delta^*_0}^2 \mid R = 0\right]  \\
        &+\frac{\pr(R = 1)}{\pr(R = 0)}\expect\left[\frac{T-\et(0,X)}{\et(0, X)(1-\et(0, X))}\prns{\tmut(T, X, S) - \mut_0(T, X)} \mid R = 0\right].
\end{align*}
\end{corollary}

\section{{Average Treatment Effect on the Treated (ATT)}}\label{sec: att}
{
    In the main text, we mainly focus on the average treatment effect over the whole population. In this part, we now consider the average treatment effect on the treated (ATT), namely, the average effect over the treated subpopulation: 
    \begin{align*}
    \delta^*_{\op{ATT}} = \Eb{Y(1) - Y(0) \mid T = 1}. 
    \end{align*}
}

{
    We can  identify this parameter under \Cref{assump: mar-1,assump: unconfound,assump: overlap} like in \Cref{lemma: identification-1}. 
\begin{lemma}\label{lemma: identification-att}
If \Cref{assump: mar-1,assump: unconfound,assump: overlap} hold, then 
\begin{align}\label{eq: identif-1}
\delta^*_{\op{ATT}} 
&= \Eb{\Eb{\Eb{Y \mid T = 1, R = 1, X, S}\mid X, T = 1}\mid T = 1} \nonumber \\
&\qquad\qquad- \Eb{\Eb{\Eb{Y \mid T = 0, R = 1, X, S}\mid X, T = 0} \mid T = 1}.
\end{align}
\end{lemma}
}

{
    We can further extend the efficiency result in \Cref{thm: efficient-if} for ATE to ATT. 
\begin{theorem}\label{thm: efficient-if-att}
Under the conditions in \Cref{thm: efficient-if}, the semiparametric efficiency lower bound for $\delta^*_{\op{ATT}}$ under  model $\mathcal M$ is $V^*_{\op{ATT}} = \expect[\psi^2(W; \delta^*_{\op{ATT}}, \eta^*) ]$ where 
\begin{align*}
&\psi_{\op{ATT}}(W; \delta^*_{\op{ATT}}, \eta^*) =   \frac{T}{\pr\prns{T = 1}}\prns{\tmut(1, X, S) - \mut(0, X) - \delta^*_{\op{ATT}}} + \frac{TR}{\pr\prns{T = 1}\rt(1, X, S)}\prns{Y - \tmut(1, X, S)} \\
&- \frac{\et(X)}{\pr\prns{T=1}}\frac{(1 - T)R}{(1 - \et(X))\rt(0, X, S)}(Y - \tmut(0, X, S)) - \frac{\et(X)}{\pr\prns{T=1}}\frac{1 - T}{1 - \et(X)}(\tmut(0, X, S) - \mut(0 , X)).
\end{align*} 
\end{theorem}
}

{
    Moreover, we can also consider the four settings described in \Cref{sec: efficiency-settings}, and derive the corresponding efficient lower bounds. 
\begin{theorem}\label{thm: efficiency-comparison-att}
Under the conditions in \Cref{thm: efficiency-comparison}, the efficiency lower bounds for $\delta^*_{\op{ATT}}$ in setting $j$ is $V^*_{{\op{ATT}}, j} = \expect[\psi^2_j(W; \delta^*_{\op{ATT}}, \eta^*)]$ for $j = \op{I}, ..., \op{IV}$, where 
\begin{align*}
\psi_{{\op{ATT}}, \op{I}}(W; \delta^*_{\op{ATT}}, \eta^*) &= \psi_{\op{ATT},\op{II}}(W; \delta^*_{\op{ATT}}, \eta^*) = \frac{T}{\pr\prns{T=1}}\prns{\mut(1, X) - \mut(0, X) - \delta^*_{\op{ATT}}}  \\
&+\frac{T}{\pr\prns{T=1}}\frac{R}{\rt(1, X)}\prns{Y - \mut(1, X)} - \frac{\et(X)}{\pr\prns{T=1}}\frac{(1 - T)R}{(1 - \et(X))\rt(0, X)}(Y - \mut(0, X)), \\
\psi_{\op{ATT}, \op{III}}(W; \delta^*_{\op{ATT}}, \eta^*) 
    &= \frac{T}{\pr\prns{T=1}}\prns{\tmut(1, X, S) - \mut(0, X) - \delta^*_{\op{ATT}}} + \frac{TR}{\pr\prns{T = 1}\rt(1, X)}\prns{Y - \tmut(1, X, S)}\\
    &- \frac{\et(X)}{\pr\prns{T=1}}\frac{(1 - T)R}{(1 - \et(X))\rt(0, X)}(Y - \tmut(0, X, S)) \\
    & - \frac{\et(X)}{\pr\prns{T=1}}\frac{1 - T}{1 - \et(X)}(\tmut(0, X, S) - \mut(0 , X)) \\
\psi_{\op{ATT}, \op{IV}}(W; \delta^*_{\op{ATT}}, \eta^*) 
    &=  \frac{T}{\pr\prns{T=1}}\prns{Y - \mut(0, X) - \delta^*_{\op{ATT}}} - \frac{\et(X)}{\pr\prns{T=1}}\frac{1 - T}{(1 - \et(X))}(Y - \mut(0, X)).
\end{align*}
\end{theorem}
}

{
In the following corollary, we further compare the efficiency bounds of the four different settings. The results are analogous to those in \Cref{corollary: efficiency-loss}, expect that they are now restricted to the treated subpopulation. This is not surprising because the target ATT parameter is restricted to the treated subpopulation.  
  \begin{corollary}\label{corollary:att-comparison}
  Under the conditions in \Cref{thm: efficiency-comparison},
  \begin{enumerate}
  \item The efficiency gain from observing the surrogates on all units is measured by 
\begin{align*}
&V_{\op{ATT}, \op{I}}^* - V^*_{\op{ATT}, \op{III}} = V_{\op{ATT}, \op{II}}^* - V^*_{\op{ATT}, \op{III}} \\ 
  =& \frac{1}{\Prb{T=1}}\Eb{\frac{1-\rt(1,X)}{\rt(1, X)}\var[\tmut(1, X, S(1))\mid X]+ \frac{\et(X)(1-\rt(0, X))}{\prns{1-\et(X)}\rt(0, X)}\var[\tmut(0, X, S(0))\mid X] \mid T = 1}.
  \end{align*}
\item The information loss due to not fully observing the primary outcome 
is measured by 
\begin{align*}
&V^*_{\op{ATT}, \op{III}}  - V^*_{\op{ATT}, \op{IV}} \\
=& \frac{1}{\Prb{T=1}}\Eb{\frac{1-\rt(1, X)}{\rt(1, X)}\var[Y(1) \mid X, S(1)] + \frac{\et(X)\prns{1-\rt(0, X)}}{\prns{1-\et(X)}\rt(0, X)}\var[Y(0) \mid X, S(0)] \mid T = 1}.
\end{align*}
  \end{enumerate} 
  \end{corollary}
}

\section{Proofs}\label{sec: proof}

\subsection{Supporting Lemmas}
{
    \begin{lemma}\label{lemma: mu-tilde-fun}
    Under \Cref{assump: unconfound,assump: mar-1}, we have 
    \begin{align*}
    \tilde \mu(t, X, S(t)) = \Eb{Y(t) \mid X, S(t)}.
    \end{align*}
    \end{lemma}
    \begin{proof}
    Under \Cref{assump: unconfound,assump: mar-1}, we have that for $t = 0, 1$, 
    \begin{align*}
      \tilde \mu(t, x, s)  &= \Eb{Y \mid T = t, X = x, S = s, R = 1}  \\
        &= \Eb{Y(t) \mid T = t, X = x, S(t) = s, R = 1} \\
        &= \Eb{Y(t) \mid T = t, X = x, S(t) = s} \tag{\Cref{assump: mar-1}} \\
        &= \Eb{Y(t) \mid X = x, S(t) = s}.  \tag{\Cref{assump: unconfound}}
    \end{align*}
    It follows that $\tilde \mu(t, X, S(t)) = \Eb{Y(t) \mid X, S(t)}$.
    \end{proof}
}

\begin{lemma}\label{lemma: nuisance}
Under \Cref{assump: unconfound,assump: mar-1,assump: MAR2,assump: overlap}, the following holds:
\begin{align*}
\expect[\tmut(T, X, S) \mid T, X] = \mut(T, X).
\end{align*}
\end{lemma}
\begin{proof}
Under \Cref{assump: unconfound,assump: mar-1}, we have that for $t = 0, 1$,
\begin{align*}
\expect[\tmut(T, X, S) \mid T = t, X]
    &= \expect[\expect[Y(t) \mid R = 1, T = t, X, S(t)] \mid T = t, X] \\
    &= \expect[\expect[Y(t) \mid T = t, X, S(t)] \mid T = t, X] \\ 
    &= \expect[Y(t) \mid T = t, X].
\end{align*}
Moreover, under \Cref{assump: MAR2,assump: mar-1}, we have $(Y(t), S(t)) \perp R \mid T, X$. 
This in turn implies $Y(t) \perp R \mid T, X$. It follows that 
\begin{align*}
\mut(t, X) = \expect[Y(t) \mid T = t, X, R = 1] = \expect[Y(t) \mid T = t, X] = \expect[\tmut(T, X, S) \mid T = t, X].
\end{align*}
\end{proof}

\begin{lemma}\label{lemma: no-positivity}
If $\pr(R = 1) = 0$, then $\rt(T, X, S) = \pr(R = 1 \mid T, X, S) = 0$ almost surely. 
\end{lemma}
\begin{proof}
Obviously $\expect[\rt(T, X, S)] = \pr(R = 1) = 0$. 

Denote $\mathcal{A} = \{\pr(R = 1 \mid T, X, S) > 0\}$ and $\mathcal{A}_m = \{\pr(R = 1 \mid T, X, S) \ge \frac{1}{m}\}$ for $m = 1, 2, \dots$. Oviously $\mathcal{A} = \cup_{m = 1}^{\infty} \mathcal{A}_m$. By Chebyshev inequality,  
\[
    0 \le \pr(\mathcal{A}_m) \le m \expect[\rt(T, X, S)] = 0.
\]
This implies that $\pr(\mathcal{A}_m) = 0$. By the countable subadditivity of probability measure, we thus have $\pr(\mathcal{A}) \le \sum_{m = 1}^{\infty} \pr(\mathcal{A}_m) = 0$. 
\end{proof}

\begin{lemma}\label{lemma: convergence-rate}
For $k = 1, \dots, K$, if $\|\hmuk - \muo\|_2 = O_p(\rhomu), ~ \|\htmuk - \tmuo\|_2 = O_p(\rhotmu)$, and $\|\hr - \ro\|^2 = O_p(\rhor)$, then for $t = 0, 1$, 
\begin{align*}
&\qquad\qquad\qquad\qquad\qquad \|\hmuk(t, X) - \muo(t, X)\|_2 = O_p(\rhomu), \\
&\|\htmuk(t, X, S(t)) - \tmuo(t, X, S(t))\|_2 = O_p(\rhotmu), \|\hrk(t, X, S(t)) - \ro(t, X, S(t))\|_2 = O_p(\rhor).
\end{align*}
If $\|\tmuo - \tmut \|_2, \|\muo - \mut\|_2$ are almost surely bounded, then $\|\tmuo(t, X, S(t)) - \tmut(t, X, S(t)) \|_2, \|\muo(t, X) - \mut(t, X)\|_2$ for $ t= 0, 1$ are also almost surely bounded. 

Moreover, if $ \|Y(0)\|_q \vee \|Y(1)\|_q  \le C$ for a constant $q > 2$, then $\|\mut(t, X)\|_q \le C, \|\tmut(t, X, S)\|_q \le C$ for $t = 0, 1$. 
\end{lemma}
\begin{proof}
We note that 
\begin{align*}
\|\hmuk - \muo\|_2 
    &= \bigg\{\expect[\hmuk(T, X) - \muo(T, X)]^2\bigg\}^{1/2} \\
    &= \bigg\{\expect\big[(\hmuk(1, X) - \muo(1, X))^2\et(X) + (\hmuk(0, X) - \muo(0, X))^2( 1- \et(X))\big]\bigg\}^{1/2} \\
    &\ge (2\epsilon)^{1/2}\big[\|\hmuk(1, X) - \muo(1, X)\|_2 \vee \|\hmuk(0, X) - \muo(0, X)\|_2\big].
\end{align*}
Thus $\|\hmuk - \muo\|_2 = O_p(\rhomu)$ implies $\|\hmu(t, X) - \muo(t, X)\|_2 = O_p(\rhomu)$ for $t = 0, 1$. Similarly, we can prove that $\|\htmu(t, X, S(t)) - \tmuo(t, X, S(t))\|_2 = O_p(\rhotmu)$ given $\|\htmuk - \tmuo\|_2 = O_p(\rhotmu)$, and $\|\tmuo(t, X, S(t)) - \tmut(t, X, S(t)) \|_2, \|\muo(t, X) - \mut(t, X)\|_2$  are almost surely bounded given that $\|\tmuo - \tmut \|_2, \|\muo - \mut\|_2$ are almost surely bounded. 

Moreover, 
\begin{align*}
\|\hr - \ro\|^2_2 
    &= \expect\bigg[\big(\hr(T, X, S) - r(T, X, S)\big)^2\bigg] \\
    &= \expect\bigg[\expect\big[\expect[\big(\hr(T, X, S) - r(T, X, S)\big)^2 \mid X, T] \mid X\big]\bigg] \\
    &= \expect\bigg[\et(X)\expect\big[\big(\hr(1, X, S(1)) - r(1, X, S(1))\big)^2 \mid X, T = 1\big] \\
    &+ ( 1- \et(X))\expect\big[\big(\hr(0, X, S(0)) - r(0, X, S(0))\big)^2 \mid X, T = 0\big] \bigg]  \\
    &=  \expect\bigg[\et(X)(\hr(1, X, S(1)) - r(1, X, S(1))\big)^2  + ( 1- \et(X))(\hr(0, X, S(0)) - r(0, X, S(0))\big)^2\bigg]  \\
    &\ge 2\epsilon \big(\|\hr(1, X, S(1)) - \ro(1, X, S(1))\|_2^2  \vee \|\hr(0, X, S(0)) - \ro(0, X, S(0))\|_2^2 \big)
\end{align*}
Therefore, $\|\hr(1, X, S(1)) - \ro(1, X, S(1))\|_2 = O_p(\rhor)$ and $ \|\hr(0, X, S(0)) - \ro(0, X, S(0))\|_2 = O_p(\rhor)$.

For the last statement, note that 
\begin{align*}
\|\mut(1, X)\|_q = \expect\big[\expect^q[Y(1) \mid X]\big]^{1/q} \overset{\text{Jensen's inequality}}{\le} \|Y(1)\|_q \le C.
 \end{align*}
 Similarly we can prove that $\|\mut(0, X)\|_q, \|\tmut(0, X, S)\|_q \le \|Y(0)\|_q$ and  $\|\tmut(1, X, S)\|_q \le \|Y(1)\|_q$. Thus $\|\psi(W; {\delta^*}, {\eta^*})\|_q  = O(1)$.
\end{proof}

\subsection{Proofs for Section 1}
\begin{proof}[Proof for \Cref{lemma: identification-1}.]
Note that we have 
\begin{align*}
\expect[Y(t)]  
    &= \expect\bigg\{\expect\big[Y(t) \mid X\big]\bigg\} \nonumber \\
    &= \expect\bigg\{\expect\big[\expect\big(Y(t) \mid X, S(t)\big) \mid X\big]\bigg\} \nonumber  \\
    &= \expect\bigg\{\expect\big[\expect\big(Y(t) \mid T = t, X, S(t)\big) \mid  X\big]\bigg\} \tag{\Cref{assump: unconfound}}  \\
    &= \expect\bigg\{\expect\big[\expect\big(Y(t) \mid T = t, X, S(t), R = 1\big) \mid  X\big]\bigg\}  \tag{\Cref{assump: mar-1}}\\
    &= \expect\bigg\{\expect\big[\expect\big(Y(t) \mid T = t, X, S(t), R = 1\big) \mid  X, T = t\big]\bigg\}.  \tag{\Cref{assump: unconfound}}
\end{align*}

It follows that  
\begin{align*}
\expect[Y(t)]   
    &= \expect\bigg\{\expect\big[\expect\big(Y(t) \mid T = t, X, S(t), R = 1\big) \mid  X, T = t\big]\bigg\} \\
    &= \expect\bigg\{\expect\big[\expect\big(Y \mid T = t, X, S, R = 1\big) \mid  X, T = t\big]\bigg\}.
\end{align*}
\end{proof}

\subsection{Proofs for \cref{sec: efficiency}}\label{sec: proof-sec2}

\begin{proof}[Proof for \Cref{lemma: identification-2}.]
The identification of ATE in setting III is already established in \Cref{lemma: identification-1}, so we focus on establishing identification in the other three settings. 

Under \Cref{assump: unconfound,assump: overlap,assump: mar-1,assump: MAR2}, we have that $(Y(t), S(t)) \perp (T, R) \mid X$. This in particular implies that $Y(t) \perp (T, R) \mid X$. 
Therefore, we have 
\begin{align}
\expect[Y(t)] 
    &= \expect\big[\expect\big(Y(t) \mid X\big)\big] \notag \\ 
    &=  \expect\big[\expect\big(Y(t) \mid T = t, R = 1, X\big)\big] \notag \\
    &= \expect\big[\expect\big(Y \mid T = t, R = 1, X\big)\big]. \label{eq: regular-id}
\end{align}
The last display only depends on distributions of observed data in setting I, II, IV in \Cref{def: settings}.
This shows the identification of ATE in these three settings.
\end{proof}

\begin{proof}[Proof for \Cref{lemma: MAR-implication}.]
If \Cref{assump: mar-1,assump: MAR2} hold, then we have $(Y(t), S(t)) \perp R \mid X, T$. 
Then \Cref{assump: unconfound} holds, i.e., $(Y(t), S(t)) \perp T \mid X$ if and only if $(Y(t), S(t)) \perp (T, R) \mid X$. According to Theorem 17.2 in \cite{wasserman2004all}, this is equivalent to $(Y(t), S(t)) \perp R \mid X, T$ and 
\begin{align*}
  (Y(t), S(t)) \perp T \mid X, R. 
\end{align*}
Therefore, under \Cref{assump: mar-1,assump: MAR2}, \Cref{assump: unconfound} holds if and only if $(Y(t), S(t)) \perp T \mid X, R$. 

Moreover, when \Cref{assump: MAR2} holds, we have $S(t) \perp R \mid T = t, X$, 
which is equivalent to $S \perp R \mid T = t, X$. 
Thus we have $\rt(t, x, s) = \Prb{R = 1 \mid T = t, X,  S} = \Prb{R = 1 \mid T = t, X}$.
Plus, 
\begin{align*}
\mut(t, x) 
   &= \expect[\tmut(T, X, S) \mid T = t, X = x] \\
   &= \expect[\tmut(t, X, S(t)) \mid T = t, X = x] \\
   &= \expect[\tmut(t, X, S(t)) \mid T = t, X = x, R = 1] \\
   &= \expect[\tmut(t, X, S) \mid T = t, X = x, R = 1] \\
   &= \expect[\Eb{Y \mid T = t, X, S, R = 1} \mid T = t, X = x, R = 1] = \Eb{Y \mid T = t, X, R = 1}.
\end{align*}
\end{proof}

\begin{proof}[Proof for \cref{thm: efficient-if}]
Suppose that distribution of $X$, conditional distributin of $S \mid X, T$, and conditional distribution of $Y \mid R, S, T, X$ have true density functions $f^*_X, f^*_{S \mid X, T}, f^*_{Y \mid R, S, T, X}$ with respect to a certain dominating measure. We consider the following  model:
\begin{align*}
    \mathcal{M}_{np} 
        &= \bigg\{f_{X, T, S, R, Y}(X, T, S, R, Y) = f_X(X)\big[e(X)^T(1 - e(X))^{1 - T}\big]f_{S \mid X, T}(S \mid X, T) \\
        &\qquad\qquad\qquad\qquad\qquad\qquad\quad  \times [r(T, X, S)^R(1 - r(T, X, S))^{1- R}]f^R_{Y \mid R = 1, S, T, X}(Y, S, T, X):  \\
        &\qquad\qquad f_X, f_{S \mid X, T}, f_{Y \mid R = 1, S, T, X}~\text{are arbitrary density functions of the distributions}  \\
        &\qquad\qquad\text{indicated by their respective subscripts, and } e(X), r(T, X, S) \text{ are arbitrary}\\&\qquad\qquad\text{functions obeying }e(X) \in [\epsilon, 1- \epsilon], r(T, X, S) \in [\epsilon, 1]\bigg\} 
\end{align*} 
The tangent space corresponding to this model is 
\begin{align*}
\Lambda = \overline\Lambda_X \oplus \overline\Lambda_{T \mid X} \oplus \overline\Lambda_{S \mid T, X} \oplus \overline\Lambda_{R \mid S, T, X} \oplus \overline\Lambda_{Y \mid R, S, T, X},
\end{align*}
where $\overline\Lambda_X, \overline\Lambda_{T \mid X}, \overline\Lambda_{S \mid T, X}, \overline\Lambda_{R \mid S, T, X}, \overline\Lambda_{Y \mid R, S, T, X}$ are mean square closures of the following sets respectively:
\begin{align*}
\Lambda_X 
    &= \{\score_X(X) \in L_2(X): \expect[\score_X(X)] = 0\}, \\
\Lambda_{T \mid X} 
    &= \{\score_{T \mid X}(T, X) \in L_2(T, X): \expect[\score_{T \mid X}(T, X) \mid X] = 0\}, \\
\Lambda_{S \mid X, T}  
    &= \{\score_{S\mid X, T}(S, X, T) \in L_2(S, X, T): \expect[\score_{S \mid X, T}(S, X, T) \mid X, T] = 0\}, \\
\Lambda_{R \mid S, X, T}  
    &= \{\score_{R\mid S, X, T}(R, S, X, T) \in L_2(R, S, X, T): \expect[\score_{R \mid S, X, T}(R, S, X, T) \mid S, X, T] = 0\}, \\
\Lambda_{Y \mid R, S, X, T}
    &= \{R \times \score_{Y \mid R = 1, S, X, T}(Y, R, S, X, T) \in L_2(Y, R, S, X, T): \\
    &\qquad\qquad\qquad\qquad \expect[\score_{Y \mid R = 1, S, X, T}(Y, R, S, X, T) \mid R = 1, S, X, T] = 0\}.
\end{align*} 

We now derive the efficient influence function of $\xi_1^* = \Eb{Y(1)}$. The efficient influence function of  $\xi_0^* = \Eb{Y(0)}$ is analogous so we omit the details for brevity. 

According to \Cref{lemma: identification-1}, 
\begin{align*}
\xi_1^* = \expect[Y(1)] = \expect\bigg[\expect\big[\expect[Y \mid R = 1, T = 1, X, S]\mid X, T = 1\big]\bigg].
\end{align*}
Consider regular parametric submodels indexed by parameters $\gamma$, where $\gamma = 0$ corresponds to the underlying true data distribution. 
We use $\expect_\gamma$ to denote the expectation under the submodel distribution with parameter value $\gamma$.
 Then the corresponding target parameter is 
\begin{align*}
\expect_\gamma[Y(1)] = \expect_\gamma\bigg[\expect_\gamma\big[\expect_\gamma[Y \mid R = 1, T = 1, X, S]\mid X, T = 1\big]\bigg].
\end{align*}

We also use $\score(Y, R, S, T, X)$ to denote the score function corresponding to the parametric submodel.
According to the discussions above, we can write 
\begin{align*}
\score(Y, R, S, T, X) = \score(X) + \score(T \mid X) + \score(S \mid T, X)
+ \score(R \mid S, T, X) + 
\score(Y \mid R, S, T, X),
\end{align*}
where the components above satisfy the restrictions imposed in the sets  $\Lambda_X, \Lambda_{T \mid X}, \Lambda_{S \mid T, X}$, $\Lambda_{R \mid S, T, X}$, $\Lambda_{Y \mid R, S, T, X}$ respectively.

We will next show that 
\begin{align}\label{eq: xi-1-eif}
\frac{\partial}{\partial\gamma}\expect_\gamma[Y(1)]\vert_{\gamma = 0} = \Eb{\psi_1(Y, R, S, T, X)\score(Y, R, S, T, X)},
\end{align}
where 
\begin{align*}
\psi_1(Y, R, S, T, X) =&~ \mut(1, X) - \xi_1^* + \frac{T}{\et(X)}(\tmut(1, X, S) - \mut(1 , X)) \\&+ \frac{TR}{\et(X)\rt(1, X, S)}(Y - \tmut(1, X, S))).
\end{align*}

Note that 
\begin{align}
\frac{\partial \expect_{\gamma}[Y(1)]}{\partial \gamma}\vert_{\gamma = 0} 
    &= \frac{\partial}{\partial\gamma}\expect_{\gamma}\big[\mut(1, X)\big]\vert_{\gamma = 0} + \expect\bigg[\frac{\partial}{\partial\gamma}\expect_{\gamma}\big[\tmut(1, X, S) \mid X, T = 1\big] \vert_{\gamma = 0}\bigg] \nonumber \\
    &+ \expect\bigg[\expect\big[\frac{\partial}{\partial \gamma}\expect_{\gamma}[Y \mid T = 1, R = 1, X, S]\vert_{\gamma = 0} \mid X, T = 1\big]\bigg] \label{eq: path-diff}. 
\end{align}
Now we deal with each term respectively.

First, 
\begin{align*}
\frac{\partial}{\partial\gamma}\expect_{\gamma}\big[\mut(1, X)\big]\vert_{\gamma = 0} &= \Eb{\mut(1, X)\score(X)} = \Eb{\prns{\mut(1, X) - \xi_1^*}\score(X)} \\&= \Eb{\prns{\mut(1, X) - \xi_1^*}\score(Y, R, S, T, X)}.
\end{align*}

Second, 
\begin{align*}
\expect\bigg[\frac{\partial}{\partial\gamma}\expect_{\gamma}\big[\tmut(1, X, S) \mid X, T = 1\big] \vert_{\gamma = 0}\bigg] 
   &= \Eb{\Eb{\tmut(1, X, S)\score(S \mid X, T) \mid X, T = 1}} \\
   &= \Eb{\Eb{\prns{\tmut(1, X, S) - \mut(1, X)}\score(S \mid X, T) \mid X, T = 1}} \\
   &= \Eb{\Eb{\prns{\tmut(1, X, S) - \mut(1, X)}\score(Y, R, S, T, X) \mid X, T = 1}} \\
   &= \Eb{\frac{T}{e^*(X)}\prns{\tmut(1, X, S) - \mut(1, X)}\score(Y, R, S, T, X)}.
\end{align*}

Third, 
\begin{align*}
&\expect\bigg[\expect\big[\frac{\partial}{\partial \gamma}\expect_{\gamma}[Y \mid T = 1, R = 1, X, S]\vert_{\gamma = 0} \mid X, T = 1\big]\bigg] \\
=& \expect\bigg[\expect\big[\expect[Y\times \score(Y \mid R, S, T, X) \mid T = 1, R = 1, X, S]] \mid X, T = 1\big]\bigg] \\
=& \expect\bigg[\expect\big[\expect[\prns{Y - \tmut(1, X, S)}\times \score(Y \mid R, S, T, X) \mid T = 1, R = 1, X, S]\mid X, T = 1\big]\bigg] \\
=& \expect\bigg[\expect\big[\frac{R}{r^*(1, X, S)}\prns{Y - \tmut(1, X, S)}\times \score(Y \mid R, S, T, X) \mid X, T = 1\big]\bigg] \\
=& \expect\bigg[\frac{TR}{e^*(X)r^*(1, X, S)}\prns{Y - \tmut(1, X, S)}\times \score(Y \mid R, S, T, X)\bigg] \\
=& \expect\bigg[\frac{TR}{e^*(X)r^*(1, X, S)}\prns{Y - \tmut(1, X, S)}\times \score(Y, R, S, T, X)\bigg]
\end{align*}
Putting these three terms together, we obtain \Cref{eq: xi-1-eif}. 

Finally, we can show that $\psi_1(Y, R, S, T, X)$ belongs to the tangent space $\Lambda$. 
We can write 
{
\begin{align*}
\psi_1(Y, R, S, T, X) 
    &= \mut(1, X) - \xi_1^* \\
    &+ \frac{T}{\et(X)}(\tmut(1, X, S) - \mut(1 , X)) \\
    &+ \frac{TR}{\et(X)\rt(1, X, S)}(Y - \tmut(1, X, S))).
\end{align*}
}%
{It is easy to show that the three terms in the right hand side above belong to 
$\bar{\Lambda}_X, \bar{\Lambda}_{S \mid T, X}, \bar{\Lambda}_{Y \mid R, S, T, X}$ respectively. 
Therefore, $\psi_1(Y, R, S, T, X)$ belongs to the tangent space $\Lambda$, and thus it is the efficient influence function for $\xi_1^*$. Moreover, this shows that $\psi_1(Y, R, S, T, X)$ is orthogonal to $\bar{\Lambda}_{T \mid X}$ and $\bar{\Lambda}_{R \mid S, T, X}$, so the efficiency bound is also invariant to any restriction on the conditional distributions of $T \mid X$ and $R \mid S, T, X$. In particular, the efficiency bound remains the same if the propensity score $\et(X)$ and $\rt(T, X, S)$ are known.}

Similarly, we can show that the efficient influence function for $\xi_0^*$ is 
\begin{align*}
\psi_0(Y, R, S, T, X) =&~ \mut(0, X) - \xi_0^* + \frac{1-T}{1 - \et(X)}(\tmut(0, X, S) - \mut(0 , X)) \\&+ \frac{(1-T)R}{1 - \et(X)\rt(0, X, S)}(Y - \tmut(0, X, S))).
\end{align*}
It follows that the efficient influence function for $\delta^*$ is $\psi = \psi_1 - \psi_0$, which proves the asserted conclusion in \cref{thm: efficient-if}. 
\end{proof}

\begin{corollary}\label{corollary: variance}
Under \Cref{assump: unconfound,assump: overlap,assump: MAR2,assump: mar-1}, the efficiency lower bound in \Cref{thm: efficient-if} is 
\begin{align*}\label{eq: efficiency-bound}
V^* = \expect[\psi^2(W; \delta^*, \eta^*) ] 
    &=  \var\bigg\{\bigg(\frac{TR}{\et(X)\rt(1, X, S)} - \frac{(1- T)R}{(1 - \et(X))\rt(0, X, S)}\bigg)Y\bigg\} \\
    &- \expect\bigg\{\bigg(\sqrt{\frac{1 - e^*(X)}{e^*(X)}}\mut(1, X) + \sqrt{\frac{\et(X)}{1 - \et(X)}}\mut(0, X)\bigg)^2\bigg\} \nonumber \\
    &- \expect\bigg\{\frac{1}{\et(X)}\frac{1 - \rt(1, X, S(1))}{\rt(1, X, S(1))}\tmutb(1, X, S(1)) \\
    &\qquad\qquad + \frac{1}{1 - \et(X)}\frac{1 - \rt(0, X, S(0))}{\rt(0, X, S(0))}\tmutb(0, X, S(0))\bigg\} \nonumber
\end{align*}
\end{corollary}
\begin{proof}
By straightforward algebra, we can show that  
\begin{align*}
&V^* = \var[\psi(W; \xi_1^*, \eta^*)]  \\
    &= \var\bigg\{\bigg(\frac{TR}{\et(X)\rt(1, X, S)} - \frac{(1- T)R}{(1 - \et(X))\rt(0, X, S)}\bigg)Y\bigg\} \\
    &+ \underbrace{\var\{\mut(1, X) - \mut(0, X)\}}_{V_1} + \underbrace{\var\bigg\{\bigg(\frac{TR}{\et(X)\rt(1, X, S)} - \frac{(1- T)R}{(1 - \et(X))\rt(0, X, S)}\bigg)\tmut(T, X, S)\bigg\}}_{V_2} \\
    &+ \underbrace{\var\bigg\{\bigg(\frac{T}{\et(X)} - \frac{1 - T}{1 - \et(X)}\bigg)\tmut(T, X, S)\bigg\}}_{V_3} + \underbrace{\var\bigg\{\bigg(\frac{T}{\et(X)} - \frac{1 - T}{1 - \et(X)}\bigg)\mut(T, X)\bigg\}}_{V_4} \\
    &+ 2\underbrace{\cov\bigg(\mut(1, X) - \mut(0, X), \bigg(\frac{TR}{\et(X)\rt(1, X, S)} - \frac{(1- T)R}{(1 - \et(X))\rt(0, X, S)}\bigg)(Y - \tmut(T, X, S)\bigg)}_{V_5} \\
    &+ 2\underbrace{\cov\bigg(\mut(1, X) - \mut(0, X), \bigg(\frac{T}{\et(X)} - \frac{1 - T}{1 - \et(X)}\bigg)(\tmut(T, X, S) - \mut(T, X))\bigg)}_{V_6}  \\
    &+ 2\underbrace{\begin{array}{l}\cov\bigg\{\bigg(\frac{T}{\et(X)} - \frac{1 - T}{1 - \et(X)}\bigg)(\tmut(T, X, S) - \mut(T, X)),\\\qquad\qquad\qquad\qquad \bigg(\frac{TR}{\et(X)\rt(1, X, S)} - \frac{(1- T)R}{(1 - \et(X))\rt(0, X, S)}\bigg)(Y - \tmut(T, X, S) \bigg\}\end{array}}_{V_7} \\
    &- 2\underbrace{\begin{array}{l}\cov\bigg\{\bigg(\frac{TR}{\et(X)\rt(1, X, S)} - \frac{(1- T)R}{(1 - \et(X))\rt(0, X, S)}\bigg)Y,\\\qquad\qquad\qquad\qquad \bigg(\frac{TR}{\et(X)\rt(1, X, S)} - \frac{(1- T)R}{(1 - \et(X))\rt(0, X, S)}\bigg)\tmut(T, X, S)\bigg\}\end{array}}_{V_8} \\
    &- 2\underbrace{\cov\bigg\{\bigg(\frac{T}{\et(X)} - \frac{1 - T}{1 - \et(X)}\bigg)\tmut(T, X, S), \bigg(\frac{T}{\et(X)} - \frac{1 - T}{1 - \et(X)}\bigg)\mut(T, X)\bigg\}}_{V_9}
\end{align*}
Now we compute these terms one by one. For $V_1$:
\begin{align*}
 \var\{\mut(1, X) - \mut(0, X)\} &= \expect[(\mut(1, X) - \mut(0, X))^2] - \expect^2[\mut(1, X) - \mut(0, X)] \\
    &= \expect[(\mut(1, X) - \mut(0, X))^2] - (\xi_1^* - \xi_0^*)^2 
\end{align*}
For $V_2 \sim V_4$:
\begin{align*}
V_2 &= \var\bigg\{\bigg(\frac{TR}{\et(X)\rt(1, X, S)} - \frac{(1- T)R}{(1 - \et(X))\rt(0, X, S)}\bigg)\tmut(T, X, S)\bigg\} \\
    &= \var\bigg\{\frac{TR}{\et(X)\rt(1, X, S(1))}\tmut(1, X, S(1))\bigg\} + \var\bigg\{\frac{(1 - T)R}{(1 - \et(X))\rt(0, X, S(0))}\tmut(0, X, S(0))\bigg\} \\
    &- 2\cov\bigg(\frac{TR}{\et(X)\rt(1, X, S(1))}\tmut(1, X, S(1)), \frac{(1 - T)R}{(1 - \et(X))\rt(0, X, S(0))}\tmut(0, X, S(0))\bigg) \\
    &= \expect\bigg\{\frac{1}{\et(X)\rt(1, X, S(1))}\tmutb(1, X, S(1))\bigg\} + \expect\bigg\{\frac{1}{(1 - \et(X))\rt(0, X, S(0))}\tmutb(0, X, S(0))\bigg\} \\&- (
    \xi_1^* - \xi_0^*)^2.
\end{align*}
since  
\begin{align*}
\var\bigg\{\frac{TR}{\et(X)\rt(1, X, S(1))}\tmut(1, X, S(1))\bigg\} 
&= \expect\bigg\{\frac{\expect[TR \mid X, S(1)]}{(\et(X)\rt(1, X, S(1)))^2}\tmutb(1, X, S(1))\bigg\} \\
&- \expect^2\bigg\{\frac{\expect[TR \mid X, S(1)]}{\et(X)\rt(1, X, S(1))}\tmut(1, X, S(1))\bigg\} \\
&= \expect\bigg\{\frac{1}{\et(X)\rt(1, X, S(1))}\tmutb(1, X, S(1))\bigg\} - \xi_1^{*2}, \\ 
\var\bigg\{\frac{(1 - T)R}{(1 - \et(X))\rt(0, X, S(0))}\tmut(0, X, S(0))\bigg\} &=  \expect\bigg\{\frac{1}{( 1- \et(X))\rt(0, X, S(0))}\tmutb(0, X, S(0))\bigg\} - \xi_0^{*2},
\end{align*}
and 
\begin{align*}
&\cov\bigg(\frac{TR}{\et(X)\rt(1, X, S(1))}\tmut(1, X, S(1)), \frac{(1 - T)R}{( 1- \et(X))\rt(0, X, S(0))}\tmut(0, X, S(0))\bigg) \\
&= -\expect\bigg[\frac{TR}{\et(X)\rt(1, X, S(1))}\tmut(1, X, S(1))\bigg]\expect\bigg[\frac{(1 - T)R}{( 1- \et(X))\rt(0, X, S(0))}\tmut(0, X, S(0)) \bigg] = -\xi_1^*\xi_0^*.
\end{align*}
Similarly, 
\begin{align*}
V_3 &= \expect\bigg\{\frac{1}{\et(X)}\tmutb(1, X, S(1))\bigg\} + \expect\bigg\{\frac{1}{1 - \et(X)}\tmutb(0, X, S(0))\bigg\} - (
    \xi_1^* - \xi_0^*)^2 \\
V_4 &= \expect\bigg\{\frac{1}{\et(X)}\tmutb(1, X)\bigg\} + \expect\bigg\{\frac{1}{1 - \et(X)}\tmutb(0, X)\bigg\} - (
    \xi_1^* - \xi_0^*)^2.
\end{align*}
For $V_5 \sim V_7$:
\begin{align*}
V_5 &= \cov\bigg(\mut(1, X) - \mut(0, X), \bigg(\frac{TR}{\et(X)\rt(1, X, S)} - \frac{(1- T)R}{(1 - \et(X))\rt(0, X, S)}\bigg)(Y - \tmut(T, X, S))\bigg) \\
    &= \expect\bigg\{(\mut(1, X) - \mut(0, X))\bigg(\frac{TR}{\et(X)\rt(1, X, S)} - \frac{(1- T)R}{(1 - \et(X))\rt(0, X, S)}\bigg)\\&\times(\expect[Y \mid R = 1, T, X, S] - \tmut(T, X, S))\bigg\}  \\
    &- \expect\bigg\{(\mut(1, X) - \mut(0, X))\bigg\}\expect\bigg\{\bigg(\frac{TR}{\et(X)\rt(1, X, S)} - \frac{(1- T)R}{(1 - \et(X))\rt(0, X, S)}\bigg)\\&\times(\expect[Y \mid R = 1, T, X, S] - \tmut(T, X, S))\bigg\} = 0,
\end{align*}
since $\expect[Y \mid R = 1, T, X, S] = \tmut(T, X, S)$. Similarly $V_7 = 0$. It is analogous to prove that $V_6 = 0$ by noting that $\expect[\tmut(T, X, S) \mid X, T] = \mut(T, X)$ according to \cref{lemma: nuisance}.

For $V_8$ and $V_9$:
\begin{align*}
V_8 &= \expect\bigg\{\bigg(\frac{TR}{\et(X)\rt(1, X, S)} - \frac{(1- T)R}{(1 - \et(X))\rt(0, X, S)}\bigg)^2Y\tmut(T, X, S)\bigg\} \\
    &- \expect\bigg\{\bigg(\frac{TR}{\et(X)\rt(1, X, S)} - \frac{(1- T)R}{(1 - \et(X))\rt(0, X, S)}\bigg)\tmut(T, X, S)\bigg\}\\&\times\expect\bigg\{\bigg(\frac{TR}{\et(X)\rt(1, X, S)} - \frac{(1- T)R}{(1 - \et(X))\rt(0, X, S)}\bigg)Y\bigg\} \\
    &= \expect\bigg\{\bigg(\frac{TR}{\et(X)\rt(1, X, S)} - \frac{(1- T)R}{(1 - \et(X))\rt(0, X, S)}\bigg)^2\tmutb(T, X, S)\bigg\}  \\
    &- \expect^2\bigg\{\bigg(\frac{TR}{\et(X)\rt(1, X, S)} - \frac{(1- T)R}{(1 - \et(X))\rt(0, X, S)}\bigg)\tmut(T, X, S)\bigg\} \\
    &= \var\bigg\{\bigg(\frac{TR}{\et(X)\rt(1, X, S)} - \frac{(1- T)R}{(1 - \et(X))\rt(0, X, S)}\bigg)\tmut(T, X, S)\bigg\} = V_2,
\end{align*}
where the second equality holds because $\expect[Y \mid R, T, X, S] = \tmut(T, X, S)$. Analogously, we can prove that $V_9 = V_4$ by again noting that $\expect[\tmut(T, X, S) \mid X, T] = \mut(T, X)$ according to \cref{lemma: nuisance}.

Therefore, 
\begin{align*}
V^* 
    &= \var\bigg\{\bigg(\frac{TR}{\et(X)\rt(1, X, S)} - \frac{(1- T)R}{(1 - \et(X))\rt(0, X, S)}\bigg)Y\bigg\}  + V_1 + V_3 - V_2 - V_4 \\
    &=  \var\bigg\{\bigg(\frac{TR}{\et(X)\rt(1, X, S)} - \frac{(1- T)R}{(1 - \et(X))\rt(0, X, S)}\bigg)Y\bigg\} \nonumber \\
    &- \expect\bigg\{\bigg(\sqrt{\frac{1 - e^*(X)}{e^*(X)}}\mut(1, X) + \sqrt{\frac{\et(X)}{1 - \et(X)}\mut(0, X)}\bigg)^2\bigg\} \nonumber \\
    &- \expect\bigg\{\frac{1}{\et(X)}\frac{1 - \rt(1, X, S(1))}{\rt(1, X, S(1))}\tmutb(1, X, S(1)) + \frac{1}{1 - \et(X)}\frac{1 - \rt(0, X, S(0))}{\rt(0, X, S(0))}\tmutb(0, X, S(0))\bigg\} \nonumber,
\end{align*}
since 
\begin{align*}
V_1 - V_4 
    &= -\expect\bigg\{\frac{1 - \et(X)}{\et(X)}\tmutb(1, X) + \frac{\et(X)}{1 - \et(X)}\tmutb(0, X) + 2\tmut(1, X)\tmut(0, X)\bigg\} \\
    &= -\expect\bigg\{\bigg(\sqrt{\frac{1 - e^*(X)}{e^*(X)}}\mut(1, X) + \sqrt{\frac{\et(X)}{1 - \et(X)}\mut(0, X)}\bigg)^2\bigg\},
\end{align*}
and 
\begin{align*}
V_3 - V_2 = - \expect\bigg\{\frac{1}{\et(X)}\frac{1 - \rt(1, X, S(1))}{\rt(1, X, S(1))}\tmutb(1, X, S(1)) + \frac{1}{1 - \et(X)}\frac{1 - \rt(0, X, S(0))}{\rt(0, X, S(0))}\tmutb(0, X, S(0))\bigg\}.
\end{align*}
\end{proof}

\begin{proof}[Proof for \cref{thm: efficiency-comparison}]
We derive the efficient influence functions and efficiency bounds for different settings respectively. In all parts, we focus on efficient influence function for $\xi_1^*$. The efficient influence function for $\xi_0^*$ and $\delta^*$ can be derived analogously. 

\paragraph{Efficienct influence function in setting I.}
Consider the following model: 
\begin{align*}
    \mathcal{M}_{np, I} 
        &= \bigg\{f_{X, T, R, Y}(X, T, R, Y) = f_X(X)\big[e(X)^T(1 - e(X))^{1 - T}\big][r(T, X)^R(1 - r(T, X))^{1- R}]\\&\qquad\qquad\qquad\qquad\qquad\qquad \times f^R_{Y \mid R = 1, T, X}(Y, R, T, X): \\
        &\qquad\qquad\qquad  f_X \text{ and } f_{Y \mid R = 1, T, X} \text{are arbitrary density functions,} \\
        &\qquad\qquad\qquad \text{and } e(X), r(T, X) \text{ are arbitrary functions obeying } e(X) \in [\epsilon, 1- \epsilon], r(T, X) \in [\epsilon, 1]  \bigg\}. 
\end{align*}
The corresponding tangent space is $\Lambda_I = \overline\Lambda_X \oplus \overline\Lambda_{T \mid X} \oplus \overline\Lambda_{R \mid T, X} \oplus \overline\Lambda_{Y \mid R, T, X}$, where $\overline\Lambda_X$ and $\overline\Lambda_{T \mid X}$ are given in the proof of \Cref{thm: efficient-if} and $\overline\Lambda_{R \mid T, X}$,  $\overline\Lambda_{Y \mid R, T, X}$ are mean square closures of the following sets:
\begin{align*}
\Lambda_{R \mid X, T}  
    &= \{\score_{R\mid X, T}(R, X, T) \in L_2(R, X, T): \expect[\score_{R \mid X, T}(R, X, T) \mid X, T] = 0\}, \\ 
 \Lambda_{Y \mid R, X, T}
    &= \{R \times \score_{Y \mid R = 1, X, T}(Y, R, X, T) \in L_2(Y, R, X, T): \\
    &\qquad\qquad\qquad\qquad \expect[\score_{Y \mid R = 1, X, T}(Y, R, X, T) \mid R = 1, X, T] = 0\}.
 \end{align*} 
We again derive the influence function of $\xi_1^* = \Eb{Y(1)}$, which can be written as\break $\xi_1^* = \Eb{\Eb{Y \mid X, T = 1, R = 1}}$ according to \Cref{lemma: identification-2}.
Consider regular parametric submodels indexed by $\gamma$ with score function 
\begin{align*}
\score(Y, R, T, X) = \score( X) + \score(T \mid X) 
+ \score(R \mid T, X) + 
\score(Y \mid R,  T, X),
\end{align*}
where the components above satisfy the restrictions imposed in the sets  $\Lambda_X, \Lambda_{T \mid X}, \Lambda_{R \mid T, X}, \Lambda_{Y \mid R, T, X}$ respectively.
The corresponding target parameter is 
\begin{align*}
\E_\gamma\bracks{Y(1)} = \E_\gamma\bracks{\E_\gamma\bracks{Y \mid X, T = 1, R = 1}}.
\end{align*}
Then 
\begin{align*}
\frac{\partial}{\partial\gamma} \E_\gamma\bracks{Y(1)}\vert_{\gamma=0} = \frac{\partial}{\partial\gamma}\E_\gamma\bracks{\E\bracks{Y \mid X, T = 1, R = 1}}\vert_{\gamma=0} +  \E\bracks{\frac{\partial}{\partial\gamma}\E_\gamma\bracks{Y \mid X, T = 1, R = 1}\vert_{\gamma=0}}.
\end{align*}
We have 
\begin{align*}
\frac{\partial}{\partial\gamma}\E_\gamma\bracks{\E\bracks{Y \mid X, T = 1, R = 1}}\vert_{\gamma=0} = \Eb{\mut(1, X)\score(X)} = \Eb{\prns{\mut(1, X) - \xi_1^*}\score(Y, R, T, X)}
\end{align*}
and 
\begin{align*}
 \E\bracks{\frac{\partial}{\partial\gamma}\E_\gamma\bracks{Y \mid X, T = 1, R = 1}\vert_{\gamma=0}} 
   &= \Eb{\Eb{Y \times \score(Y \mid X, T, R) \mid X, T = 1, R = 1}} \\
   &= \Eb{\frac{TR}{\et(X)\rt(1, X)}\prns{Y - \mut(1, X)}\score(Y, X, T, R)}.
\end{align*}
It follows that 
\begin{align*}
 \frac{\partial}{\partial\gamma} \E_\gamma\bracks{Y(1)}\vert_{\gamma=0} = \Eb{\psi_{I, 1}(Y, X, T, R)\score(Y, X, T, R)},
 \end{align*} 
where 
\begin{align*}
\psi_{I, 1}(Y, X, T, R) = \frac{TR}{\et(X)\rt(1, X)}\prns{Y - \mut(1, X)} + \mut(1, X) - \xi_1^*.
\end{align*}
Finally, we can show that $\psi_{I, 1}(Y, R, T, X)$ belongs to the tangent space $\Lambda_I$. 
We can write 
{
\begin{align*}
\psi_{I, 1}(Y, R, T, X) 
    &= \mut(1, X) - \xi_1^* \\
    &+ \frac{TR}{\et(X)\rt(1, X)}\prns{Y - \mut(1, X)}.
\end{align*}
}
It is easy to verify that the terms on the right hand side of the equation above belong to\break 
$\bar{\Lambda}_X, \bar{\Lambda}_{Y \mid R, T, X}$ respectively. 
Thus $\psi_{I, 1}(Y, R, T, X)$ belongs to the tangent space $\Lambda_I$. 
This shows that $\psi_{I, 1}$ is the efficient influence function of $\xi_1^*$.
Similarly, we can derive the efficient influence function $\psi_{I, 0}$ for $\xi_0^*$ as follows: 
\begin{align*}
\psi_{I, 0}(Y, R, T, X)  = \frac{(1-T)R}{\et(X)\rt(0, X)}\prns{Y - \mut(0, X)} + \mut(0, X) - \xi_0^*.
\end{align*}
It follows that $\psi_{I, 1} - \psi_{I, 0}$ is the efficient influence function for $\delta^*$. {Moreover, from the analysis above, we can see that the efficient influence function is orthogonal to $\bar{\Lambda}_{T \mid X}$ and $\bar{\Lambda}_{R \mid T, X}$, so the corresponding efficiency bound is invariant to any restriction on $\bar{\Lambda}_{T \mid X}$ or $\bar{\Lambda}_{R \mid T, X}$. In particular, the corresponding efficiency bound is invariant to the knowledge of $\et, \rt$. }

\paragraph{Efficienct influence function in setting II.} 
Now we consider the model 
\begin{align*}
    \mathcal{M}_{np, II} 
        &= \bigg\{f_{X, T, R, Y, S}(X, T, R, Y, S) = f_X(X)\big[e(X)^T(1 - e(X))^{1 - T}\big][r(T, X)^R(1 - r(T, X))^{1- R}]: \\
        &\qquad\qquad\qquad \qquad\qquad\qquad \times f^R_{Y \mid R = 1, T, X}(Y, R, T, X) f^R_{S \mid R = 1, T, X, Y}(S, R, T, X, Y), \\
        &\qquad\qquad\qquad f_X, f_{Y \mid R = 1, T, X}, f_{S \mid R = 1, T, X, Y} \text{ are arbitrary density functions,} \\
        &\qquad\qquad\qquad \text{and } e(X), r(T, X) \text{ are arbitrary functions obeying } e(X) \in [\epsilon, 1- \epsilon], r(T, X) \in [\epsilon, 1]  \bigg\}. 
\end{align*}
The corresponding tangent space is $\Lambda_{II} = \overline\Lambda_X \oplus \overline\Lambda_{T \mid X} \oplus \overline\Lambda_{R \mid T, X} \oplus \overline\Lambda_{Y \mid R, T, X} \oplus \overline\Lambda_{S \mid Y, R, T, X}$, where $\overline\Lambda_X$ and $\overline\Lambda_{T \mid X}$ are given in the proof of \Cref{thm: efficient-if} and $\overline\Lambda_{R \mid T, X}$,  $\overline\Lambda_{Y \mid R, T, X}$, $\overline\Lambda_{S \mid Y, R, T, X}$ are mean square closures of the following sets:
\begin{align*}
\Lambda_{R \mid X, T}  
    &= \{\score_{R\mid X, T}(R, X, T) \in L_2(R, X, T): \expect[\score_{R \mid X, T}(R, X, T) \mid X, T] = 0\}, \\ 
 \Lambda_{Y \mid R, X, T}
    &= \{R \times \score_{Y \mid R = 1, X, T}(Y, R, X, T) \in L_2(Y, R, X, T): \\
    &\qquad\qquad\qquad\qquad \expect[\score_{Y \mid R = 1, X, T}(Y, R, X, T) \mid R = 1, X, T] = 0\} \\
\Lambda_{S \mid Y, R, X, T}
    &= \{R \times \score_{S \mid Y, R = 1, X, T}(S, Y, R, X, T) \in L_2(S, Y, R, X, T): \\
    &\qquad\qquad\qquad\qquad \expect[\score_{S \mid Y, R = 1, X, T}(S, Y, R, X, T) \mid Y, R = 1, X, T] = 0\}.
 \end{align*} 
 Consider regular parametric submodels indexed by $\gamma$ with score function 
\begin{align*}
\score(S, Y, R, T, X) = \score( X) + \score(T \mid X) 
+ \score(R \mid T, X) + 
\score(Y \mid R,  T, X) + \score(S \mid Y, R, T, X),
\end{align*}
where the components above satisfy the restrictions imposed in the sets  $\Lambda_X, \Lambda_{T \mid X}, \Lambda_{R \mid T, X}$, $\Lambda_{Y \mid R, T, X}$, $\Lambda_{S \mid Y, R, T, X}$ respectively.
The corresponding target parameter is 
\begin{align*}
\E_\gamma\bracks{Y(1)} = \E_\gamma\bracks{\E_\gamma\bracks{Y \mid X, T = 1, R = 1}}.
\end{align*}
The analyses for setting I already shows that 
\begin{align*}
 \frac{\partial}{\partial\gamma} \E_\gamma\bracks{Y(1)}\vert_{\gamma=0} = \Eb{\psi_{I, 1}(Y, X, T, R)\score(Y, X, T, R)},
 \end{align*} 
 It follows that 
 \begin{align*}
 \frac{\partial}{\partial\gamma} \E_\gamma\bracks{Y(1)}\vert_{\gamma=0} = \Eb{\psi_{I, 1}(Y, X, T, R)\score(S, Y, R, T, X)},
 \end{align*} 
 since
 \begin{align*}
  \Eb{\psi_{I, 1}(Y, X, T, R)\score(S \mid Y, R, T, X)} = \Eb{\psi_{I, 1}(Y, X, T, R)\Eb{\score(S \mid Y, R, T, X)\mid Y, R, T, X}} = 0.
  \end{align*} 
  This means that $\psi_{I, 1}$ is also an influence function for $\xi_1^*$ under the model $\mathcal{M}_{np, II}$.

  Moreover, we showed that $\psi_{I, 1} \in \Lambda_I$. 
  Since $\Lambda_I \subset \Lambda_{II}$, we also have $\psi_{I, 1} \in \Lambda_{II}$.
  It follows that $\psi_{I, 1}$ is also the efficient influence function for $\xi_1^*$ under the model $\mathcal{M}_{np, II}$.
  Similarly, we can validate that 
$\psi_{I, 1} - \psi_{I, 0}$ is the efficient influence function of $\delta^*$ udner the model $\mathcal{M}_{np, II}$.
{From the analysis for setting I, we can also see that the efficient influence function is orthogonal to $\bar{\Lambda}_{T \mid X}$ and $\bar{\Lambda}_{R \mid T, X}$, so the corresponding efficiency bound is invariant to the knowledge of $\et, \rt$ as well. }

\paragraph{Efficienct influence function in setting III.} 
Under the additional \Cref{assump: MAR2}, we have $R \perp S \mid T, X$, thus the tangent space under \Cref{assump: unconfound,assump: MAR2,assump: mar-1,assump: overlap} now becomes 
\begin{align*}
\Lambda_{III} = \overline\Lambda_X \oplus \overline\Lambda_{T \mid X} \oplus \overline\Lambda_{S \mid T, X} \oplus \overline\Lambda_{R \mid T, X} \oplus \overline\Lambda_{Y \mid R, S, T, X},
\end{align*}
where $\overline\Lambda_X, \overline\Lambda_{T \mid X}, \overline\Lambda_{S \mid T, X}, \overline\Lambda_{Y \mid R, S, T, X}$ are given in the proof for \Cref{thm: efficient-if}, and $\overline\Lambda_{R \mid T, X}$ is the mean-square closure of the set 
\begin{align*}
\Lambda_{R \mid X, T}  
    &= \{\score_{R\mid X, T}(R, X, T) \in L_2(R, X, T): \expect[\score_{R \mid X, T}(R, X, T) \mid X, T] = 0\}.
\end{align*}

The function $\psi$ in \Cref{eq: efficient-IF} is again an influence function in setting III with the additional \Cref{assump: MAR2}. 
Moreover, it is easy to show that 
\begin{align*}
&\mut(1, X) - \mut(0, X) - \delta^* \in \overline\Lambda_X \oplus \overline\Lambda_{T \mid X}  \\
&\frac{T}{\et(X)}(\tmut(1, X, S) - \mut(1 , X)) - \frac{1 - T}{1 - \et(X)}(\tmut(0, X, S) - \mut(0 , X)) \in \overline\Lambda_{S \mid T, X} \\
&\frac{TR}{\et(X)\rt(1, X, S)}(Y - \tmut(1, X, S)) - \frac{(1 - T)R}{(1 - \et(X))\rt(0, X, S)}(Y - \tmut(0, X, S)) \in \overline\Lambda_{Y \mid R, S, T, X}.
\end{align*}

It follows that $\psi \in \Lambda_{III}$. 
Thus $\psi$ is again the efficient influence function. {From the analysis in the proof for \cref{thm: efficient-if}, we also know that the corresponding efficiency bound is invariant to the knowledge of $\et, \rt$. }

\paragraph{Efficienct influence function in setting IV.}
The efficient influence function {and its invariance to the knowledge of $\et$} in setting IV directly follows from \cite{hahn1998role} so we omit the details. Moreover, in this setting the $\rt$ is always known to be equal to $1$. 

\end{proof}  

\begin{corollary}\label{corollary: variance-setting}
Under \Cref{assump: unconfound,assump: overlap,assump: MAR2,assump: mar-1}, the efficiency lower bounds for setting I-IV in \Cref{def: settings} are given as follows: 
\begin{align*}
V_I^* = V_\text{II}^*  
    &= \var\left\{\left(\frac{TR}{\et(X)\rt(1, X)} - \frac{(1 - T)R}{(1 - \et(X))\rt(0, X)}\right)Y\right\} \\
    &- \expect\bigg\{\frac{\pr(R = 0 \mid X)}{\et(X)\rt(1, X)}\mutb(1, X) + \frac{\pr(R = 0 \mid X)}{\et(X)\rt(0, X)}\mutb(0, X)\bigg\} \\
    &-\expect\bigg\{\bigg(\sqrt{\frac{1 - e^*(X)}{e^*(X)}\frac{\rt(0, X)}{\rt(1, X)}}\mut(1, X) + \sqrt{\frac{\et(X)}{1 - \et(X)}\frac{\rt(1, X)}{\rt(0, X)}}\mut(0, X)\bigg)^2\bigg\} \\
V^*_\text{III}
    &=  \var\bigg\{\bigg(\frac{TR}{\et(X)\rt(1, X)} - \frac{(1- T)R}{(1 - \et(X))\rt(0, X)}\bigg)Y\bigg\} \\
    &- \expect\bigg\{\bigg(\sqrt{\frac{1 - e^*(X)}{e^*(X)}}\mut(1, X) + \sqrt{\frac{\et(X)}{1 - \et(X)}}\mut(0, X)\bigg)^2\bigg\} \nonumber \\
    &- \expect\bigg\{\frac{1}{\et(X)}\frac{1 - \rt(1, X)}{\rt(1, X)}\tmutb(1, X, S(1)) + \frac{1}{1 - \et(X)}\frac{1 - \rt(0, X)}{\rt(0, X)}\tmutb(0, X, S(0))\bigg\} \\
V_\text{IV}^* 
    &= \var\left\{\left(\frac{T}{\et(X)} - \frac{1 - T}{1 - \et(X)}\right)Y\right\} - \expect\bigg\{\bigg(\sqrt{\frac{1 - e^*(X)}{e^*(X)}}\mut(1, X) + \sqrt{\frac{\et(X)}{1 - \et(X)}}\mut(0, X)\bigg)^2\bigg\}.
\end{align*}
\end{corollary}

\begin{proof}
\textbf{Efficiency bound in setting I.} The semiparametric efficiency bound is given by $\var\{\psi_{I}(W; \delta^*, \eta^*)\}$:
\begin{align*}
V_I^* 
    &= \var\{\psi_{I}(W; \delta^*, \eta^*)\}   \\
    &=\var\left\{\left(\frac{TR}{\et(X)\rt(1, X)} - \frac{(1 - T)R}{(1 - \et(X))\rt(0, X)}\right)Y\right\} + \underbrace{\var\left\{\mut(1, X) - \mut(0, X)\right\}}_{V_{10}} \\
    &+ \underbrace{\var\left\{\left(\frac{TR}{\et(X)\rt(1, X)} - \frac{(1 - T)R}{(1 - \et(X))\rt(0, X)}\right)\mut(T, X)\right\}}_{V_{11}} \\
    &- 2\underbrace{\begin{array}{l}\cov\biggl\{\left(\frac{TR}{\et(X)\rt(1, X)} - \frac{(1 - T)R}{(1 - \et(X))\rt(0, X)}\right)Y,\\\qquad\qquad\qquad \left(\frac{TR}{\et(X)\rt(1, X)} - \frac{(1 - T)R}{(1 - \et(X))\rt(0, X)}\right)\mut(T, X)\biggr\}\end{array}}_{V_{12}} \\
    &- 2\underbrace{\cov\left\{\mut(1, X) - \mut(0, X), \left(\frac{TR}{\et(X)\rt(1, X)} - \frac{(1 - T)R}{(1 - \et(X))\rt(0, X)}\right)\mut(T, X)\right\}}_{V_{13}} \\
    &+ 2\underbrace{\cov\left\{\mut(1, X) - \mut(0, X), \left(\frac{TR}{\et(X)\rt(1, X)} - \frac{(1 - T)R}{(1 - \et(X))\rt(0, X)}\right)Y\right\}}_{V_{14}}
\end{align*}
Similarly to Step IV in the proof of \Cref{corollary: variance}, we can show that 
\begin{align*}
V_{10} 
    &= \expect\left[\mut(1, X) - \mut(0, X)\right]^2 - (\xi_1^* - \xi_0^*)^2\\ 
V_{11} = V_{12}
    &= \expect\left\{\frac{{\mut}^2(1, X)}{\et(X)\rt(1, X)} + \frac{{\mut}^2(0, X)}{\et(X)\rt(0, X)}\right\} - (\xi_1^* - \xi_0^*)^2 \\
V_{13} &= V_{14}.
\end{align*}
Therefore,
\begin{align*}
 V_I^*  = \var\{\psi_{I}(W; \delta^*, \eta^*)\} 
    &= \var\left\{\left(\frac{TR}{\et(X)\rt(1, X)} - \frac{(1 - T)R}{(1 - \et(X))\rt(0, X)}\right)Y\right\} + V_{10} - V_{11} \\
    &= \var\left\{\left(\frac{TR}{\et(X)\rt(1, X)} - \frac{(1 - T)R}{(1 - \et(X))\rt(0, X)}\right)Y\right\} \\
    &- \expect\bigg\{\bigg(\sqrt{\frac{1 - e^*(X)}{e^*(X)}\frac{\rt(0, X)}{\rt(1, X)}}\mut(1, X) + \sqrt{\frac{\et(X)}{1 - \et(X)}\frac{\rt(1, X)}{\rt(0, X)}}\mut(0, X)\bigg)^2\bigg\} \\
    &- \expect\bigg\{\frac{\pr(R = 0 \mid X)}{\et(X)\rt(1, X)}\mutb(1, X) + \frac{\pr(R = 0 \mid X)}{\et(X)\rt(0, X)}\mutb(0, X)\bigg\}.
 \end{align*} 
 The second equality above holds because
 \begin{align*} 
 V_{10} - V_{11} 
    &=  \expect\left[\mut(1, X) - \mut(0, X)\right]^2 - \expect\left\{\frac{{\mut}^2(1, X)}{\et(X)\rt(1, X)} + \frac{{\mut}^2(0, X)}{\et(X)\rt(0, X)}\right\}  \\
    &= -\expect\left\{\frac{1 - \pr(T = 1, R = 1 \mid X)}{\pr(T = 1, R = 1 \mid X)}{\mut}^2(1, X) + \frac{1 - \pr(T = 0, R = 1 \mid X)}{\pr(T = 0, R = 1 \mid X)}{\mut}^2(0, X)\right\} \\
    &\qquad\qquad- 2\expect\left\{\mut(1, X)\mut(0, X)\right\} \\
    &= -\expect\bigg\{\frac{\pr(R = 1 \mid X) - \pr(T = 1, R = 1 \mid X)}{\pr(T = 1, R = 1 \mid X)}{\mut}^2(1, X) \\ 
    &\qquad + \frac{\pr(R = 1 \mid X) - \pr(T = 0, R = 1 \mid X)}{\pr(T = 0, R = 1 \mid X)}{\mut}^2(0, X)\bigg\}  \\
    &- \expect\left\{\frac{\pr(R = 0 \mid X)}{\pr(T = 1, R = 1 \mid X)}{\mut}^2(1, X) + \frac{\pr(R = 0 \mid X)}{\pr(T = 0, R = 1 \mid X)}{\mut}^2(0, X)\right\} - 2\expect\left\{\mut(1, X)\mut(0, X)\right\} \\
    &= -\expect\left\{\frac{(1 - \et(X))\rt(0, X)}{\et(X)\rt(1, X)}{\mut}^2(1, X) + \frac{\et(X)\rt(1, X)}{(1 - \et(X))\rt(0, X)}{\mut}^2(0, X)\right\} - 2\expect\left\{\mut(1, X)\mut(0, X)\right\} \\
    &- \expect\left\{\frac{\pr(R = 0 \mid X)}{\et(X)\rt(1, X)}{\mut}^2(1, X) + \frac{\pr(R = 0 \mid X)}{(1 - \et(X))\rt(0, X)}{\mut}^2(0, X)\right\} \\
    &= - \expect\bigg\{\bigg(\sqrt{\frac{1 - e^*(X)}{e^*(X)}\frac{\rt(0, X)}{\rt(1, X)}}\mut(1, X) + \sqrt{\frac{\et(X)}{1 - \et(X)}\frac{\rt(1, X)}{\rt(0, X)}}\mut(0, X)\bigg)^2\bigg\} \\
    &- \expect\bigg\{\frac{\pr(R = 0 \mid X)}{\et(X)\rt(1, X)}\mutb(1, X) + \frac{\pr(R = 0 \mid X)}{\et(X)\rt(0, X)}\mutb(0, X)\bigg\}.
 \end{align*}

 \textbf{Efficiency lower bound in setting II.} From the proof of \Cref{thm: efficiency-comparison}, we know that $V_2^* = V_1^*$. 

 \textbf{Efficiency lower bound in setting III. } The conclusion follows directly from \Cref{corollary: variance} by noting that $\rt(t, X, S) = \rt(t, X)$.

 \textbf{Efficiency lower bound in setting IV.} The efficiency lower bound is given by $\expect\{\psi^2_{IV}(W; \delta^*, \eta^*)\}$:
\begin{align*}
&V_\text{IV}^* = \expect\{\psi^2_{IV}(W; \delta^*, \eta^*)\}  \\
    &= \var\left\{\left(\frac{T}{\et(X)} - \frac{1 - T}{1 - \et(X)}\right)Y\right\} + \var\left\{\mut(1, X) - \mut(0, X) \right\}  \\
    &+  \var\left\{\left(\frac{T}{\et(X)} - \frac{1 - T}{1 - \et(X)}\right)\mut(T, X)\right\} \\
    &+ 2\cov\left\{\left(\frac{T}{\et(X)} - \frac{1 - T}{1 - \et(X)}\right)Y, \mut(1, X) - \mut(0, X)\right\} \\
    &- 2\cov\left\{\left(\frac{T}{\et(X)} - \frac{1 - T}{1 - \et(X)}\right)\mut(T, X), \mut(1, X) - \mut(0, X)\right\} \\
    &- 2\cov\left\{\left(\frac{T}{\et(X)} - \frac{1 - T}{1 - \et(X)}\right)Y, \left(\frac{T}{\et(X)} - \frac{1 - T}{1 - \et(X)}\right)\mut(T, X)\right\}.
\end{align*}
Analogously to step IV in the proof of \Cref{corollary: variance}, we can show that 
\begin{align*}
 &\var\left\{\left(\frac{T}{\et(X)} - \frac{1 - T}{1 - \et(X)}\right)\mut(T, X)\right\} 
    \\
    &\qquad\qquad= \cov\left\{\left(\frac{T}{\et(X)} - \frac{1 - T}{1 - \et(X)}\right)Y, \left(\frac{T}{\et(X)} - \frac{1 - T}{1 - \et(X)}\right)\mut(T, X)\right\}\\
    &\qquad\qquad = \expect\left\{\frac{1}{\et(X)}{\mut}^2(1, X) + \frac{1}{1 - \et(X)}{\mut}^2(0, X)\right\} - (\lambda^*_1 - \lambda^*_0)^2 \\
&\cov\left\{\left(\frac{T}{\et(X)} - \frac{1 - T}{1 - \et(X)}\right)Y, \mut(1, X) - \mut(0, X)\right\} \\
    &\qquad\qquad = \cov\left\{\left(\frac{T}{\et(X)} - \frac{1 - T}{1 - \et(X)}\right)\mut(T, X), \mut(1, X) - \mut(0, X)\right\}. 
\end{align*}
Thus 
\begin{align*}
&\var\{\psi_{IV}(W; \delta^*, \eta^*)\}  \\
    &= \var\left\{\left(\frac{T}{\et(X)} - \frac{1 - T}{1 - \et(X)}\right)Y\right\} - \expect\left\{\frac{1}{\et(X)}{\mut}^2(1, X) + \frac{1}{1 - \et(X)}{\mut}^2(0, X)\right\} \\
    &+ \expect\left[\mut(1, X) - \mut(0, X)\right]^2 \\
    &= \var\left\{\left(\frac{T}{\et(X)} - \frac{1 - T}{1 - \et(X)}\right)Y\right\} - \expect\bigg\{\bigg(\sqrt{\frac{1 - e^*(X)}{e^*(X)}}\mut(1, X) + \sqrt{\frac{\et(X)}{1 - \et(X)}}\mut(0, X)\bigg)^2\bigg\}.
\end{align*}
\end{proof}

\begin{proof}[Proof for \cref{corollary: efficiency-loss}]
According to \cref{corollary: variance-setting} and \cref{corollary: variance}, we can verify that 
\begin{align*}
&\qquad\qquad\qquad V_\text{I}^* - V^*_\text{III}  = V_\text{II}^* - V^*_\text{III} \\
&=\expect\bigg\{\bigg(\sqrt{\frac{1 - e^*(X)}{e^*(X)}}\mut(1, X) + \sqrt{\frac{\et(X)}{1 - \et(X)}}\mut(0, X)\bigg)^2\bigg\} \\
&- \expect\bigg\{\frac{\pr(R = 0 \mid X)}{\et(X)\rt(1, X)}\mutb(1, X) + \frac{\pr(R = 0 \mid X)}{\et(X)\rt(0, X)}\mutb(0, X)\bigg\}\\
&\qquad -\expect\bigg\{\bigg(\sqrt{\frac{1 - e^*(X)}{e^*(X)}\frac{\rt(0, X)}{\rt(1, X)}}\mut(1, X) + \sqrt{\frac{\et(X)}{1 - \et(X)}\frac{\rt(1, X)}{\rt(0, X)}}\mut(0, X)\bigg)^2\bigg\}\\
&+ \expect\bigg\{\frac{1}{\et(X)}\frac{1 - \rt(1, X)}{\rt(1, X)}\tmutb(1, X, S(1)) + \frac{1}{1 - \et(X)}\frac{1 - \rt(0, X)}{\rt(0, X)}\tmutb(0, X, S(0))\bigg\} \\
&= \expect\bigg\{\frac{1 - e^*(X)}{e^*(X)}{\mut}^2(1, X) + \frac{e^*(X)}{1 - e^*(X)}{\mut}^2(0, X) \bigg\} +  2\expect\left\{\mut(1, X)\mut(0, X)\right\}\\
&\qquad -\expect\left\{\frac{1 - \pr(T = 1, R = 1 \mid X)}{\pr(T = 1, R = 1 \mid X)}{\mut}^2(1, X) + \frac{1 - \pr(T = 0, R = 1 \mid X)}{\pr(T = 0, R = 1 \mid X)}{\mut}^2(0, X)\right\} \\
&\qquad - 2\expect\left\{\mut(1, X)\mut(0, X)\right\} \\
&\qquad + \expect\bigg\{\frac{1}{\et(X)}\frac{1 - \rt(1, X)}{\rt(1, X)}\tmutb(1, X, S(1)) + \frac{1}{1 - \et(X)}\frac{1 - \rt(0, X)}{\rt(0, X)}\tmutb(0, X, S(0))\bigg\}\\ 
&=  \expect\bigg\{\frac{1 - e^*(X)}{e^*(X)}{\mut}^2(1, X) + \frac{e^*(X)}{1 - e^*(X)}{\mut}^2(0, X) \bigg\} \\
&\qquad- \expect\left\{\frac{1 - \et(X)\rt(1, X)}{\et(X)\rt(1, X)}{\mut}^2(1, X) + \frac{1 - (1 - \et(X))\rt(0, X)}{(1 - \et(X))\rt(0, X)}{\mut}^2(0, X)\right\}   \\ 
&\qquad+ \expect\bigg\{\frac{1}{\et(X)}\frac{1 - \rt(1, X)}{\rt(1, X)}\tmutb(1, X, S(1)) + \frac{1}{1 - \et(X)}\frac{1 - \rt(0, X)}{\rt(0, X)}\tmutb(0, X, S(0))\bigg\} \\
&= \expect\bigg\{\frac{1}{\et(X)}\frac{1 - \rt(1, X)}{\rt(1, X)}(\tmutb(1, X, S(1)) - \mutb(1, X)) \\
&\qquad+ \frac{1}{1 - \et(X)}\frac{1 - \rt(0, X)}{\rt(0, X)}(\tmutb(0, X, S(0)) - \mutb(0, X))\bigg\} 
\end{align*}
Then the conclusion in statement 1 follows from the fact that $\mut(t, x) = \expect[Y \mid T = t, X = x,R = 1] = \Eb{\tmut(t, X, S(t)) \mid X = x}$ under \Cref{assump: MAR2} according to \Cref{lemma: MAR-implication}.
{Moreover, according to \Cref{lemma: mu-tilde-fun}, \Cref{assump: unconfound,assump: mar-1} imply that $\tmut(t, X, S(t)) = \Eb{Y(t) \mid X, S(t)}$, so we have $\var[\tmut(t, X, S(t)) \mid X] = \var[\Eb{Y(t) \mid X, S(t)} \mid X]$ for $t = 0, 1$. }

Furthermore, 
\begin{align*}
&\qquad\qquad\qquad\qquad\qquad V_\text{III}^* - V^*_\text{IV}  \\ 
&= \var\bigg\{\bigg(\frac{TR}{\et(X)\rt(1, X)} - \frac{(1- T)R}{(1 - \et(X))\rt(0, X)}\bigg)Y\bigg\} - \var\bigg\{\bigg(\frac{T}{\et(X)} - \frac{(1- T)}{(1 - \et(X))}\bigg)Y\bigg\}  \\
&- \expect\bigg\{\frac{1}{\et(X)}\frac{1 - \rt(1, X)}{\rt(1, X)}\tmutb(1, X, S(1)) + \frac{1}{1 - \et(X)}\frac{1 - \rt(0, X)}{\rt(0, X)}\tmutb(0, X, S(0))\bigg\} \\
&= \expect\bigg\{\frac{1}{\et(X)\rt(1, X)}Y^2(1) + \frac{1}{\et(X)\rt(0, X)}Y^2(0)\bigg\} - (\xi_1^* - \xi_0^*)^2 - \expect\bigg\{\frac{1}{\et(X)}Y^2(1) + \frac{1}{\et(X)}Y^2(0)\bigg\} \\
&\qquad\qquad\qquad + (\xi_1^* - \xi_0^*)^2  \\
&- \expect\bigg\{\frac{1}{\et(X)}\frac{1 - \rt(1, X)}{\rt(1, X)}\tmutb(1, X, S(1)) + \frac{1}{1 - \et(X)}\frac{1 - \rt(0, X)}{\rt(0, X)}\tmutb(0, X, S(0))\bigg\} \\
&= \expect\bigg\{\frac{1}{\et(X)}\frac{1 - \rt(1, X)}{\rt(1, X)}(Y^2(1) - \tmutb(1, X, S(1))) + \frac{1}{1 - \et(X)}\frac{1 - \rt(0, X)}{\rt(0, X)}(Y^2(0) - \tmutb(0, X, S(0)))\bigg\}.
\end{align*}
This obviously implies the conclusion in statement 2. 
\end{proof}

\begin{proof}[Proof for \Cref{thm: effect-unlabelled}.]
Note that 
\begin{align*}
&\Eb{\Eb{\Eb{Y \mid S, X, T = t, R = 1} \mid X, T =t, R = 0} \mid R = 0} \\
=& \Eb{\Eb{\Eb{Y(t) \mid S(t), X, T = t, R = 1} \mid X, T = t, R = 0} \mid R = 0} \\
=& \Eb{\Eb{\Eb{Y(t) \mid S(t), X, T = t, R = 0} \mid X, T = t, R = 0} \mid R = 0} \tag{\Cref{assump: mar-1}} \\
=& \Eb{\Eb{\Eb{Y(t) \mid S(t), X, R = 0} \mid X, R = 0} \mid R = 0}  \tag{$\prns{Y(t), S(t)} \perp T \mid X, R = 0$} \\
=& \Eb{Y(t) \mid R = 0}. 
\end{align*}
This proves the identification of $\delta_0^*$. 

Now we derive the efficiency bound for $\Eb{Y(1) \mid R = 0}$, based on the $\mathcal{M}_{np}$ model and the corresponding parametric submodels in the proof for \Cref{thm: efficient-if}.
Consider the target parameter under parametric submodels indexed by parameters $\gamma$. Then 
\begin{align*}
&\frac{\partial}{\partial\gamma}\E_{\gamma}\bracks{Y(1) \mid R = 0}\vert_{\gamma = 0} \\
   =&  \frac{\partial}{\partial\gamma}\E_\gamma\bracks{\E_\gamma\bracks{\E_\gamma\bracks{Y \mid S, X, T = 1, R = 1} \mid X, T =1, R = 0} \mid R = 0}\vert_{\gamma = 0} \\
   =& \frac{\partial}{\partial\gamma}\E_{\gamma}\bracks{\mut_0(1, X) \mid R = 0}\vert_{\gamma = 0} + \frac{\partial}{\partial\gamma}\Eb{\E_\gamma\bracks{\tmut(1, X, S)\mid X, T=1,R=0} \mid R = 0}\vert_{\gamma = 0} \\
   +& \frac{\partial}{\partial\gamma}\E\bracks{\E\bracks{\E_\gamma\bracks{Y \mid S, X, T = 1, R = 1} \mid X, T =1, R = 0} \mid R = 0}\vert_{\gamma = 0}.
\end{align*}
Again, we deal with each term respectively. 

First of all, 
\begin{align*}
\frac{\partial}{\partial\gamma}\E_{\gamma}\bracks{\mut_0(1, X) \mid R = 0}\vert_{\gamma = 0} = \Eb{\frac{1-R}{\Prb{R=0}}\prns{\mut_0(1, X) - \Eb{Y(1) \mid R = 0}}\score(Y, R, S, T, X)}.
\end{align*}

Secondly, 
\begin{align*}
&\frac{\partial}{\partial\gamma}\Eb{\E_\gamma\bracks{\tmut(1, X, S)\mid X, T=1,R=0} \mid R = 0}\vert_{\gamma = 0} \\
=& \Eb{\Eb{{\tmut(1, X, S)}\score(S \mid X, T, R) \mid X, T = 1, R = 0}\mid R = 0} \\
=& \Eb{\Eb{\prns{\tmut(1, X, S) - \mut(1, X)}\score(S \mid X, T, R) \mid X, T = 1, R = 0}\mid R = 0} \\
=& \Eb{\Eb{\prns{\tmut(1, X, S) - \mut(1, X)}\score(Y, S, X, T, R) \mid X, T = 1, R = 0}\mid R = 0} \\
=& \Eb{\frac{T(1-R)}{\Prb{T=1\mid R=0, X}\Prb{R= 0}}\prns{\tmut(1, X, S) - \mut_0(1, X)}\score(Y, S, X, T, R)}.
\end{align*}

Thirdly, 
\begin{align*}
&\frac{\partial}{\partial\gamma}\E\bracks{\E\bracks{\E_\gamma\bracks{Y \mid S, X, T = 1, R = 1} \mid X, T =1, R = 0} \mid R = 0}\vert_{\gamma = 0} \\
=& \E\bracks{\E\bracks{\E\bracks{Y \times\score(Y \mid S, X, T, R) \mid S, X, T = 1, R = 1} \mid X, T =1, R = 0} \mid R = 0} \\
=& \E\bracks{\E\bracks{\E\bracks{\prns{Y - \tmut(1, X, S)} \times\score(Y \mid S, X, T, R) \mid S, X, T = 1, R = 1} \mid X, T =1, R = 0} \mid R = 0} \\
=& \E\bracks{\E\bracks{\E\bracks{\prns{Y - \tmut(1, X, S)} \times\score(Y, S, X, T, R) \mid S, X, T = 1, R = 1} \mid X, T =1, R = 0} \mid R = 0} \\
=& \E\big[\E\big[\frac{R}{\Prb{R=1\mid S, X, T=1}}\frac{\Prb{R=0\mid S, X, T=1}}{\Prb{R= 0\mid X, T=1}}\prns{Y - \tmut(1, X, S)} \\&\qquad\qquad\qquad\times\score(Y, S, X, T, R) \mid X, T = 1\big] \mid R = 0\big] \\
=& \mathbb{E}\bigg[\frac{R}{\Prb{R=1\mid S, X, T=1}}\frac{\Prb{R=0\mid S, X, T=1}}{\Prb{R= 0\mid X, T=1}}\frac{T}{\Prb{T=1\mid X}}\frac{\Prb{R= 0 \mid X}}{\Prb{R=0}}\\
&\qquad\qquad\qquad \times \prns{Y - \tmut(1, X, S)}\score(Y, S, X, T, R)\bigg] \\
=& \Eb{\frac{R}{\Prb{R=1\mid S, X, T=1}}\frac{T}{\Prb{T=1\mid R=0,X}}\frac{\Prb{R=0\mid S, X, T=1}}{\Prb{R= 0}}\prns{Y - \tmut(1, X, S)}\score(Y, S, X, T, R)}.
\end{align*}

Putting the three equations above together, we have 
\begin{align*}
&\frac{\partial}{\partial\gamma}\E_{\gamma}\bracks{Y(1) \mid R = 0}\vert_{\gamma = 0}= \Eb{\tilde\psi_1(W)\score(Y, S, X, T, R)},
\end{align*}
where 
\begin{align*}
\tilde\psi_1(W) 
   &= \frac{1-R}{\Prb{R=0}}\prns{\mut_0(1, X) - \Eb{Y(1) \mid R = 0}} + \frac{T(1-R)}{\Prb{R= 0}\Prb{T=1\mid R=0, X}}\prns{\tmut(1, X, S) - \mut_0(1, X)} \\
  &+ \frac{TR}{\Prb{R= 0}\Prb{T=1\mid R=0,X}}\frac{\Prb{R=0\mid S, X, T}}{\Prb{R=1\mid S, X, T}}\prns{Y - \tmut(1, X, S)}.
\end{align*}
We can also use the decomposition in the proof for \Cref{thm: efficient-if} to show that $\tilde\psi_1$ belongs to the tangent space, so it is also the efficient influence function for $\Eb{Y(1) \mid R=0}$.
We can similarly derive the efficient influence function for $\Eb{Y(0) \mid R=0}$.
The final efficient influence function for $\delta_0^*$  is 
\begin{align*}
&\frac{1-R}{\Prb{R=0}}\prns{\mut_0(1, X) - \mut_0(0, X) - \delta_0^*} \\
+& \frac{1-R}{\Prb{R=0}}\biggl\{\frac{T}{\Prb{T=1\mid R=0, X}}\prns{\tmut(1, X, S) - \mut_0(1, X)} \\&- \frac{1- T}{1 - \Prb{T=1\mid R=0, X}}\prns{\tmut(0, X, S) - \mut_0(0, X)}\biggr\} \\
+& \frac{R}{\Prb{R=0}}\frac{\Prb{R=0\mid S, X, T}}{\Prb{R=1\mid S, X, T}} \times \\
&\qquad \braces{\frac{T}{\Prb{T=1\mid R=0,X}}\prns{Y - \tmut(1, X, S)} - \frac{1-T}{1-\Prb{T=1\mid R=0,X}}\prns{Y - \tmut(0, X, S)}}
\end{align*}
\end{proof}

\begin{proof}[Proof for \Cref{prop: latent-unconf}]
{We note that 
\begin{align*}
 &\Eb{\Eb{\Eb{Y(t) \mid S, X, T = t, R = 1} \mid X, T = t, R = 0} \mid R = 0} \\
=&\Eb{\Eb{\Eb{Y(t) \mid S(t), X, T = t, R = 1} \mid X, T = t, R = 0} \mid R = 0} \\
=& \Eb{\Eb{\Eb{Y(t) \mid S(t), X, R = 1} \mid X, R = 0} \mid R = 0} \\
=& \Eb{\Eb{\Eb{Y(t) \mid S(t), X, R = 0} \mid X, R = 0} \mid R = 0} \\
=& \Eb{Y(t) \mid R = 0}.
\end{align*}
Here the second equality follows from the assumptions $Y(t) \perp T \mid X, S(t), R = 1$, $S(t) \perp T \mid X, R = 0$ and the third equality follows from the assumption $Y(t) \perp R \mid X, S(t)$. This shows that the identification formula in \Cref{thm: effect-unlabelled} is still valid. 
}

{
Moreover, the asserted assumptions (all in terms of counterfactuals) impose no additional restrictions on the distributions of the observed variables, so we can still consider the model class $\mathcal{M}_{np}$ and its associated tangent space as we do in the proof of \Cref{thm: effect-unlabelled}. Because both the tanget space and the identification formula do not change, the efficiency bounds do not change either. 
}
\end{proof}

\subsection{Proofs for \cref{sec: estimation}}
\begin{proof}[Proof for \Cref{lemma: multiplicative-bias}.]
Let 
\begin{align*}
&\psi_1(W; \xi_1, \eta) = \mu(1, X) - \xi_1 + \frac{T}{e(X)}(\mu(1, X, S) - \mu(1 , X)) + \frac{TR}{e(X)r(1, X, S)}(Y -\tmu(1, X, S))), \\
&\psi_0(W; \xi_0, \eta) = \mu(0, X) - \xi_0 + \frac{1-T}{1 - e(X)}(\tmu(0, X, S) - \mu(0 , X)) + \frac{(1-T)R}{\prns{1 - e(X)}r(0, X, S)}(Y - \tmu(0, X, S))).
\end{align*}
Then $\psi(W; \delta, \eta) = \psi_1(W; \xi_1, \eta) - \psi_0(W; \xi_0, \eta)$ for any $\delta = \xi_1 + \xi_0$. 

By straightforward algebra, we can show that 
\begin{align*}
\Eb{\psi_1(W; \xi_1, \eta_0)-\psi_1(W; \xi_1, \eta^*)} 
   &= \Eb{(1-\frac{T}{\et(X)})\prns{\mut(1, X) - \mu_0(1, X)}} \\
   &+ \Eb{\frac{T}{\et(X)}\prns{1-\frac{R}{\rt(1, X, S)}}\prns{\tmu_0(1, X, S) - \tmut(1, X, S)}} \\
   &+ \Eb{\prns{\frac{T}{e_0(X)} - \frac{T}{\et(X)}}\prns{\mut(1, X) - \mu_0(1, X)}} \\
   &+ \Eb{\frac{T}{e_0(X)}\prns{1-\frac{R}{r_0(1, X, S)}}\prns{\tmu_0(1, X, S) - \tmut(1, X, S)}} \\
   &+ \Eb{\frac{R}{r_0(1, X, S)}\prns{\frac{T}{e_0(X)} - \frac{T}{\et(X)}}\prns{\tmu_0(1, X, S) - \tmut(1, X, S)}}. 
\end{align*}
Here the first two terms on the right hand side are equal to $0$ because $\Eb{T \mid X} = \et(X)$ and $\Eb{R \mid T = 1, X, S} = \rt(1, X, S)$. Moreover, 
\begin{align*}
\Eb{\prns{\frac{T}{e_0(X)} - \frac{T}{\et(X)}}\prns{\mut(1, X) - \mu_0(1, X)}} = \Eb{\frac{\et(X) - e_0(X)}{e_0(X)}\prns{\mut(1, X) - \mu_0(1, X)}},
\end{align*}
and 
\begin{align*}
\Eb{\frac{T}{e_0(X)}\prns{1-\frac{R}{r_0(1, X, S)}}\prns{\tmu_0(1, X, S) - \tmut(1, X, S)}} 
&= \mathbb{E}\bigg[\frac{\et(X)\prns{r_0(1, X, S) - \rt(1, X, S)}}{e_0(X)r_0(1, X, S)}\\
&\qquad \times \prns{\tmu_0(1, X, S) - \tmut(1, X, S)}\bigg],
\end{align*}
and 
\begin{align*}
&\Eb{\frac{R}{r_0(1, X, S)}\prns{\frac{T}{e_0(X)} - \frac{T}{\et(X)}}\prns{\tmu_0(1, X, S) - \tmut(1, X, S)}} \\
=& \Eb{\frac{r^*(1, X, S)}{r_0(1, X, S)}\frac{\et(X) - e_0(X)}{e_0(X)}\prns{\tmu_0(1, X, S) - \tmut(1, X, S)}}. 
\end{align*}

Similarly, we can derive $\Eb{\psi_0(W; \xi_1, \eta_0)-\psi_0(W; \xi_1, \eta^*)}$ and obtain 
\begin{align*}
\abs{\Eb{\psi(W; \delta, \eta_0)-\psi(W; \delta, \eta^*)}} 
   &\lesssim \|\et - e_0\|_2\prns{\|\mut(1, X) - \mu_0(1, X)\|_2 + \|\mut(0, X) - \mu_0(0, X)\|_2} \\
   &+ \|\et - e_0\|_2\prns{\|\tmut(1, X, S) - \tmu_0(1, X, S)\|_2 + \|\tmut(0, X, S) - \tmu_0(0, X, S)\|_2} \\
   &+ \|r_0(1, X, S) - \rt(1, X, S)\|_2\|\tmu_0(1, X, S) - \tmut(1, X, S)\|_2 \\
   &+ \|r_0(0, X, S) - \rt(0, X, S)\|_2\|\tmu_0(0, X, S) - \tmut(0, X, S)\|_2. 
\end{align*}
By following the proof of \Cref{lemma: convergence-rate}, we have
\begin{align*}
 &\max\braces{\|\tmu_0(0, X, S) - \tmut(0, X, S)\|_2, \|\tmu_0(1, X, S) - \tmut(1, X, S)\|_2} \lesssim \|\tmu_0(T, X, S) - \tmut(T, X, S)\|_2, \\
 &\max\braces{\|r_0(0, X, S) - \rt(0, X, S)\|_2, \|r_0(1, X, S) - \rt(1, X, S)\|_2} \lesssim \|r_0(T, X, S) - \rt(T, X, S)\|_2.
 \end{align*} 
 Therefore, 
 \begin{align*}
 \abs{\Eb{\psi(W; \delta, \eta_0)-\psi(W; \delta, \eta^*)}} \lesssim \|e_0 - \et\|_2\|\mu_0 - \mut\|_2 + \|e_0 - \et\|_2\|\tmu_0 - \tmut\|_2 +  \|r_0 - \rt\|_2\|\tmu_0 - \tmut\|_2.
 \end{align*}
\end{proof}

\begin{proof}[Proof for \cref{thm: DR}]
With slight abuse of notation, we define the following function for  $\eta = \prns{e, r, \mu, \tmu}$: 
\begin{align*}
\psi(W; \eta)
  & =  \mu(1, X) - \mu(0, X)  \\
&\phantom{=}+ \frac{TR}{e(X)r(1, X, S)}(Y - \tmu(1, X, S)) - \frac{(1 - T)R}{(1 - e(X))r(0, X, S)}(Y - \tmu(0, X, S)) \nonumber \\
    &\phantom{=}+ \frac{T}{e(X)}(\tmu(1, X, S) - \mu(1 , X)) - \frac{1 - T}{1 - e(X)}(\tmu(0, X, S) - \mu(0 , X)).
\end{align*}
Then our estimator is 
\begin{align*}
\hat\delta = \frac{1}{K}\sum_{k=1}^K \expectnk\left[\psi(W; \hat\eta_k)\right].
\end{align*}

We can decompose the estimation error of $\hat\delta$ as follows
\begin{align*}
\hat\delta - \delta^*&= \frac{1}{K}\sum_{k=1}^K\underbrace{\bracks{\prns{\expectnk\left[\psi(W; \hat\eta_k)\right] - \Eb{\psi(W; \hat\eta_k)\mid  \hat\eta_k}} - \prns{\expectnk\left[\psi(W; \eta_0)\right] - \expect\left[\psi(W; \eta_0)\right]}}}_{\mathcal{R}_{1, k}}\nonumber \\
&+ \frac{1}{K}\sum_{k=1}^K\underbrace{\bracks{\Eb{\psi(W; \hat\eta_k)\mid  \hat\eta_k} - \expect\left[\psi(W; \eta_0)\right]}}_{\mathcal{R}_{2, k}} + \underbrace{\frac{1}{N}\sum_{i=1}^N \bracks{\psi(W_i; \eta_0) - \delta^*}}_{\mathcal{R}_{3}}, 
\end{align*}
where $\eta_0$ is the limit of $\hat\eta_k$ for $k = 1, \dots, K$ as $N \to \infty$.

Here we can easily show that as $N \to \infty$, $\|\hat\eta_k - \eta_0\| \to 0$, we have 
\begin{align*}
\operatorname{Var}\prns{\mathcal{R}_{1, k} \mid \hat\eta_k} = o_p(1/N) \to 0.
\end{align*}
Moreover, we have $\mathcal{R}_{2, k} \to 0$ because $\|\hat\eta_k - \eta_0\| \to 0$. 
Finally, by law of large number, we have 
\begin{align*}
\abs{ \mathcal{R}_3} \to \abs{\Eb{\psi(W_i; \eta_0) - \delta^*} } = \abs{\Eb{\psi(W_i; \eta_0) - \psi(W_i; \eta^*)}}.
 \end{align*} 
 According to \Cref{lemma: multiplicative-bias}, the last display is equal to $0$ as long as 
 \[(\tmuo - \tmut)(\ro - \rt) = 0, ~~ (\muo - \mut)(\eo - \et) = 0, ~~ (\tmuo - \tmut)(\eo - \et) = 0. \] 
 Therefore, $\hat\delta \to \delta^*$ as $N \to \infty$ as long as $(\tmuo - \tmut)(\ro - \rt) = 0$, $(\muo - \mut)(\eo - \et) = 0$ and $(\tmuo - \tmut)(\eo - \et) = 0$.
\end{proof}

\begin{proof}[Proof for \Cref{thm: normality}.]
We start with the error decomposition in the proof for  \cref{thm: DR} with $\eta_0 = \eta^*$. 
As we mentioned there, 
\begin{align*}
\operatorname{Var}\prns{\mathcal{R}_{1, k} \mid \hat\eta_k} = o_p(1/N).
\end{align*}
So by Chebyshev's inequality, we have $\abs{\mathcal{R}_{1, k}} = o_p(N^{-1/2})$. 

Moreover, by \Cref{lemma: multiplicative-bias}, we have 
\begin{align*}
\abs{\mathcal{R}_{2, k}} 
   &\lesssim \|\hat e_k - \et\|_2\|\hat\mu_k - \mut\|_2 + \|\hat e_k - \et\|_2\|\hat\tmu_k - \tmut\|_2 + \|\hat r_k - \rt\|_2\|\hat\tmu_k - \tmut\|_2 \\
   &\le \sqrt{2}\epsilon^{-3/2} + \epsilon^{-3} = O_p\prns{\rho_{N, e}\rho_{N, \mu} + \rho_{N, e}\rho_{N, \tmu} + \rho_{N, r}\rho_{N, {\tmu}}}.
\end{align*}
So  under the conditions that $\rhor\rhotmu = o(N^{-1/2})$, 
$\rhoe\rhotmu = o(N^{-1/2})$, 
$\rhoe\rhomu = o(N^{-1/2})$, we have $\abs{\mathcal{R}_{2, k}} = o_p(N^{-1/2})$. 

Therefore,
\begin{align*}
\sqrt{N}\prns{\hat\delta - \delta^*}  = \frac{1}{\sqrt{N}}\sum_{i=1}^n \bracks{\psi(W_i; \eta^*) - \delta^*} + o_p(1).
\end{align*}
Then the asserted conclusion follows from  central limit theorem. 
\end{proof}

\begin{proof}[Proof for \cref{thm: conf-interval}]
Note that 
\[
|\hat{V} -  V^*| = |\hat{V} -  \expect[\psi^2(W; {\delta^*}, {\eta^*})]| = \frac{1}{K}\sum_{k = 1}^K|\expectnk[\psi^2(W; \hat{\delta}, \hat{\eta}_k)] - \expect[\psi^2(W; {\delta^*}, {\eta^*})]|,
\]
We only need to prove that $|\expectnk[\psi^2(W; \hat{\delta}, \hat{\eta}_k)] - \expect[\psi^2(W; {\delta^*}, {\eta^*})]| = o_p(1)$. Consider the following decomposition:
\begin{align*}
|\expectnk[\psi^2(W; \hat{\delta}, \hat{\eta}_k)] - \expect[\psi^2(W; {\delta^*}, {\eta^*})]|  
    &\le |\expectnk[\psi^2(W; \hat{\delta}, \hat{\eta}_k)] - \expectnk[\psi^2(W; {\delta^*}, {\eta^*})]| \\
    &+ |\expectnk[\psi^2(W; {\delta^*}, {\eta^*})] - \expect[\psi^2(W; {\delta^*}, {\eta^*})]| \\
    &= \mathcal{R}_{4} + \mathcal{R}_{5}.
\end{align*}
Thus we only need to prove that both $\mathcal{R}_{4}$ and $\mathcal{R}_{5}$ are $o_p(1)$.

\paragraph*{Bounding $\mathcal{R}_{5}$. } According to \Cref{lemma: convergence-rate}, 
\begin{align*}
\|\psi(W; {\delta^*}, {\eta^*})\|_q 
    &\le \frac{1}{\epsilon^2}(\|Y(1)\|_q + \|Y(0)\|_q) + \frac{1 + \epsilon}{\epsilon^2}(\|\tmut(1, X, S(1))\|_q + \|\tmut(0, X, S(0))\|_q) \\
    &+ \frac{1 + \epsilon}{\epsilon} (\|\mut(1, X)\|_q + \|\mut(0, X)\|_q) \le \left[\frac{2}{\epsilon^2} + \frac{2(1 + \epsilon)}{\epsilon^2} + \frac{2(1 + \epsilon)}{\epsilon}\right]C.
\end{align*}

If $q$ in \cref{assump: bounded} satisfies that $q \ge 4$, then 
\begin{align*}
\expect[\mathcal{R}_{5}^2] = \frac{1}{n}\var\{\psi^2(W; {\delta^*}, {\eta^*})\} \le \frac{1}{n}\var\{\psi^4(W; {\delta^*}, {\eta^*})\}.
\end{align*}
By Markov inequality, $\mathcal{R}_{5} = O_p(N^{-1/2})$.

If $2 < q < 4$, then we apply the von Bahr-Esseen inequality with $p = q/2 \in (1, 2)$: 
\begin{align*}
\expect[\mathcal{R}_{5}^{q/2}] \le 2n^{-q/2 + 1}\expect[\psi^q(W; {\delta^*}, {\eta^*})] = 2n^{-q/2 + 1}\|\psi(W; {\delta^*}, {\eta^*})\|_q^q.
\end{align*}

Thus $\expect[\mathcal{R}_{5}^{q/2}] = O(2N^{-q/2 + 1})$, which implies that $\mathcal{R}_{5} = O_p(N^{-1 + 2/q})$ according to Markov inequality.

Therefore $$\mathcal{R}_{5} = O_p(N^{-[(1 - 2/q) \vee 1/2]}) = o_p(1).$$

\paragraph*{Bounding $\mathcal{R}_{4}$.} Simple algebra shows that for any $a, \delta a$, 
\[
(a + \delta a)^2 - a^2 = \delta a(2a + \delta a).
\]
Now take $a + \delta a = \psi(W; \hat{\delta}, \hat{\eta}_k)$ and $a = \psi(W; {\delta^*}, {\eta^*})$, then
\begin{align*}
    &|\expectnk[\psi^2(W; \hat{\delta}, \hat{\eta}_k)] - \expectnk[\psi^2(W; {\delta^*}, {\eta^*})]| \\
=&  \left|\expectnk \bigg(\psi(W; \hat{\delta}, \hat{\eta}_k) - \psi(W; {\delta^*}, {\eta^*})\bigg)\bigg(2\psi(W; {\delta^*}, {\eta^*}) + \psi(W; \hat{\delta}, \hat{\eta}_k) - \psi(W; {\delta^*}, {\eta^*})\bigg)\right| \\
\le & \bigg(\expectnk \big[\psi(W; \hat{\delta}, \hat{\eta}_k) - \psi(W; {\delta^*}, {\eta^*})\big]^2\bigg)^{1/2}\bigg(\expectnk \big[2\psi(W; {\delta^*}, {\eta^*}) + \psi(W; \hat{\delta}, \hat{\eta}_k) - \psi(W; {\delta^*}, {\eta^*})\big]^2\bigg)^{1/2} \\
\le& \mathcal{R}_{6}^{1/2} \times (\mathcal{R}_{6}^{1/2} + 2(\expectnk [\psi^2(W; \delta^*, \eta^*)])^{1/2}).
\end{align*}
where $\mathcal{R}_{6} =\expectnk \big[\psi(W; \hat{\delta}, \hat{\eta}_k) - \psi(W; {\delta^*}, {\eta^*})\big]^2$.

Since $\expect [\psi^2(W; \delta^*, \eta^*)] = O(1)$, Markov inequality implies that $\expectnk [\psi^2(W; \delta^*, \eta^*)] = O_p(1)$.

Moreover,
\begin{align*}
\mathcal{R}_{6} \le 2\expectnk (\hat{\delta} - \delta^*)^2 + 2\expectnk \big[\psi(W; {\delta^*}, \hat{\eta}_k) - \psi(W; {\delta^*}, {\eta^*})\big]^2.
\end{align*}
Since $\hat{\delta} - \delta^* = o_p(1)$ according to \cref{thm: DR}, thus we only need to prove $\expectnk \big[\psi(W; {\delta^*}, \hat{\eta}_k) - \psi(W; {\delta^*}, {\eta^*})\big]^2 = o_p(1)$ as well. 
We can further decompose this term: 
\begin{align*}
\expectnk \big[\psi(W; {\delta^*}, \hat{\eta}_k) - \psi(W; {\delta^*}, {\eta^*})\big]^2 
   &\le \bigg|\expectnk\bracks{\prns{\psi(W; {\delta^*}, \hat{\eta}_k) - \psi(W; {\delta^*}, {\eta^*})}^2}\\
   &\qquad\qquad -\expect\bracks{\prns{\psi(W; {\delta^*}, \hat{\eta}_k) - \psi(W; {\delta^*}, {\eta^*})}^2 \mid \hat\eta_k}\bigg| \\
   &+ \expect\bracks{\prns{\psi(W; {\delta^*}, \hat{\eta}_k) - \psi(W; {\delta^*}, {\eta^*})}^2 \mid \hat\eta_k}. 
\end{align*}
Note that 
\begin{align*}
\Eb{\expectnk\bracks{\prns{\psi(W; {\delta^*}, \hat{\eta}_k) - \psi(W; {\delta^*}, {\eta^*})}^2} - \expect\bracks{\prns{\psi(W; {\delta^*}, \hat{\eta}_k) - \psi(W; {\delta^*}, {\eta^*})}^2 \mid \hat\eta_k} \mid \hat\eta_k} = 0,
\end{align*}
so  by Markov inequality, we have 
\begin{align*}
\abs{\expectnk\bracks{\prns{\psi(W; {\delta^*}, \hat{\eta}_k) - \psi(W; {\delta^*}, {\eta^*})}^2} - \expect\bracks{\prns{\psi(W; {\delta^*}, \hat{\eta}_k) - \psi(W; {\delta^*}, {\eta^*})}^2 \mid \hat\eta_k}} = o_p(1).
\end{align*}
Moreover, it is easy to verify that $\expect\bracks{\prns{\psi(W; {\delta^*}, \hat{\eta}_k) - \psi(W; {\delta^*}, {\eta^*})}^2 \mid \hat\eta_k} = o_p(1)$ as $\|\hat\eta_k - \eta^*\| = o_p(1)$.

Putting all above together, we have 
\begin{align*}
\abs{\mathcal{R}_4} = o_p(1).
\end{align*}

\paragraph*{Conclusion.} 
Therefore, $\hat{V} = V^* + o_p(1)$, and by Slutsky's theorem, 
\[
\frac{\sqrt{N}(\hat{\delta}  - \delta^*)}{\sqrt{\hat{V}}} \overset{d}{\to} \mathcal{N}(0, 1), 
\]
so that $\pr(\delta^* \in \operatorname{CI}) \to 1 - \alpha$.
\end{proof}

\subsection{Proofs for \cref{sec: extension}}\label{sec: proof-extension}

\begin{proof}[Proof for \cref{thm: vanishing-if}]
Since we only consider labeled data drawn from the conditional distribution of $(X, T, S, Y)$ given $R = 1$, we consider the following model for the distribution of the observed data 
\begin{align*}
    \tilde{\mathcal{M}}_{np} 
        = \bigg\{&f_{X, T, S, Y \mid R = 1}(X, T, S, Y \mid R = 1) 
            = f_{X \mid R= 1}(X \mid R = 1)\big[e(1, X)^T(1 - e(1, X))^{1 - T}\big]\\
            &\qquad f_{S \mid X, T, R = 1}(S \mid X, T, R = 1) f_{Y \mid S, T, X, R = 1}(Y\mid  S, T, X, R = 1):  \forall e(1, X) \in [\epsilon, 1- \epsilon], \\
            &\qquad  f_{X \mid R = 1}, f_{S \mid T, X, R = 1} \text{ and } f_{Y \mid S, T, X, R = 1} \text{are arbitrary density functions}  \bigg\}.
\end{align*}
The corresponding tangent space is 
\begin{align*}
\tilde{\Lambda}_{np} 
    &= \{\score\prns{Y, X, T, S} \in \mathcal{L}_2(Y, X, T, S): \expect[\score\prns{Y, X, T, S} \mid R = 1] = 0\}. 
\end{align*}

When \Cref{assump: unconfound,assump: mar-1} hold and $\Prb{T  = 1 \mid R = 1, X, S} \in (0, 1)$ almost surely, we can easily verify that the conclusion of \cref{lemma: identification-1} is still valid. In particular, we have 
\begin{align*}
\xi_1^* = \Eb{Y(1)} = \Eb{\Eb{\Eb{Y \mid T = 1, R = 1, X, S} \mid X, T = 1}}.
\end{align*}

Since $f^*_X, f^*_{S \mid X, T}$ are assumed to be known, when we analyze the path-differentiability of $\xi^*$ under parametric submodels for $\tilde{\mathcal{M}}_{np}$, we can fix these two distributions and only vary the distribution of $Y \mid T = 1, R = 1, X, S$.  In the following part, we suppress the subscripts in the density functions $f^*$, and the meaning of the density functions should be self-evident from the arguments. 
We have 
\begin{align*}
&\frac{\partial}{\partial \gamma}\expect_{\gamma}[Y(1)] \vert_{\gamma = \gamma^*} \\
    =&  \int f^*(x) f^*(s \mid X = x, T = 1)\frac{\partial}{\partial \gamma}\expect_{\gamma}[Y \mid X = x, S = s, T = 1, R = 1] \vert_{\gamma = 0}\diff x \diff s \\
    =& \int f^*(x) f^*(s \mid X = x, T = 1)\expect[Y \times \score\prns{Y \mid X, S, T, R=1} \mid X = x, S = s, T = 1, R = 1]\diff x \diff s \\
    =& \int f^*(x)\Eb{\frac{f^*(S \mid X, T = 1)}{f^*(S \mid X, T = 1, R = 1)}\prns{Y - \tmut(T, X, S)} \times \score\prns{Y \mid X, S, T, R=1} \mid X = x, T = 1, R = 1}\diff x \\
    =&  \int f^*(x)\Eb{\frac{T}{\et(1, X)}\frac{f^*(S \mid X, T = 1)}{f^*(S \mid X, T = 1, R = 1)}\prns{Y - \tmut(T, X, S)} \times \score\prns{Y \mid X, S, T, R=1} \mid X = x, R = 1}\diff x \\
    =& \Eb{\frac{T}{\et(1, X)}\frac{f^*(S \mid X, T = 1)}{f^*(S \mid X, T = 1, R = 1)}\frac{f^*(X)}{f^*(X \mid R = 1)}\prns{Y - \tmut(1, X, S)} \times \score\prns{Y \mid X, S, T, R=1} \mid R = 1} \\
    =& \Eb{\frac{T}{\et(1, X)}\frac{f^*(S \mid X, T = 1)}{f^*(S \mid X, T = 1, R = 1)}\frac{f^*(X)}{f^*(X \mid R = 1)}\prns{Y - \tmut(1, X, S)} \times \score\prns{Y, X, S, T \mid R=1} \mid R = 1}. 
\end{align*}
In the equation above, we use $\et(1, X)$ to denote $\Prb{T = 1 \mid R = 1, X}$. 

By repeatedly applying Bayes' rule, we can show that 
\begin{align*}
&\frac{T}{\et(1, X)}\frac{f^*(S \mid X, T = 1)}{f^*(S \mid X, T = 1, R = 1)}\frac{f^*(X)}{f^*(X \mid R = 1)} \\
=& \frac{T}{\et(X)}\frac{\Prb{T=1}}{\Prb{T=1\mid R=1}}\frac{f^*(S, X \mid T = 1)}{f^*(S, X \mid T = 1, R = 1)}
\end{align*}

This means that the following is an influence function for $\xi_1^*$: 
\begin{align*}
  \frac{T}{\et(X)}\frac{\Prb{T=0}}{\Prb{T=0\mid R=1}}\frac{f^*(S, X \mid T = 1)}{f^*(S, X \mid T = 1, R = 1)}\prns{Y - \tmut(1, X, S)}. 
  \end{align*} 
Similarly, we can show that an influence function for $\xi_0^*$ is given by 
\begin{align*}
\frac{1 - T}{1 - \et(X)}\frac{\Prb{R=1}}{\Prb{R=1\mid T=0}}\frac{f^*(S, X \mid T = 0)}{f^*(S, X \mid  T = 0, R = 1)}(Y - \tmut(0, X, S)). 
\end{align*}
These together mean that \Cref{eq: vanishing-if} gives an influence function for the average treatment effect $\Eb{Y(1) - Y(0)}$. Obviously this influence function belongs to the tangent space $\tilde{\Lambda}_{np}$, so it is also the efficient influence function for the average treatment effect. 
\end{proof}

\begin{proposition}\label{lemma: overlap-t}
If \Cref{assump: overlap} holds, then $\Prb{T = 1 \mid R= 1, X, S} \in \prns{\epsilon^2, \frac{1-\epsilon}{\epsilon}}$. 
\end{proposition}
\begin{proof}[Proof for \Cref{lemma: overlap-t}]
By Bayes' rule, 
\begin{align*}
\Prb{T = 1 \mid R = 1, X, S} = \frac{\Prb{R = 1 \mid T = 1, X, S}\Prb{T = 1 \mid X, S}}{\Prb{R = 1 \mid X, S}} \in \prns{\epsilon^2, \frac{1-\epsilon}{\epsilon}}.
\end{align*}
\end{proof}

\begin{proof}[Proof for \Cref{prop: efficient-bound-rescale}]
We only need to prove that $\tilde V^* = \expect[\tilde{\psi}^2(W; \delta^*, \tilde{\eta}^*) \mid R = 1]$ is equal to the following quantity: 
\begin{align*}
& \Prb{R=1}\Eb{\prns{\frac{TR}{\et(X)\rt(1, X, S)}(Y - \tmut(1, X, S)) - \frac{(1 - T)R}{(1 - \et(X))\rt(0, X, S)}(Y - \tmut(0, X, S))}^2} \\
&=\Eb{\frac{T^2\mathbb{P}^2\prns{R=1}}{e^{*2}(X)r^{*2}(1, X, S)}(Y - \tmut(1, X, S))^2 + \frac{(1-T)^2\mathbb{P}^2\prns{R=1}}{(1-e^{*}(X))^2r^{*2}(0, X, S)}(Y - \tmut(0, X, S))^2 \mid R = 1}
\end{align*}
According to the Bayes' rule, 
\begin{align*}
\frac{\Prb{R=1}}{\rt(t, X, S)} = \frac{f^*(S, X \mid T = t)}{f^*(S, X \mid T = t, R = 1)}\frac{\Prb{T = t}}{\Prb{T = t \mid R = 1}} = \lambda^*(S, X, t)\frac{\Prb{T = t}}{\Prb{T = t \mid R = 1}}.
\end{align*}
Thus 
\begin{align*}
&\Eb{\frac{T^2\mathbb{P}^2\prns{R=1}}{e^{*2}(X)r^{*2}(1, X, S)}(Y - \tmut(1, X, S))^2 + \frac{(1-T)^2\mathbb{P}^2\prns{R=1}}{(1-e^{*}(X))^2r^{*2}(1, X, S)}(Y - \tmut(0, X, S))^2 \mid R = 1} \\
=&\expect\bigg[\frac{T^2 \lambda^{*2}(S, X, T)\mathbb{P}^2\prns{T = 1}}{e^{*2}(X)\mathbb{P}^2\prns{T = 1 \mid R = 1}}(Y - \tmut(1, X, S))^2 + \frac{(1-T)^2\lambda^{*2}(S, X, T)\mathbb{P}^2\prns{T=0}}{(1-e^{*}(X))^2\mathbb{P}^2\prns{T=0 \mid R = 1}}(Y - \tmut(0, X, S))^2 \mid R = 1\bigg] \\
=& \expect[\tilde{\psi}^2(W; \delta^*, \tilde{\eta}^*) \mid R = 1] = \tilde V^*.
\end{align*}
\end{proof}

\begin{proof}[Proof for \Cref{thm: normality2}]
We use $\expect^{(N)}$ and $\op{Var}^{(N)}$ to denote the expectation and variance operators with respect to the $\pr^{(N)}$ distribution described in \Cref{sec: est-vanish}. 
We only need to prove the following: 
\begin{align*}
\sqrt{\bar{N}_l}(\hat\delta - \delta^*) \overset{d}{\to} \mathcal{N}(0, \tilde V^*), ~~ \sqrt{\bar{N}_l}(\hat\delta^{\op{rev}} - \delta^*) \overset{d}{\to} \mathcal{N}(0, \tilde V^*). 
\end{align*}
Then the asserted conclusions follow from Slutsky's theorem and  that $N_l/\bar{N}_l = (N_l/N)/\pi_N = o_p(1)$. 

\textbf{Proving the first statement regarding $\hat\delta$}. We can directly use the error decomposition in the proof for \cref{thm: DR} with $\eta_0$ there being replaced by $\eta^*_N = (\et_N, \rt_N, \mut, \tmut)$. 
We can show that given $\|\hmuk - \mut\| = o_p(1)$,  $\|\htmuk - \tmut\| = o_p(1)$, $\|\hek - \et_N\| = o_p(1)$, and $\|\rt_N/\hrk - 1\| = o_p(1)$, the stochastic equicontinuity term $\mathcal{R}_{1, k}$ satisfies that 
\begin{align*}
&\op{Var}^{(N)}(\mathcal{R}_{1, k} \mid \hat\eta_k)\\ 
=& o_p(N^{-1}) + \frac{1}{N}\expect^{(N)}\left[{\prns{\frac{TR}{\hek \hrk} - \frac{TR}{\et_N \rt_N}}^2(Y - \htmuk)^2} \mid \hek, \hrk, \htmuk\right] + \frac{1}{N}\expect^{(N)}\left[\frac{TR}{e^{*2}_N r^{*2}_N}(\tmut - \htmuk)^2 \mid \htmuk \right] \\
+& \frac{1}{N}\expect^{(N)}\left[{\prns{\frac{(1-T)R}{(1-\hek) \hrk} - \frac{(1-T)R}{(1-\et_N) \rt_N}}^2(Y - \htmuk)^2} \mid \hek, \hrk, \htmuk\right] + \frac{1}{N}\expect^{(N)}\left[\frac{(1-T) R}{(1-e^{*}_N)^2 r^{*2}_N}(\tmut - \htmuk)^2 \mid \htmuk \right].
\end{align*}
Note that 
\begin{align*}
&\frac{1}{N}\expect^{(N)}\left[{\prns{\frac{TR}{\hek \hrk} - \frac{TR}{\et_N \rt_N}}^2(Y - \htmuk)^2} \mid \hek, \hrk, \htmuk\right]  \\
  \lesssim& \frac{1}{N}\expect^{(N)}\left[\frac{\prns{\et_N\rt_N/\hek\hrk - 1}^2}{\et_N\rt_N}\mid \hek, \hrk\right] \\
  \lesssim& \frac{1}{N}\expect^{(N)}\left[\frac{e^{*2}_N/\hek^2(\rt_N/\hrk - 1)^2}{\et_N\rt_N} + \frac{(\hek - \et_N)^2}{\hek^2\et_N\rt_N}\mid \hek, \hrk\right] \\
  \lesssim& \frac{1}{\pi_N N}\prns{\|\rt_N/\hrk - 1\|^2 + \|\hek - \et_N\|^2} = o_p((N\pi_N)^{-1}) = o_p(\bar{N}_l^{-1}), 
\end{align*}
and 
\begin{align*}
\frac{1}{N}\expect^{(N)}\left[\frac{TR}{e^{*2}_N r^{*2}_N}(\tmut - \htmuk)^2 \mid \htmuk \right] 
&\lesssim \frac{1}{N}\expect^{(N)}\left[\frac{(\tmut - \htmuk)^2}{\et_N\rt_N}\mid \htmuk \right] \\
&\lesssim \frac{1}{\pi_N N}\|\htmuk - \tmut\| = o_p((N\pi_N)^{-1}) = o_p(\bar{N}_l^{-1}).
\end{align*}
Similarly, we can show that other terms in the decomposition of $\op{Var}^{(N)}(\mathcal{R}_{1, k} \mid \hat\eta_k)$ are also $o_p(\bar{N}_l^{-1})$. Therefore, $\op{Var}^{(N)}(\mathcal{R}_{1, k} \mid \hat\eta_k) = o_p(\bar{N}_l^{-1})$. By Chebyshev inequality, we have $\abs{\mathcal{R}_{1, k}} = o_p(\bar{N}_l^{-1/2})$. 

Moreover, we can follow the proof of \Cref{lemma: multiplicative-bias} to show that 
\begin{align*}
\abs{\mathcal{R}_{2, k}}
\le & \abs{\expect^{(N)}\left[\prns{\frac{T}{\hek(X)} - \frac{T}{\et_N(X)}}\prns{\mut(1, X) - \hmuk(1, X)} \mid \hmuk, \hek \right]} \\
+& \abs{\expect^{(N)}\left[\frac{T}{\hek(X)}\prns{1-\frac{R}{\hrk(1, X, S)}}\prns{\htmuk(1, X, S) - \tmut(1, X, S)} \mid \hek, \htmuk\right]} \\
+& \abs{\expect^{(N)}\left[\prns{\frac{1-T}{1-\hek(X)} - \frac{1-T}{1-\et_N(X)}}\prns{\mut(0, X) - \hmuk(0, X)} \mid \hmuk, \hek \right]} \\
+& \abs{\expect^{(N)}\left[\frac{1-T}{1-\hek(X)}\prns{1-\frac{R}{\hrk(0, X, S)}}\prns{\htmuk(0, X, S) - \tmut(0, X, S)} \mid \hek, \htmuk\right]} \\
+& \abs{\expect^{(N)}\left[\frac{\rt_N(1, X, S)}{\hrk(1, X, S)}\prns{\frac{T}{\hek(X)} - \frac{T}{\et_N(X)}}\prns{\htmuk(1, X, S) - \tmut(1, X, S)} \mid \hek, \htmuk \right]} \\
+& \abs{\expect^{(N)}\left[\frac{\rt_N(0, X, S)}{\hrk(0, X, S)}\prns{\frac{1-T}{1-\hek(X)} - \frac{1-T}{1-\et_N(X)}}\prns{\htmuk(0, X, S) - \tmut(0, X, S)} \mid \hek, \htmuk \right]} \\
\lesssim & \|\hek - \et_N\|\|\hmuk - \tmu\| + \|\hek - \et_N\|\|\htmuk - \tmut\| + \left\|\frac{\rt_N}{\hrk} - 1\right\|\|\htmuk - \mut\| \\
=& O_p\prns{\rho_{N, e}\rho_{\bar{N}_l, \mu} + \rho_{N, e}\rho_{\bar{N}_l, \tmu} + \rho_{\bar{N}_l, r}\rho_{\bar{N}_l, \tilde\mu}} = o_p\prns{\bar{N}_l^{-1/2}}. 
\end{align*}
Given that $\abs{\mathcal{R}_{1, k}} = o_p(\bar{N}_l^{-1/2})$ and $\abs{\mathcal{R}_{2, k}} = o_p(\bar{N}_l^{-1/2})$ for $k = 1, \dots, K$, we have 
\begin{align*}
&\sqrt{\bar{N}_l}\prns{\hat\delta - \delta^*} = \sum_{i=1}^N Z_{i, N} + o_p(1), 
\end{align*}
where 
\begin{align*}
Z_{i, N} 
=& \frac{\sqrt{\pi_N}}{\sqrt{N}}\bigg\{\mut(1, X_i) - \mut(0, X_i)  - \delta^* \\
&+ \frac{T_i}{\et_N(X_i)}(\tmut(1, X_i, S_i) - \mut(1 , X_i)) - \frac{1 - T_i}{1 - \et_N(X_i)}(\tmut(0, X_i, S_i) - \mut(0 , X_i)) \\
&+ \frac{T_iR_i}{\et_N(X_i)\rt_N(1, X_i, S_i)}(Y_i - \tmut(1, X_i, S_i)) - \frac{(1 - T_i)R_i}{(1 - \et_N(X_i))\rt_N(0, X_i, S_i)}(Y_i - \tmut(0, X_i, S_i))\bigg\}.
\end{align*}
Let $S_N = \sum_{i=1}^N Z_{i, N}$. 
We can easily verify that $\expect^{(N)}\bracks{S_{N}} = \sum_{i=1}^N\expect^{(N)}\bracks{Z_{i, N}} = 0$, and 
\begin{align*}
\op{Var}^{(N)}\prns{S_N} 
  &= \pi_N \expect^{(N)}\bracks{\prns{\mut(1, X_i) - \mut(0, X_i)  - \delta^*}^2} \\
  &+ \pi_N \expect^{(N)}\bracks{\prns{\frac{T_i}{\et_N(X_i)}(\tmut(1, X_i, S_i) - \mut(1 , X_i)) - \frac{1 - T_i}{1 - \et_N(X_i)}(\tmut(0, X_i, S_i) - \mut(0 , X_i))}^2} \\
  &+ \pi_N \expect^{(N)}\bigg[\bigg(\frac{T_iR_i}{\et_N(X_i)\rt_N(1, X_i, S_i)}(Y_i - \tmut(1, X_i, S_i)) \\
  &\qquad\qquad\qquad\qquad\qquad - \frac{(1 - T_i)R_i}{(1 - \et_N(X_i))\rt_N(0, X_i, S_i)}(Y_i - \tmut(0, X_i, S_i))\bigg)^2\bigg] \\
  &\to 0 + 0 + \tilde V^*. 
\end{align*}
To prove the asymptotic normality, we will  use the Lindberg-Feller Central Limit Theorem. To this end, we now verify the Lyapunov condition. Note that for any $q > 2$, the above already shows that $\prns{\op{Var}^{(N)}\prns{S_N}}^q = O(1)$. Then we only need to verify that 
\begin{align*}
\expect^{(N)}\bracks{\sum_{i=1}^N \abs{Z_{i, N}}^q} 
  &= N \expect^{(N)}\bracks{\abs{Z_{i, N}}^q} = N\|Z_{i, N}\|_q^q \to 0.
\end{align*}
We note that 
\begin{align*}
& N\|Z_{i, N}\|_q^q 
  \le \bigg[\frac{\pi_N^{1/2}}{N^{1/2-1/q}}
  \left\|{\mut(1, X_i) - \mut(0, X_i)  - \delta^*}\right\|_q \\
  &\quad + \frac{\pi_N^{1/2}}{N^{1/2-1/q}}\left\|\frac{T_i}{\et_N(X_i)}(\tmut(1, X_i, S_i) - \mut(1 , X_i)) - \frac{1 - T_i}{1 - \et_N(X_i)}(\tmut(0, X_i, S_i) - \mut(0 , X_i))\right\|_q \\
  &\quad + \frac{\pi_N^{1/2}}{N^{1/2-1/q}}\left\|\frac{T_iR_i}{\et_N(X_i)\rt_N(1, X_i, S_i)}(Y_i - \tmut(1, X_i, S_i)) - \frac{(1 - T_i)R_i}{(1 - \et_N(X_i))\rt_N(0, X_i, S_i)}(Y_i - \tmut(0, X_i, S_i))\right\|_q \bigg]^q.
\end{align*}
Under the regularity conditions in \Cref{sec: more-extension} \Cref{assump: moment}, we have  
\begin{align*}
&\frac{\pi_N^{1/2}}{N^{1/2-1/q}}
  \left\|{\mut(1, X_i) - \mut(0, X_i)  - \delta^*}\right\|_q = O\prns{\frac{\pi_N^{1/2}}{N^{1/2-1/q}}} \to 0, \\
&\frac{\pi_N^{1/2}}{N^{1/2-1/q}} \left\|\frac{T_i}{\et_N(X_i)}(\tmut(1, X_i, S_i) - \mut(1 , X_i)) - \frac{1 - T_i}{1 - \et_N(X_i)}(\tmut(0, X_i, S_i) - \mut(0 , X_i))\right\|_q = O\prns{\frac{\pi_N^{1/2}}{N^{1/2-1/q}}} \to 0.  
\end{align*}
Moreover, according to the relationship between $\pi_N$ and $\rt_N$, we have 
\begin{align*}
 &\frac{\pi_N^{1/2}}{N^{1/2-1/q}}\left\|\frac{T_iR_i}{\et_N(X_i)\rt_N(1, X_i, S_i)}(Y_i - \tmut(1, X_i, S_i)) - \frac{(1 - T_i)R_i}{(1 - \et_N(X_i))\rt_N(0, X_i, S_i)}(Y_i - \tmut(0, X_i, S_i))\right\|_q \\
 =& \frac{\pi_N^{1/2}}{N^{1/2-1/q}}\bigg\|\frac{R_i}{\pi_N}\frac{T_i\lambda^*(S_i, X_i, T_i)\Prb{T_i=1}}{\et_N(X_i)\Prb{T_i = 1\mid R_i=1}}(Y_i - \tmut(1, X_i, S_i))  \\
 &\qquad\qquad\qquad\qquad\qquad\qquad  -\frac{R_i}{\pi_N} \frac{(1 - T_i)\lambda^*(S_i, X_i, T_i)\Prb{T_i=0}}{(1 - \et_N(X_i))\Prb{T_i = 0 \mid R_i = 1}}(Y_i - \tmut(0, X_i, S_i))\bigg\|_q \\
 =& O\prns{\frac{1}{(N\pi_N)^{1/2-1/q}}} \to 0.  
 \end{align*} 
 This means that the Lyapunov condition holds. Then by the Lindberg-Feller Central Limit Theorem, 
 \begin{align*}
  \frac{\sqrt{\bar{N}_l}\prns{\hat\delta - \delta^*}}{\sqrt{\op{Var}^{(N)}\prns{S_N}}} \overset{d}{\to} \mathcal{N}(0, 1). 
  \end{align*} 
  Since $\op{Var}^{(N)}\prns{S_N} \to \tilde V^*$, we further have 
   \begin{align*}
  \sqrt{\bar{N}_l}\prns{\hat\delta - \delta^*} \overset{d}{\to} \mathcal{N}(0, \tilde V^*). 
  \end{align*}

\textbf{Proving the second statement regarding $\hat\delta^{\op{rev}}$}. For any $\eta = (e, \lambda, \mu, \tmu)$ and $\pi, \nu_1, \nu_2$, we define 
\begin{align*}
\tilde \psi'(W; \pi, \nu_1, \nu_0, \eta)   &= \frac{T}{e(X)}\prns{\tmu(1, X, S) - \mu(1, X)} - \frac{1-T}{1 - e(X)}\prns{\tmu(0, X, S) - \mu(0, X)} + \mu(1, X) - \mu(0, X) \\
  &\qquad + \frac{R\lambda(S, X, T)}{\pi}\frac{T\nu_1}{e(X)}(Y - \tmu(1, X, S)) - \frac{R\lambda(S, X, T)}{\pi}\frac{(1-T)\nu_0}{1-e(X)}(Y - \tmu(1, X, S)). 
\end{align*}
Then 
\begin{align*}
\hat\delta^{\op{rev}} = \frac{1}{K}\sum_{k=1}^K \hat\expect_k\bracks{\tilde \psi'(W; \hat \pi_N, \hat \nu_1, \hat \nu_0, \hat{\tilde\eta}_k)}.
\end{align*}
We can further decompose the estimation error of $\hat\delta^{\op{rev}}$ as follows: 
\begin{align*}
\hat\delta^{\op{rev}} - \delta^* = \frac{1}{K}\sum_{k=1}^K \tilde{\mathcal{R}}_{1, k} + \tilde{\mathcal{R}}_{2, k} + \tilde{\mathcal{R}}_{3, k},
\end{align*}
where 
\begin{align*}
\tilde{\mathcal{R}}_{1, k} 
  &= \prns{\hat\expect_k\bracks{\tilde \psi'(W; \hat \pi_N, \hat \nu_1, \hat \nu_0, \hat{\tilde\eta}_k)} - \Eb{\tilde \psi'(W; \hat \pi_N, \hat \nu_1, \hat \nu_0, \hat{\tilde\eta}_k) \mid \hat{\tilde{\eta}}_k}} \\
  &\qquad - \prns{\hat\expect_k\bracks{\tilde \psi'(W; \pi_N, \nu^*_1,  \nu^*_0, {\tilde\eta}^*_N)} - \Eb{\tilde \psi'(W; \pi_N, \nu^*_1,  \nu^*_0, {\tilde\eta}^*_N)}},  \\
\tilde{\mathcal{R}}_{2, k} &= \Eb{\tilde \psi'(W; \hat \pi_N, \hat \nu_1, \hat \nu_0, \hat{\tilde\eta}_k) \mid \hat{\tilde{\eta}}_k} - \Eb{\tilde \psi'(W; \pi_N, \nu^*_1,  \nu^*_0, {\tilde\eta}^*_N)},  \\
\tilde{\mathcal{R}}_{3, k} &=\hat\expect_k\bracks{\tilde \psi'(W; \pi_N, \nu^*_1,  \nu^*_0, {\tilde\eta}^*_N)} - \delta^*,
\end{align*}
and $\tilde\eta^*_N = \prns{\et, \lambda^*_N, \mut, \tmut}$. 

Here $\tilde{\mathcal{R}}_{1, k}$ is again a stochastic equicontinuity term. It is again $o_p\prns{\bar{N}_l^{-1/2}}$ because of sample splitting and $\abs{\hat \pi_N/\pi_N - 1} = o_p(1)$, $\abs{\hat \nu_1 -  \nu^*_1} = o_p(1)$, $\abs{\hat \nu_0 -  \nu^*_0} = o_p(1)$, and all nuisance estimators in $\hat{\tilde{\eta}}_k$ are consistent.  

Moreover, we can easily verify that 
\begin{align*}
\abs{\tilde{\mathcal{R}}_{2, k}} 
  &\lesssim \|\hek - \et_N\|\|\hmuk - \mut\| + \prns{\|\hat\lambda_k - \lambda^*_N\| + \|\hek - \et_N\| + |\hat \nu_1 - \nu_1^*| + |\hat \nu_0 - \nu_0^*| + \abs{\frac{\pi_N}{\hat \pi_N} - 1}}\|\htmuk - \tmut\| \\
  &= o_p(\bar{N}_l^{-1/2}). 
\end{align*}

Therefore, 
\begin{align*}
\sqrt{\bar{N}_l}\prns{\hat\delta^{\op{rev}} - \delta^*} = \sqrt{\bar{N}_l}\prns{\hat\expect_k\bracks{\tilde \psi'(W; \pi_N, \nu^*_1,  \nu^*_0, {\tilde\eta}^*_N)} - \delta^*} + o_p(1). 
\end{align*}

Finally, we can similarly  apply the Lindberg-Feller Central Limit Theorem to show that 
\begin{align*}
\sqrt{\bar{N}_l}\prns{\hat\expect_k\bracks{\tilde \psi'(W; \pi_N, \nu^*_1,  \nu^*_0, {\tilde\eta}^*_N)} - \delta^*} \overset{d}{\to} \mathcal{N}\prns{0, \tilde V^*}. 
\end{align*}
\end{proof}

\subsection{Proofs for \Cref{sec: criteria}}\label{sec: proof-criteria}
\begin{proof}[Proof for \Cref{prop: sufficient-stat-surrogacy}]
In order to prove that $T \perp Y \mid X, S$, we need to verify that for any $x \in \mathcal{X}$, $s \in \mathcal{S}$, and $y \in \mathcal{Y}$,
\[
    \pr(Y \le y \mid T = 1, X = x, S = s) = \pr(Y \le y \mid T = 0, X = x, S = s), 
\]
or equivalently, 
\begin{align}\label{eq: eq-to-verify}
\pr(Y(1, s) \le y \mid T = 1, X = x, S(1) = s) = \pr(Y(0, s) \le y \mid T = 0, X = x, S(0) = s). 
\end{align}

We note that condition \ref{cond: S-Y-unconfound} in \Cref{prop: sufficient-stat-surrogacy} implies that 
\begin{align*}
\pr(Y(0, s) \le y \mid T = 0, X = x, S(0) = s)  = \pr(Y(0, s) \le y \mid T = 0, X = x), \\
 \pr(Y(1, s) \le y \mid T = 1, X = x, S(1) = s)  = \pr(Y(1, s) \le y \mid T = 1, X = x).
\end{align*}
Then the condition \ref{cond: no-direct} in \Cref{prop: sufficient-stat-surrogacy} implies that 
\[
    \pr(Y(0, s) \le y \mid T = 0, X = x) = \pr(Y(1, s) \le y \mid T = 0, X = x). 
\]
Moreover, the condition \ref{cond: T-Y-unconfound} in  \Cref{prop: sufficient-stat-surrogacy} implies that 
\[
    \pr(Y(1, s) \le y \mid T = 0, X = x) =  \pr(Y(1, s) \le y \mid T = 1, X = x).
\]
These equations together ensure \cref{eq: eq-to-verify}. 
\end{proof}

\subsection{Proofs for \Cref{sec: missing-pattern}}
\begin{proof}[Proof for \Cref{thm: efficiency-intermediate}]
    First, we consider the following model:
    \begin{align*}
        \mathcal{M}_{np, I-II} 
        &= \bigg\{f_{X, T, R, Y, S, R_S}(X, T, R, Y, S, R_S) = f_X(X)\big[e(X)^T(1 - e(X))^{1 - T}\big][r(T, X)^R(1 - r(T, X))^{1- R}] \\
        &\times f^R_{Y \mid R = 1, T, X}(Y, R, T, X)\times [r_S(T, X, R, Y)^{R_S}(1- r_S(T, X, R, Y))^{1 - R_S}] f^{R_S}_{S \mid R_S = 1, T, X, Y}(S, R_S, T, X, Y): \\
        &\qquad f_X,  f_{Y \mid R = 1, T, X}, f_{S \mid R_S = 1, T, X, Y} \text{ are arbitrary density functions, and } e(X), r(T, X), r_S(T, X) \\
        &\qquad \text{ are arbitrary functions obeying } e(X) \in [\epsilon, 1- \epsilon], r(T, X) \in [\epsilon, 1], r_S(T, X) \in [\epsilon, 1]  \bigg\}. 
    \end{align*}
    The tangent space of this model is equal to 
    \begin{align*}
        \Lambda_{I-II} = \Lambda_I + \oplus \bar{\Lambda}(R_S \mid X, T, R, Y) + \oplus \bar{\Lambda}(S \mid R_S, X, T, R, Y), 
    \end{align*} 
    where $\Lambda_I$ is the tangent space for the model $\mathcal{M}_{np, I}$ in the proof for \cref{thm: efficiency-comparison}, and $\bar{\Lambda}_{R_S \mid X, T, R, Y}$ and $\bar{\Lambda}_{S \mid R_S, X, T, R, Y}$ are mean square closures of the following sets: 
    \begin{align*}
        {\Lambda}_{R_S \mid X, T, R, Y}
        &= \{\score_{R_S \mid X, T, R, Y}\prns{R_S,  X, T, R, Y} \in L_2(X, T, R, Y): \\
        &\qquad\qquad\qquad\qquad  \Eb{\score_{R_S \mid X, T, R, Y}\prns{R_S,  X, T, R, Y} \mid  X, T, R, Y} = 0\}, \\
        \Lambda_{S \mid R_S, Y, R, X, T}
        &= \{R_S \times \score_{S \mid R_S = 1, Y, R, X, T}(S, Y, R, X, T) \in L_2(S, Y, R, X, T): \\
        &\qquad\qquad\qquad\qquad \expect[\score_{S \mid R_S = 1, Y, R, X, T}(S, Y, R, X, T) \mid R_S = 1, Y, R, X, T] = 0\}.
    \end{align*}
    Then we can easily show that the efficient influence function of $\delta^*$ corresponding to this tangent space is identical to the efficient influence function of $\delta^*$ corresponding to the tangent space $\Lambda_i$, by following the proof for \cref{thm: efficiency-comparison}. Indeed,
    we first note that the efficient influence function $\psi_{I, 1}$ of $\Eb{Y(1)}$ is based on path differentiability analysis of the  following identification formula under parametric submodels:
    \begin{align*}
        \E\bracks{Y(1)} = \E\bracks{\E\bracks{Y \mid X, T = 1, R = 1}}.
    \end{align*} 
    This identification formula is also valid under the setting I-II and we already know that 
    \begin{align*}
        \frac{\partial}{\partial\gamma} \E_\gamma\bracks{Y(1)}\vert_{\gamma=0} = \Eb{\psi_{I, 1}(Y, X, T, R)\score(Y, X, T, R)},
    \end{align*} 
    for $\score(Y, X, T, R) \in \Lambda_I$. Moreover, we can easily verify that  for any $\score(S, R_S \mid Y, X, T, R) \in {\Lambda}_{R_S \mid X, T, R, Y} \oplus  \Lambda_{S \mid R_S, Y, R, X, T}$, we have $\Eb{\score(S, R_S, Y, X, T, R) \mid  Y, X, T, R = 0}$. 
    This further implies that 
    \begin{align*}
        \frac{\partial}{\partial\gamma} \E_\gamma\bracks{Y(1)}\vert_{\gamma=0} = \Eb{\psi_{I, 1}(Y, X, T, R)\prns{\score(Y, X, T, R) + \score(S, R_S \mid Y, X, T, R)}}.
    \end{align*} 
    Thus $\psi_{I, 1}$ is also an influence function of $\Eb{Y(1)}$ under model $\mathcal{M}_{np, I-II}$. Moreover, we have $\psi_{I, 1} \in \Lambda_I \subseteq \Lambda_{I-II}$, so $\psi_{I, 1}$ is also the efficient influence function of $\Eb{Y(1)}$ under model $\mathcal{M}_{np, I-II}$. Similarly, we can prove that   the efficient influence function of $\delta^*$ under the model $\mathcal{M}_{np, I-II}$ is also the efficient influence function $\psi_I$ under the model $\mathcal{M}_{np, I}$. This gives our desired conclusion. 
    
    Next, we note that under the asserted assumptions, we have 
    $R \perp S \mid T, X, R_S$, and $R_S \perp Y \mid T, X, R, S$. We consider the following model: 
    \begin{align*}
        \mathcal{M}_{np, II-III} 
        &= \bigg\{f_{X, T, S, R, Y}(X, T, S, R, Y) = f_X(X)\big[e(X)^T(1 - e(X))^{1 - T}\big][r_S(X, T)^{R_S}(1 - r_S(X, T))^{1-R_S}]\\
        &\qquad  f^{R_S}_{S \mid X, T, R_S = 1}(S, X, T)[r(X, T, R_S)^R(1 - r(X, T, R_S))^{1- R}]f^R_{Y \mid S, X, T, R = 1}(Y, S, X, T):  \\
        &\qquad\qquad f_X, f_{S \mid X, T, R_S = 1}, f_{Y \mid S, X, T, R = 1}~\text{are arbitrary density functions of the distributions}  \\
        &\qquad\qquad\text{indicated by their respective subscripts, and } e(X), r_S(X, T), r(X, T) \text{ are arbitrary}\\&\qquad\qquad\text{functions obeying }e(X) \in [\epsilon, 1- \epsilon], r_S(X, T), r(X, T, R_S) \in [\epsilon, 1]\bigg\}.
    \end{align*}
    The corresponding tangent space is given by 
    \begin{align*}
        \Lambda_{II-III} = \overline\Lambda_X \oplus \overline\Lambda_{T \mid X} \oplus \overline\Lambda_{R_S \mid X, T} \oplus  \overline\Lambda_{S \mid R_S, X, T} \oplus \overline\Lambda_{R \mid X, T, R_S}  \oplus  \overline\Lambda_{Y \mid  S, X, T, R},
    \end{align*}
    where $\overline\Lambda_X, \overline\Lambda_{T \mid X}, \overline{\Lambda}_{R \mid X, T}$ are given in the proof for \Cref{thm: efficiency-comparison}, and $\overline\Lambda_{R_S \mid X, T}, \overline\Lambda_{S \mid R_S, X, T}$ and $\overline{\Lambda}_{Y \mid  S, X, T,  R}$ are the mean-square closures of the following sets:
    \begin{align*}
        \Lambda_{R_S \mid X, T}  &= \{\score_{R_S\mid X, T}(R_S, X, T) \in L_2(R_S, X, T): \expect[\score_{R_S \mid X, T}(R_S, X, T) \mid X, T] = 0\} \\
        \Lambda_{S \mid R_S, X, T}  &= \{R_S \times \score_{S \mid R_S = 1, X, T}(S, X, T) \in L_2(S, R_S, X, T): \expect[\score_{S \mid R_S = 1, X, T}(S, X, T) \mid R_S = 1, X, T] = 0\} \\
        \Lambda_{R \mid X, T, R_S}  &= \{\score_{R \mid X, T, R_S}(R, X, T, R_S) \in L_2(R, X, T, R): \expect[\score_{R \mid X, T, R_S}(R, X, T, R_S) \mid X, T, R_S] = 0\} \\
        \Lambda_{Y \mid S, X, T, R}  &= \{R \times \score_{Y \mid S, X, T, R = 1}(Y, S, X, T) \in L_2(Y, S, X, T, R): \\
        &\qquad\qquad\qquad\qquad \expect[\score_{Y \mid S, X, T, R = 1}(Y, S, X, T) \mid S, X, T, R = 1] = 0\}.
    \end{align*} 
    
    Again, we focus on the counterfactual mean $\xi_1^* = \Eb{Y(1)}$. We note that under the asserted assumptions, we have 
    \begin{align*}
        &\Eb{\Eb{\Eb{Y \mid X, T = 1, S, R = 1} \mid R_S = 1, X, T = 1}} \\
        =& \Eb{\Eb{\Eb{Y(1) \mid X, T = 1, S(1), R = 1} \mid R_S = 1, X, T = 1}} \\
        =&  \Eb{\Eb{\Eb{Y(1) \mid X, T = 1, S(1)} \mid X, T = 1}} \\
        =& \Eb{\Eb{\Eb{Y(1) \mid X, S(1)} \mid X}} \\
        =& \Eb{Y(1)},
    \end{align*}
    where the second equality holds because $R \perp (Y(t), S(t)) \mid X, T$ according to \cref{assump: mar-1,assump: MAR2} and $R_S \perp S(t) \mid X, T$ according to \cref{assump: mar3}   
    and the third equality holds because $T \perp (Y(1), S(1)) \mid X$.
    Then, to derive an influence function of $\Eb{Y(1)}$, we need consider the following path-differentiability analysis under a regular parametric submodel indexed by a parameter $\gamma$:
        \begin{align*}
        &\frac{\partial}{\partial\gamma}\E_{\gamma}\bracks{\E_{\gamma}\bracks{\E_{\gamma}\bracks{Y \mid X, T = 1, S, R = 1} \mid R_S = 1, X, T = 1}}\vert_{\gamma = 0} \\
        =&\frac{\partial}{\partial\gamma}\E_{\gamma}\bracks{\E\bracks{\E\bracks{Y \mid X, T = 1, S, R = 1} \mid R_S = 1, X, T = 1}}\vert_{\gamma = 0} \\
        +& \frac{\partial}{\partial\gamma}\E\bracks{\E_{\gamma}\bracks{\E\bracks{Y \mid X, T = 1, S, R = 1} \mid R_S = 1, X, T = 1}}\vert_{\gamma = 0} \\
        +& \frac{\partial}{\partial\gamma}\E\bracks{\E\bracks{\E_{\gamma}\bracks{Y \mid X, T = 1, S, R = 1} \mid R_S = 1, X, T = 1}}\vert_{\gamma = 0}.
    \end{align*}
    We can evaluate each of the derivatives respectively. We have 
    \begin{align*}
        \frac{\partial}{\partial\gamma}\E_{\gamma}\bracks{\E\bracks{\E\bracks{Y \mid X, T = 1, S, R = 1} \mid R_S = 1, X, T = 1}}\vert_{\gamma = 0} = \Eb{\prns{\mut(1, X) - \xi_1^*}\score(X, T, R_S, R, S, Y)},
    \end{align*}
    and 
    \begin{align*}
        &\frac{\partial}{\partial\gamma}\E\bracks{\E_{\gamma}\bracks{\E\bracks{Y \mid X, T = 1, S, R = 1} \mid R_S = 1, X, T = 1}}\vert_{\gamma = 0} \\
        &=  \Eb{\Eb{\tmut(1, X, S) \score(S \mid R_S, X, T) \mid R_S = 1, X, T = 1}} \\
        &= \Eb{\Eb{(\tmut(1, X, S) - \mut(1, X)) \score(S \mid R_S, X, T) \mid R_S = 1, X, T = 1}} \\
        &=  \Eb{\Eb{\frac{TR_S}{\et(X)\rt_S(1, X)}(\tmut(1, X, S) - \mut(1, X)) \score(X, T, R_S, S) \mid X}} \\
        &= \Eb{\frac{TR_S}{\et(X)\rt_S(1, X)}(\tmut(1, X, S) - \mut(1, X)) \score((X, T, R_S, S, R, Y)},
    \end{align*}
    and 
        \begin{align*}
        &\frac{\partial}{\partial\gamma}\E\bracks{\E\bracks{\E_{\gamma}\bracks{Y \mid X, T = 1, S, R = 1} \mid R_S = 1, X, T = 1}}\vert_{\gamma = 0} \\
        =& \E\bracks{\E\bracks{\E\bracks{\prns{Y - \tmut(T, X, S)}\score(Y \mid X, T, S, R) \mid X, T = 1, S, R = 1} \mid R_S = 1, X, T = 1}} \\
        =& \Eb{\Eb{\Eb{\frac{R}{\rt(T, X)}\prns{Y - \tmut(T, X, S)}\score(Y \mid X, T, S, R) \mid X, T = 1, S} \mid R_S = 1, X, T= 1}} \\
        =& \Eb{\Eb{\Eb{\frac{R}{\rt(T, X)}\prns{Y - \tmut(T, X, S)}\score(Y \mid X, T, S, R) \mid X, T = 1, S} \mid X, T= 1}} \\
        =& \Eb{\Eb{\frac{R}{\rt(T, X)}\prns{Y - \tmut(T, X, S)}\score(Y \mid X, T, S, R) \mid X, T= 1}} \\
        =& \Eb{\frac{RT}{\rt(T, X)\et(X)}\prns{Y - \tmut(T, X, S)}\score(X, T, S, R_S, R, Y)},
    \end{align*}
    where the second equality holds because $R \perp S \mid X, T$ following \cref{assump: mar-1,assump: MAR2}, the third equality holds because $R_S \perp S \mid X, T$ following \cref{assump: mar3}.
    
    Thus we have that the following function $\psi_{1, II-III}$ is an influence function of $\xi_1^*$: 
    \begin{align*}
        \psi_{1, II-III}(X, T, R, R_S, S, Y) 
        &= \mut(1, X) - \xi_1^* + \frac{TR_S}{\et(X)\rt_S(1, X)}(\tmut(1, X, S) - \mut(1, X)) \\
        +& \frac{TR}{\et(X)\rt(T, X)}\prns{Y - \tmut(T, X, S)}.
    \end{align*}
    It is easy to verify that 
    \begin{align*}
        &\mut(1, X) - \xi_1^* \in \Lambda_{X}, \\
        &\frac{TR_S}{\et(X)\rt_S(1, X)}(\tmut(1, X, S) - \mut(1, X)) \in \Lambda_{S \mid R_S, X, T}, \\
        &\frac{TR}{\et(X)\rt(T, X)}\prns{Y - \tmut(T, X, S)} \in \Lambda_{Y \mid S, X, T, R}.
    \end{align*}
    This means that $\psi_{1, II-III}(X, T, R, R_S, S, Y) \in \Lambda_{II-III}$, so it is the efficient influence function. We can similarly derive the efficient influence function of $\xi_0^* = \Eb{Y(0)}$ and verify that the efficient influence function of $\delta^*$ is given by $\psi_{II-III}$ states in this theorem. 
\end{proof}

\begin{proof}[Proof for \Cref{thm: efficiency-comparison-intermediate}]
    \textbf{We first derive $V^*_{II} - V^*_{II-III}$}. We can decompose $\psi_{II-III}$ into six different terms: 
    \begin{align*}
    \psi_{\text{II-III}}(W; \delta^*, \eta^*) = \Psi_1 + \Psi_2 + \Psi_3 - (\Psi_4 + \Psi_5 + \Psi_6),
    \end{align*}
    where 
    \begin{align*}
    &\Psi_1 = \mut(1, X)  - \xi_1^*, \Psi_2 = \frac{TR}{\et(X)\rt(1, X)}(Y - \tmut(1, X, S)),  \Psi_3 =  \frac{TR_S}{\et(X)\rt_S(1, X)}\prns{\tmut(1, X, S) - \mut(1, X)} \\
    &\Psi_4 = \mut(0, X)  - \xi_0^*, \Psi_5 = \frac{(1-T)R}{(1-\et(X))\rt(0, X)}(Y - \tmut(0, X, S)), \\
    &\Psi_6 = \frac{(1-T)R_S}{(1-\et(X))\rt_S(0, X)}\prns{\tmut(0, X, S) - \mut(0, X)}. 
    \end{align*}
    Then 
    \begin{align*}
    \Eb{\psi^2_{\text{II-III}}(W; \delta^*, \eta^*)} = \op{Var}(\psi_{\text{II-III}}(W; \delta^*, \eta^*)) = \sum_{i=1}^6 \op{Var}(\Psi_i) + \sum_{i\ne j}\op{Cov}(\Psi_i, \Psi_j). 
    \end{align*}
    It is easy to verify that $\op{Cov}(\Psi_i, \Psi_j) = 0$ for all $i, j$ except $i = 1, j = 4$. So we have 
    \begin{align*}
    V^*_{II-III} = \Eb{\psi^2_{\text{II-III}}(W; \delta^*, \eta^*)} = \op{Var}(\Psi_1 - \Psi_4) + \op{Var}(\Psi_2) + \op{Var}(\Psi_3) + \op{Var}(\Psi_5) + \op{Var}(\Psi_6). 
    \end{align*}
    Similarly, we have that 
    \begin{align*}
    V^*_{II} = \op{Var}(\Psi_1 - \Psi_4) + \op{Var}\prns{\frac{TR}{\et(X)\rt(1, X)}(Y - \mut(1, X))} + \op{Var}\prns{\frac{(1 - T)R}{(1 - \et(X))\rt(0, X)}(Y - \mut(0, X))}.
    \end{align*}
    We note that 
    $$\frac{TR}{\et(X)\rt(1, X)}(Y - \mut(1, X)) = \Psi_2 + \Psi_3 + \prns{\frac{R}{\rt(1, X)} - \frac{R_S}{\rt_S(1, X)}}\frac{T}{\et(X)}(\tmut(1, X, S) - \mut(1, X)).$$
    Moreover, we can easily verify that  
    \begin{align*}
    \op{Cov}(\Psi_2, \Psi_3) 
     &= \op{Cov}\left(\Psi_2, \prns{\frac{R}{\rt(1, X)} - \frac{R_S}{\rt_S(1, X)}}\frac{T}{\et(X)}(\tmut(1, X, S) - \mut(1, X))\right)= 0 
    \end{align*}
    It follows that 
    \begin{align*}
    &\op{Var}\prns{\frac{TR}{\et(X)\rt(1, X)}(Y - \mut(1, X))} \\
    =& \op{Var}\prns{\Psi_2} + \op{Var}\prns{\Psi_3} + 2\op{Cov}\left(\Psi_3, \prns{\frac{R}{\rt(1, X)} - \frac{R_S}{\rt_S(1, X)}}\frac{T}{\et(X)}(\tmut(1, X, S) - \mut(1, X))\right) \\
    +& \Eb{\prns{\frac{R}{\rt(1, X)} - \frac{R_S}{\rt_S(1, X)}}^2\frac{T}{(\et(X))^2}(\tmut(1, X, S) - \mut(1, X))^2} \\
    =& \op{Var}\prns{\Psi_2} + \op{Var}\prns{\Psi_3} + \Eb{\frac{T}{(\et(X))^2}\prns{\frac{R}{\rt(1, X)} - \frac{R_S}{\rt_S(1, X)}}\prns{\tmut(1, X, S) - \mut(1, X)}^2} \\
    =& \op{Var}\prns{\Psi_2} + \op{Var}\prns{\Psi_3} + \Eb{\frac{\rt_S(1, X) - \rt(1, X)}{\et(X)\rt(1, X)\rt_S(1, X)}\prns{\tmut(1, X, S) - \mut(1, X)}^2}
    \end{align*}
    Similarly, we have 
    \begin{align*}
    &\op{Var}\prns{\frac{(1 - T)R}{(1 - \et(X))\rt(0, X)}(Y - \mut(0, X))} \\
    =& \op{Var}\prns{\Psi_5} + \op{Var}\prns{\Psi_6} + \Eb{\frac{1 - T}{(1 - \et(X))^2}\prns{\frac{R}{\rt(0, X)} - \frac{R_S}{\rt_S(0, X)}}\prns{\tmut(0, X, S) - \mut(0, X)}^2} \\
    =& \op{Var}\prns{\Psi_5} + \op{Var}\prns{\Psi_6} + \Eb{\frac{\rt_S(0, X) - \rt(0, X)}{(1 - \et(X))\rt(0, X)\rt_S(0, X)}\prns{\tmut(0, X, S) - \mut(0, X)}^2}. 
    \end{align*}
    Therefore, we have 
    \begin{align*}
     &V_{II}^* - V^*_{II-III} \\
    =& \Eb{\frac{\rt_S(1, X) - \rt(1, X)}{\et(X)\rt(1, X)\rt_S(1, X)}\op{Var}\bracks{\tmut(1, X, S(1)) \mid X} + \frac{\rt_S(0, X) - \rt(0, X)}{(1 - \et(X))\rt(0, X)\rt_S(0, X)}\op{Var}\bracks{\tmut(0, X, S(0)) \mid X}}.
    \end{align*}

    \textbf{Now we derive $V^*_{II-III} - V^*_{III}$}. We can similarly show that 
    \begin{align*}
    V_{III}^* 
        &=  \op{Var}(\Psi_1 - \Psi_4) + \op{Var}(\Psi_2) + \op{Var}(\Psi_5)  \\
        &+ \op{Var}\left(\Psi_3 + \frac{T}{\et(X)}\prns{1 - \frac{R_S}{\rt_S(1, X)}}\prns{\tmut(1, X, S) - \mut(1, X)}\right) \\
        &+ \op{Var}\left(\Psi_6 + \frac{1 - T}{1 - \et(X)}\prns{1 - \frac{R_S}{\rt_S(0, X)}}\prns{\tmut(0, X, S) - \mut(0, X)}\right). 
    \end{align*}
    Note that 
    \begin{align*}
    &\op{Var}\left(\Psi_3 + \frac{T}{\et(X)}\prns{1 - \frac{R_S}{\rt_S(1, X)}}\prns{\tmut(1, X, S) - \mut(1, X)}\right) \\
    =& \op{Var}\left(\Psi_3\right) + \Eb{\frac{T}{(\et(X))^2}\prns{1 - \frac{R_S}{(\rt_S(1, X))^2}}\prns{\tmut(1, X, S) - \mut(1, X)}^2} \\
    =& \op{Var}\left(\Psi_3\right) - \Eb{\frac{1 - \rt_S(1, X)}{\et(X)\rt_S(1, X)}\op{Var}\bracks{\tmut(1, X, S(1)) \mid X}},
    \end{align*}
    and similarly, 
    \begin{align*}
     &\op{Var}\left(\Psi_6 + \frac{1 - T}{1 - \et(X)}\prns{1 - \frac{R_S}{\rt_S(0, X)}}\prns{\tmut(0, X, S) - \mut(0, X)}\right)  \\
     =& \op{Var}\left(\Psi_6\right) - \Eb{\frac{1 - \rt_S(0, X)}{(1 - \et(X))\rt_S(0, X)}\op{Var}\bracks{\tmut(0, X, S(0)) \mid X}}. 
    \end{align*}
    Therefore, 
    \begin{align*}
    V_{II-III}^* - V_{III}^* 
        &= \Eb{\frac{1 - \rt_S(1, X)}{\et(X)\rt_S(1, X)}\op{Var}\bracks{\tmut(1, X, S(1)) \mid X}} \\
        &+ \Eb{\frac{1 - \rt_S(0, X)}{(1 - \et(X))\rt_S(0, X)}\op{Var}\bracks{\tmut(0, X, S(0)) \mid X}}.
    \end{align*}
\end{proof}

\subsection{Proofs for \Cref{sec: more-extension}}
\begin{proof}[Proof for \Cref{thm: efficiency-comp-extension}]
By following the proof of \Cref{thm: efficiency-comparison}, we can show that the efficient influence functions for settings I and II are identical. We thus only need to consider setting I. Specifically, 
consider the following model: 
\begin{align*}
    \tilde{\mathcal{M}}_{np, I} 
        = \bigg\{&f_{X, T, Y \mid R = 1}(X, T, Y \mid R = 1) 
            = f_{X \mid R= 1}(X \mid R = 1)\big[e(1, X)^T(1 - e(1, X))^{1 - T}\big]\\
            &\qquad f_{Y \mid  T, X, R = 1}(Y\mid T, X, R = 1):  \forall e(1, X) \in [\epsilon, 1- \epsilon], \\
            &\qquad  f_{X \mid R = 1} \text{ and } f_{Y \mid S, T, X, R = 1} \text{are arbitrary density functions}  \bigg\}.
\end{align*}
The corresponding tangent space is 
\begin{align*}
\tilde{\Lambda}_{np, I} 
    &= \{\score\prns{Y, X, T} \in \mathcal{L}_2(Y, X, T): \expect[\score\prns{Y, X, T} \mid R = 1] = 0\}. 
\end{align*}

Note that under assumptions in  \Cref{thm: efficiency-comp-extension},
\begin{align*}
\xi_1^* 
    &= \expect[Y(1)] =  \expect\left[\expect[Y(1) \mid X]\right] =  \expect\left[\expect[Y \mid X, T = 1, R = 1]\right] \\
    &=\iint yf^*_X(x)f_{Y \mid X, T = 1, R = 1}(y \mid x, T = 1, R = 1)dxdy,
\end{align*}
where the unconditional density function $f^*_X(x)$ is known. 

Again we consider parametric submodels indexed by $\gamma$ in path-differentiability analysis for $\xi^*_1$. In the following analysis, we suppress the subscripts in the density functions to ease the notations. 
\begin{align*}
&\frac{\partial}{\partial \gamma}\expect_{\gamma}[Y(1)] \vert_{\gamma = \gamma^*} \\
    &=  \int f^*(x) \frac{\partial}{\partial \gamma}\expect_{\gamma}[Y \mid X = x, T = 1, R = 1] \vert_{\gamma = 0}dxds \\
    &= \int f^*(x)\expect[Y \times \score\prns{Y \mid X, T} \mid X = x, T = 1, R = 1]dx  \\
    &= \int f^*(x)\Eb{\frac{T}{\et(1, X)}(Y - \mut(T, X))\score\prns{Y \mid X, T} \mid X = x, R = 1}dx \\
    &= \Eb{\frac{f^*(X)}{f^*(X \mid R = 1)}\frac{T}{\et(1, X)}(Y - \mut(T, X))\score\prns{Y \mid X, T}  \mid R = 1} \\
        &= \Eb{\frac{f^*(X)}{f^*(X \mid R = 1)}\frac{T}{\et(1, X)}(Y - \mut(T, X))\score\prns{Y, X, T}  \mid R = 1}. 
\end{align*}
Moreover, we can apply Bayes' rule to show that
\begin{align*}
 \frac{f^*(X)}{f^*(X \mid R = 1)}\frac{T}{\et(1, X)} = \frac{T\lambda^*(X, 1)}{\et(1, X)}\frac{\Prb{T=1}}{\Prb{T=1 \mid R = 1}}.
 \end{align*}
 This means that  $\frac{T\lambda^*(X, 1)}{\et(1, X)}\frac{\Prb{T=1}}{\Prb{T=1 \mid R = 1}}\prns{Y - \mut(T, X)}$ is an influence function for $\xi_1^*$. It is easy to verify that this influence function belongs to the tangent space, so it is also the efficient influence function for $\xi_1^*$. Similarly, we can show that the efficient influence function for $\xi_0^*$ is $\frac{(1-T)\lambda^*(X, 0)}{1-\et(1, X)}\frac{\Prb{T=0}}{\Prb{T=0 \mid R = 1}}\prns{Y - \mut(T, X)}$. This establishes the efficient influence function in \Cref{thm: efficiency-comp-extension}: 
 \begin{align*}
    \tilde{\psi}_I(W; \delta^*, \tilde{\eta}^*) = \tilde{\psi}_{II}(W; \delta^*, \tilde{\eta}^*)  = \frac{T\lambda^*(X, 1)}{\et(X)}(Y - \mut(1, X)) - \frac{(1 - T)\lambda^*(X, 0)}{1 - \et(X)}(Y - \mut(0, X)). 
\end{align*}

 Moreover, under the additional \Cref{assump: MAR2}, we can easily show that $\lambda^*(S, X, t) = \lambda^*(X, t)$, so the efficient influence function $\tilde\psi(W; \delta^*; \tilde\eta^*)$ in \Cref{thm: vanishing-if} reduces to 
 \begin{align*}
 \tilde\psi(W; \delta^*; \tilde\eta^*) 
 &= \frac{T\lambda^*(X, 1)}{\et(1, X)}\frac{\Prb{T=1}}{\Prb{T=1 \mid R = 1}}\prns{Y - \tmut(T, X, S)} \\
  &\quad -  \frac{(1-T)\lambda^*(X, 0)}{\et(1, X)}\frac{\Prb{T=0}}{\Prb{T=0 \mid R = 1}}\prns{Y - \tmut(T, X, S)}.
 \end{align*}

We note that 
\begin{align*}
\tilde{\psi}_I(W; \delta^*, \tilde{\eta}^*) = \tilde{\psi}_{II}(W; \delta^*, \tilde{\eta}^*) = \tilde\psi(W; \delta^*; \tilde\eta^*)  + \omega(W; \delta^*, \tilde\eta^*), 
\end{align*}
where 
\begin{align*}
\omega(W; \delta^*, \tilde\eta^*) 
  &= \frac{T\lambda^*(X, 1)}{\et(1, X)}\frac{\Prb{T=1}}{\Prb{T=1 \mid R = 1}}\prns{\tmut(T, X, S) - \mut(T, X)} \\
  & -  \frac{(1-T)\lambda^*(X, 0)}{\et(1, X)}\frac{\Prb{T=0}}{\Prb{T=0 \mid R = 1}}\prns{\tmut(T, X, S)  - \mut(T, X)}. 
\end{align*}
It is easy to verify that $\omega(W; \delta^*, \tilde\eta^*)$ is uncorrelated with $\tilde\psi(W; \delta^*; \tilde\eta^*)$ given $R = 1$. Therefore, 
\begin{align*}
\tilde V^*_I - \tilde V^*  = \tilde V^*_I - \tilde V^* &= \Eb{\omega^2(W; \delta^*, \tilde\eta^*) \mid R = 1} \\
        &= \expect\bigg[\frac{\lambda^{*2}(X, 1)}{\et(X)}\frac{\prns{\mathbb{P}\prns{T=1}}^2}{\prns{\mathbb{P}\prns{T=1 \mid R = 1}}^2}\var\{\tmut(1, X, S(1))\mid X\} \\
        &\qquad\qquad + \frac{\lambda^{*2}(X, 0)}{1 - \et(X)}\frac{\prns{\mathbb{P}\prns{T=0}}^2}{\prns{\mathbb{P}\prns{T=0 \mid R = 1}}^2}\var\{\tmut(0, X, S(0))\mid X\} \mid R = 1\bigg].
\end{align*}
\end{proof}

\begin{proof}[Proof for \Cref{corollary: effect-unlabelled-vanishing}]
    The proof is identical to the proof for \Cref{thm: vanishing-if}, noting that the distribution of $(X, T, S)$ on the unlabelled population $R = 0$ is identical to its distribution on the combined population. 
\end{proof}

\begin{proof}[Proof for \Cref{corollary: effect-unlabelled-vanishing-rescale}]
    Again, we only need to prove that $\tilde{V}^* =  \expect[\tilde{\psi}^2(W; \delta^*, \tilde{\eta}^*) \mid R = 1]$ for $\tilde{\psi}$ in \Cref{eq: vanishing-if} is equal to the following quantity: 
    \begin{align*}
        &\pr(R = 1)\Eb{\prns{\frac{R}{\Prb{R=0}}\frac{\Prb{R=0\mid S, X, T}}{\Prb{R=1\mid S, X, T}}{
\frac{T-\et(0,X)}{\et(0, X)(1-\et(0, X))}\prns{Y - \tmut(T, X, S)}}}^2} \\
    =& \Eb{\frac{T^2\pr^2(R = 1)}{e^{*2}(0, X)\rt(1, X, S)}\prns{Y - \tmut(1, X, S)}^2 + \frac{(1-T)^2\pr^2(R = 1)}{e^{*2}(0, X)\rt(0, X, S)}\prns{Y - \tmut(0, X, S)}^2 \mid R = 1}
    \end{align*}
    
We can again apply Bayes' rule to $\pr(R = 1)/\rt(t, X, S)$, and show that the quantity above is equal to   
\begin{align*}
    &\expect\bigg[\frac{T^2 \lambda^{*2}(S, X, T)\mathbb{P}^2\prns{T = 1}}{e^{*2}(0, X)\mathbb{P}^2\prns{T = 1 \mid R = 1}}(Y - \tmut(1, X, S))^2 \\
    &\qquad\qquad + \frac{(1-T)^2\lambda^{*2}(S, X, T)\mathbb{P}^2\prns{T=0}}{(1-e^{*}(0, X))^2\mathbb{P}^2\prns{T=0 \mid R = 1}}(Y - \tmut(0, X, S))^2 \mid R = 1\bigg].
\end{align*}
In the limit we have $\et(0, X) = \pr(T = 1 \mid R = 0, X) = \pr(T = 1 \mid X) = \et(X)$. So this is identical to $\tilde V^* = \expect[\tilde{\psi}^2(W; \delta^*, \tilde{\eta}^*) \mid R = 1]$. 
\end{proof}

\subsection{Proofs for \Cref{sec: att}}
\begin{proof}[Proof for \Cref{lemma: identification-att}]
{
We note that 
\begin{align*}
\Eb{\Eb{Y \mid T = 1, R = 1, X, S} \mid T = 1}
    &= \Eb{\Eb{Y(1) \mid T = 1, X, S(1)} \mid T = 1} \\
    &= \Eb{Y(1) \mid T = 1}. 
\end{align*}
Moreover, 
\begin{align*}
\Eb{\Eb{Y \mid T = 0, R = 1, X, S}\mid X, T = 0}
&= \Eb{\Eb{Y(0) \mid T = 0, R = 1, X, S(0)}\mid X, T = 0} \\
&= \Eb{\Eb{Y(0) \mid T = 0, X, S(0)}\mid X, T = 0} \\
&= \Eb{\Eb{Y(0) \mid T = 0, X, S(0)}\mid X, T = 1} \\
&= \Eb{\Eb{Y(0) \mid T = 1, X, S(0)}\mid X, T = 1} \\
&=\Eb{Y(0)  \mid X, T = 1}.
\end{align*}
Thus 
\begin{align*}
\Eb{\Eb{\Eb{Y \mid T = 0, R = 1, X, S}\mid X, T = 0} \mid T = 1} = \Eb{Y(0) \mid T = 1}.
\end{align*}
The equations above imply the conclusion in \Cref{eq: identif-1}. 
}
\end{proof}

\begin{proof}[Proof for \Cref{thm: efficient-if-att}]
{
Again, we consider parametric submodels indexed by parameters $\gamma$ as in the proof for \Cref{thm: efficient-if}.
} 

{
We first note that 
\begin{align*}
\frac{\partial}{\partial\gamma}\expect_\gamma\bracks{\expect_{\gamma}\bracks{Y \mid T = 1, R = 1, X, S}\mid T = 1}\vert_{\gamma = 0} 
    =& \frac{\partial}{\partial\gamma}\expect_\gamma\bracks{\expect\bracks{Y \mid T = 1, R = 1, X, S}\mid T = 1}\vert_{\gamma = 0} \\
    +& \frac{\partial}{\partial\gamma}\expect\bracks{\expect_{\gamma}\bracks{Y \mid T = 1, R = 1, X, S}\mid T = 1}\vert_{\gamma = 0}. 
\end{align*}
Here 
\begin{align*}
&\frac{\partial}{\partial\gamma}\expect_\gamma\bracks{\expect\bracks{Y \mid T = 1, R = 1, X, S}\mid T = 1}\vert_{\gamma = 0} \\
    =& \Eb{\prns{\tmut(1, X, S) - \Eb{\tmut(1, X, S)\mid T = 1}}\times\score(T, S, X) \mid T = 1} \\
    =& \Eb{\frac{T}{\pr\prns{T = 1}}\prns{\tmut(1, X, S) - \Eb{\tmut(1, X, S)\mid T = 1}}\times\score(Y, R, T, S, X)}
\end{align*}
and 
\begin{align*}
&\frac{\partial}{\partial\gamma}\expect\bracks{\expect_{\gamma}\bracks{Y \mid T = 1, R = 1, X, S}\mid T = 1}\vert_{\gamma = 0} \\
=& {\Eb{\Eb{Y \times \score(Y \mid T, R, S, X) \mid T = 1, R = 1, X, S} \mid T = 1}}  \\
=& {\Eb{\Eb{\prns{Y - \tmut(1, X, S)} \times \score(Y \mid T, R, S, X) \mid T = 1, R = 1, X, S} \mid T = 1}} \\
=& {\Eb{\Eb{\frac{R}{\rt(1, X, S)}\prns{Y - \tmut(1, X, S)} \times \score(Y, T, R, S, X) \mid T = 1, X, S} \mid T = 1}} \\
=& {{\Eb{\frac{R}{\rt(1, X, S)}\prns{Y - \tmut(1, X, S)} \times \score(Y, T, R, S, X) \mid T = 1}}} \\
=& \Eb{\frac{TR}{\pr\prns{T=1}\rt(1, X, S)}\prns{Y - \tmut(1, X, S)} \times \score(Y, T, R, S, X)}.
\end{align*}
This means that the part of the influence function corresponding to $\Eb{Y \mid T = 1}$ is 
\begin{align*}
\frac{T}{\pr\prns{T = 1}}\braces{\prns{\tmut(1, X, S) - \Eb{\tmut(1, X, S)\mid T = 1}} + \frac{R}{\rt(1, X, S)}\prns{Y - \tmut(1, X, S)}}. 
\end{align*}
}

{
Now we further derive 
\begin{align*}
 &\frac{\partial}{\partial\gamma}\expect_{\gamma}\bracks{\expect_{\gamma}\bracks{\expect_{\gamma}\bracks{Y \mid T = 0, R = 1, X, S}\mid X, T = 0} \mid T = 1}  \\
=& \frac{\partial}{\partial\gamma}\expect_{\gamma}\big[\mut(0, X) \mid T = 1\big]\vert_{\gamma = 0} + \expect\bigg[\frac{\partial}{\partial\gamma}\expect_{\gamma}\big[\tmut(0, X, S) \mid X, T = 0\big] \vert_{\gamma = 0} \mid T = 1\bigg] \nonumber \\
+& \expect\bigg[\expect\big[\frac{\partial}{\partial \gamma}\expect_{\gamma}[Y \mid T = 0, R = 1, X, S]\vert_{\gamma = 0} \mid X, T = 0\big] \mid T = 1\bigg].
 \end{align*} 
 First, 
 \begin{align*}
 \frac{\partial}{\partial\gamma}\expect_{\gamma}\big[\mut(0, X) \mid T = 1\big]\vert_{\gamma = 0} = \Eb{\frac{T}{\pr\prns{T = 1}} \prns{\mut(0, X) - \Eb{\mut(0, X) \mid T = 1}}\times \score\prns{Y, R, S, T, X}}.
 \end{align*}
 Second, 
 \begin{align*}
 &\expect\bigg[\frac{\partial}{\partial\gamma}\expect_{\gamma}\big[\tmut(0, X, S) \mid X, T = 0\big] \vert_{\gamma = 0} \mid T = 1\bigg] \\
    =& \expect\bigg[\expect\big[\tmut(0, X, S) \times \score(S \mid X, T)\mid X, T = 0\big] \mid T = 1\bigg] \\
    =& \expect\bigg[\expect\big[\prns{\tmut(0, X, S) - \mut(0, X)} \times \score(S \mid X, T)\mid X, T = 0\big] \mid T = 1\bigg] \\
    =& \expect\bracks{\frac{\et(X)}{\pr\prns{T = 1}}\frac{1-T}{1-\et(X)}\prns{\tmut(0, X, S) - \mut(0, X)} \times \score(Y, R, S, T, X)}. 
 \end{align*}
 Third, 
 \begin{align*}
 &\expect\bigg[\expect\big[\frac{\partial}{\partial \gamma}\expect_{\gamma}[Y \mid T = 0, R = 1, X, S]\vert_{\gamma = 0} \mid X, T = 0\big] \mid T = 1\bigg] \\
 =& \expect\bigg[\expect\big[\expect[(Y - \tmut(0, X, S)) \times \score(Y \mid R, S, T, X) \mid T = 0, R = 1, X, S] \mid X, T = 0\big] \mid T = 1\bigg] \\
 =& \expect\bigg[\expect\big[\expect[\frac{R}{\rt(0, X, S)}(Y - \tmut(0, X, S)) \times \score(Y, R, S, T, X) \mid T = 0, X, S] \mid X, T = 0\big] \mid T = 1\bigg] \\
 =& \expect\left[\expect\left[\frac{R}{\rt(0, X, S)}(Y - \tmut(0, X, S)) \times \score(Y, R, S, T, X) \mid X, T = 0 \right] \mid T = 1\right] \\
 =& \expect\left[\frac{\et(X)}{\pr\prns{T = 1}}\frac{1-T}{1-\et(X)}\frac{R}{\rt(0, X, S)}(Y - \tmut(0, X, S)) \times \score(Y, R, S, T, X)\right].
 \end{align*}
Combining the equations above, we have that 
\begin{align*}
&\frac{\partial}{\partial\gamma}\expect_{\gamma}\bracks{Y \mid T = 1}\vert_{\gamma = 0} - \expect\bigg[\expect\big[\frac{\partial}{\partial \gamma}\expect_{\gamma}[Y \mid T = 0, R = 1, X, S]\vert_{\gamma = 0} \mid X, T = 0\big] \mid T = 1\bigg] \\
=& \expect\left[\psi_{\op{ATT}}(W; \delta^*_{\op{ATT}}, \eta^*) \times \score(Y, R, S, T, X)\right].
\end{align*}
Moreover, 
\begin{align*}
\psi_{\op{ATT}}(W; \delta^*_{\op{ATT}}, \eta^*) 
    &=  \frac{\et(X)}{\pr\prns{T=1}}\prns{\mut(1, X) - \mut(0, X) - \delta^*_{\op{ATT}}} \\
    &+ \frac{T - \et(X)}{\pr\prns{T = 1}}\prns{\mut(1, X) - \mut(0, X) - \delta^*_{\op{ATT}}} \\
    &+ \frac{T}{\pr\prns{T = 1}}\prns{\tmut(1, X, S) - \mut(1, X)} \\
    &- \frac{\et(X)}{\pr\prns{T=1}}\frac{1 - T}{1 - \et(X)}(\tmut(0, X, S) - \mut(0 , X)) \\
    &+\frac{TR}{\pr\prns{T=1}\rt(1, X, S)}(Y - \tmut(1, X, S)) \\
    &- \frac{\et(X)}{\pr\prns{T=1}}\frac{(1 - T)R}{(1 - \et(X))\rt(0, X, S)}(Y - \tmut(0, X, S)). 
\end{align*}
It is easy to show that the six terms  in the right hand side above belong to $\Lambda_{X}$, $\Lambda_{T \mid X}$,  $\Lambda_{S \mid T, X}$, $\Lambda_{S \mid T, X}$, $\Lambda_{Y \mid R, T, S, X}$ and $\Lambda_{Y \mid R, T, S, X}$ in the proof for \Cref{thm: efficient-if}, respectively. 
Therefore, $\psi_{\op{ATT}}(W; \delta^*_{\op{ATT}}, \eta^*)$ belongs to the tangent space and is therefore the efficient influence function. 
From this analysis, we can also see that $\psi_{\op{ATT}}(W; \delta^*_{\op{ATT}}, \eta^*)$ is orthogonal to $\Lambda_{R \mid S, T, X}$, so the efficiency bound is invariant to any restriction on the conditional distribution of $R \mid S, T, X$. 
}
\end{proof}

\begin{proof}[Proof for \Cref{thm: efficiency-comparison-att}]
{
In setting I, 
\begin{align*}
\delta^*_{\op{ATT}} = \Eb{\Eb{Y \mid T = 1, R = 1, X}\mid T = 1} - \Eb{\Eb{Y \mid X, T = 0, R = 1} \mid T = 1}
\end{align*}
By standard path differentiability analyses, we can easily show that 
\begin{align*}
&\frac{\partial}{\partial\gamma}\expect_\gamma\bracks{\expect_\gamma\bracks{Y \mid T = 1, R = 1, X}\mid T = 1}\vert_{\gamma=0} \\
=& \Eb{\prns{\frac{T}{\pr\prns{T=1}}\prns{\mut(1, X) - \Eb{\mut(1, X)\mid T = 1}} + \frac{T}{\pr\prns{T=1}}\frac{R}{\rt(1, X)}\prns{Y - \mut(1, X)}} \times \score\prns{Y, R, T, X}}
\end{align*}
Similarly, we have 
\begin{align*}
&\frac{\partial}{\partial\gamma}\expect_\gamma\bracks{\expect_\gamma\bracks{Y \mid T = 0, R = 1, X}\mid T = 1}\vert_{\gamma=0} \\
=& \expect\bigg[\prns{\frac{T}{\pr\prns{T=1}}\prns{\mut(0, X) - \Eb{\mut(0, X)\mid T = 1}} + \frac{\et(X)}{\pr\prns{T=1}}\frac{1-T}{1-\et(X)}\frac{R}{\rt(0, X)}\prns{Y - \mut(0, X)}} \\
&\qquad\qquad\qquad\qquad\qquad\qquad\qquad\qquad\qquad \times \score\prns{Y, R, T, X}\bigg]
\end{align*}
These give the form of the efficient influence function $\psi_{\op{ATT}, \op{I}}$ for setting I. 
According to the proof for \cref{thm: efficiency-comparison}, the efficient influence function $\psi_{\op{ATT}, \op{II}}$ in setting II is identical to $\psi_{\op{ATT}, \op{I}}$.
}

{
According to the proof for \Cref{thm: efficient-if-att}, the efficient influence function in \Cref{thm: efficient-if-att} is invariant to restrictions on the conditional distribution of $R \mid S, T, X$ so the additional \Cref{assump: MAR2} does not change the efficient influence function in setting III, the setting also considered in \Cref{thm: efficient-if-att}. Thus the efficient influence function  in setting III follows from the efficient influence function in \Cref{thm: efficient-if-att} with the additional fact that $\rt(0, X, S) = \rt(0, X)$ under \Cref{assump: MAR2}.  
}

{
The efficient influence function in setting IV follows from \cite{hahn1998role}. 
}
\end{proof}

\begin{proof}[Proof for \Cref{corollary:att-comparison}]
Note that 
\begin{align*}
\psi_{{\op{ATT}}, \op{I}}(W; \delta^*_{\op{ATT}}, \eta^*) &= \psi_{\op{ATT},\op{II}}(W; \delta^*_{\op{ATT}}, \eta^*) = \psi_{\op{ATT},\op{III}}(W; \delta^*_{\op{ATT}}, \eta^*) + \omega_{\op{ATT}, \op{I-III}}(W; \delta^*_{\op{ATT}}, \eta^*), 
\end{align*}
where 
\begin{align*}
\omega_{\op{ATT}, \op{I-III}}(W; \delta^*_{\op{ATT}}, \eta^*) 
  &= \frac{T}{\Prb{T=1}}\prns{\frac{R}{\rt(1, X)} - 1}\prns{\tmut(1, X, S) - \mut(1, X)}
 \\
  &+ \frac{\et(X)}{\Prb{T=1}}\frac{1-T}{1-\et(X)}\prns{1 - \frac{R}{\rt(0, X)}}\prns{\tmut(0, X, S) - \mut(0, X)}.
 \end{align*}
 We can easily show that $\omega_{\op{ATT}, \op{I-III}}(W; \delta^*_{\op{ATT}}, \eta^*)$ and $\psi_{\op{ATT},\op{III}}(W; \delta^*_{\op{ATT}}, \eta^*)$ are uncorrelated based on the facts that $\Eb{\omega_{\op{ATT}, \op{I-III}}(W; \delta^*_{\op{ATT}}, \eta^*)  \mid T, X, S} = 0$ and  that
 \begin{align*}
  \Eb{\frac{TR}{\pr\prns{T = 1}\rt(1, X)}\prns{Y - \tmut(1, X, S)} - \frac{\et(X)}{\pr\prns{T=1}}\frac{(1 - T)R}{(1 - \et(X))\rt(0, X)}(Y - \tmut(0, X, S)) \mid R, T, X, S} = 0.
  \end{align*} 
  It then follows that 
  \begin{align*}
  &V_{\op{ATT}, \op{I}}^* - V^*_{\op{ATT}, \op{III}} 
    = V_{\op{ATT}, \op{II}}^* - V^*_{\op{ATT}, \op{III}} = \Eb{\omega^2_{\op{ATT}, \op{I-III}}(W; \delta^*_{\op{ATT}}, \eta^*) } \\
    = & \Eb{\frac{1-\rt(1,X)}{\Prb{T=1}\rt(1, X)}\var[\tmut(1, X, S(1))\mid X]+ \frac{\et(X)(1-\rt(0, X))}{\Prb{T=1}\prns{1-\et(X)}\rt(0, X)}\var[\tmut(0, X, S(0))\mid X] \mid T = 1}. 
  \end{align*}

  Moreover, 
  \begin{align*}
  \psi_{{\op{ATT}}, \op{III}}(W; \delta^*_{\op{ATT}}, \eta^*) = \psi_{\op{ATT},\op{IV}}(W; \delta^*_{\op{ATT}}, \eta^*) + \omega_{\op{ATT}, \op{III-IV}}(W; \delta^*_{\op{ATT}}, \eta^*),
  \end{align*}
  where 
  \begin{align*}
  \omega_{\op{ATT}, \op{III-IV}}(W; \delta^*_{\op{ATT}}, \eta^*) 
  &= \frac{T}{\Prb{T = 1}}\prns{\frac{R}{\rt(1, X)} - 1}\prns{Y - \tmut(1, X, S)} \\
  &- \frac{\et(X)}{\Prb{T = 1}}\prns{\frac{R}{\rt(0, X)} - 1}\prns{Y - \tmut(0, X, S)} 
  \end{align*}
  Therefore, we hve 
  \begin{align*}
  V_{\op{ATT}, \op{III}}^* - V^*_{\op{ATT}, \op{IV}} = \Eb{\omega^2_{\op{ATT}, \op{III-IV}}(W; \delta^*_{\op{ATT}}, \eta^*) } + 2\op{Cov}\prns{\psi_{\op{ATT},\op{IV}}(W; \delta^*_{\op{ATT}}, \eta^*), \omega_{\op{ATT}, \op{III-IV}}(W; \delta^*_{\op{ATT}}, \eta^*) }.
  \end{align*}
  We can easily show that 
  \begin{align*}
  \Eb{\omega^2_{\op{ATT}, \op{III-IV}}(W; \delta^*_{\op{ATT}}, \eta^*) } &= \expect\bigg[\frac{T}{(\Prb{T=1})^2}\prns{\frac{R}{\rt(1, X)} - 1}^2\prns{Y - \tmut(1, X,S)}^2\bigg] \\
    &+ \expect\bigg[\frac{(\et(X))^2}{(\Prb{T=1})^2}\frac{1-T}{(1-\et(X))^2}\prns{\frac{R}{\rt(0, X)} - 1}^2\prns{Y - \tmut(0, X,S)}^2\bigg],
  \end{align*}
  and meanwhile 
  \begin{align*}
  &2\op{Cov}\prns{\psi_{\op{ATT},\op{IV}}(W; \delta^*_{\op{ATT}}, \eta^*), \omega_{\op{ATT}, \op{III-IV}}(W; \delta^*_{\op{ATT}}, \eta^*) } \\
  =& 2\expect\bigg[\frac{T}{(\Prb{T=1})^2}\prns{\frac{R}{\rt(1, X)} - 1}\prns{Y - \tmut(1, X,S)}^2\bigg] \\
  +& 2\expect\bigg[\frac{(\et(X))^2}{(\Prb{T=1})^2}\frac{1-T}{(1-\et(X))^2}\prns{\frac{R}{\rt(0, X)} - 1}\prns{Y - \tmut(0, X,S)}^2\bigg].
  \end{align*}
  We can combine them and get 
  \begin{align*}
  &V_{\op{ATT}, \op{III}}^* - V^*_{\op{ATT}, \op{IV}}= \expect\bigg[\frac{T}{(\Prb{T=1})^2}\prns{\frac{R}{(\rt(1, X))^2} - 1}\prns{Y - \tmut(1, X,S)}^2\bigg] \\
    &+ \expect\bigg[\frac{(\et(X))^2}{(\Prb{T=1})^2}\frac{1-T}{(1-\et(X))^2}\prns{\frac{R}{(\rt(0, X))^2} - 1}\prns{Y - \tmut(0, X,S)}^2\bigg] \\
    &= \frac{1}{\Prb{T=1}}\Eb{\frac{1-\rt(1, X)}{\rt(1, X)}\var[Y(1) \mid X, S(1)] + \frac{\et(X)\prns{1-\rt(0, X)}}{\prns{1-\et(X)}\rt(0, X)}\var[Y(0) \mid X, S(0)] \mid T = 1}.
  \end{align*}
\end{proof}

\section{Additional Numerical Results}\label{sec: numeric-more}
In this section we provide additional results for the experiment in \cref{sec: empirics}.

\subsection{Real Data Experiment}
 In \cref{sec: numerics-real} \cref{fig: error-river-rf}  we presented the results for Riverside county with nuisances estimated by random forests. Here, in \cref{fig: error-river-extra} we present the results for Riverside county with other nuisance estimators: gradient boosting in \cref{fig: error-river-gb} and lasso in \cref{fig: error-river-lasso}. We also present results for Los Angeles county in \cref{fig: error-la-rf,fig: error-la-gb,fig: error-la-lasso} and for Los Angeles county in \cref{fig: error-sd-rf,fig: error-sd-gb,fig: error-sd-lasso}. For both counties, \cref{fig: error-lasd-rf} presents the results with nuisances fitted using random forests, \cref{fig: error-la-gb} with gradient boosting, and \cref{fig: error-la-lasso} with lasso.

\begin{figure}[h]
\centering 
\begin{subfigure}[b]{\textwidth}
         \centering
         \includegraphics[width=0.9\textwidth]{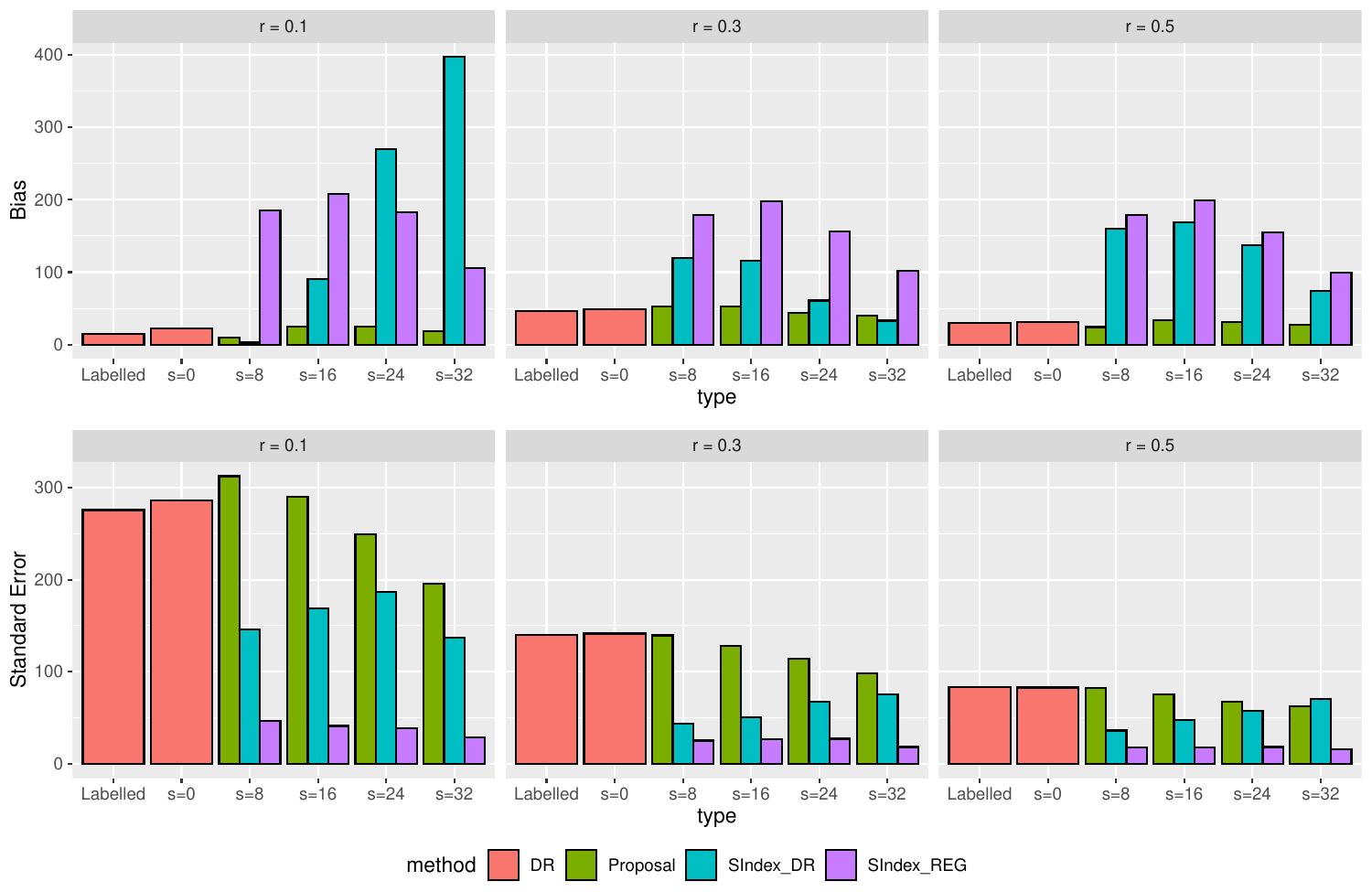}
         \caption{Gradient boosting nuisance estimation}\label{fig: error-river-gb}
 \end{subfigure}
 \begin{subfigure}[b]{\textwidth}
         \centering
         \includegraphics[width=0.9\textwidth]{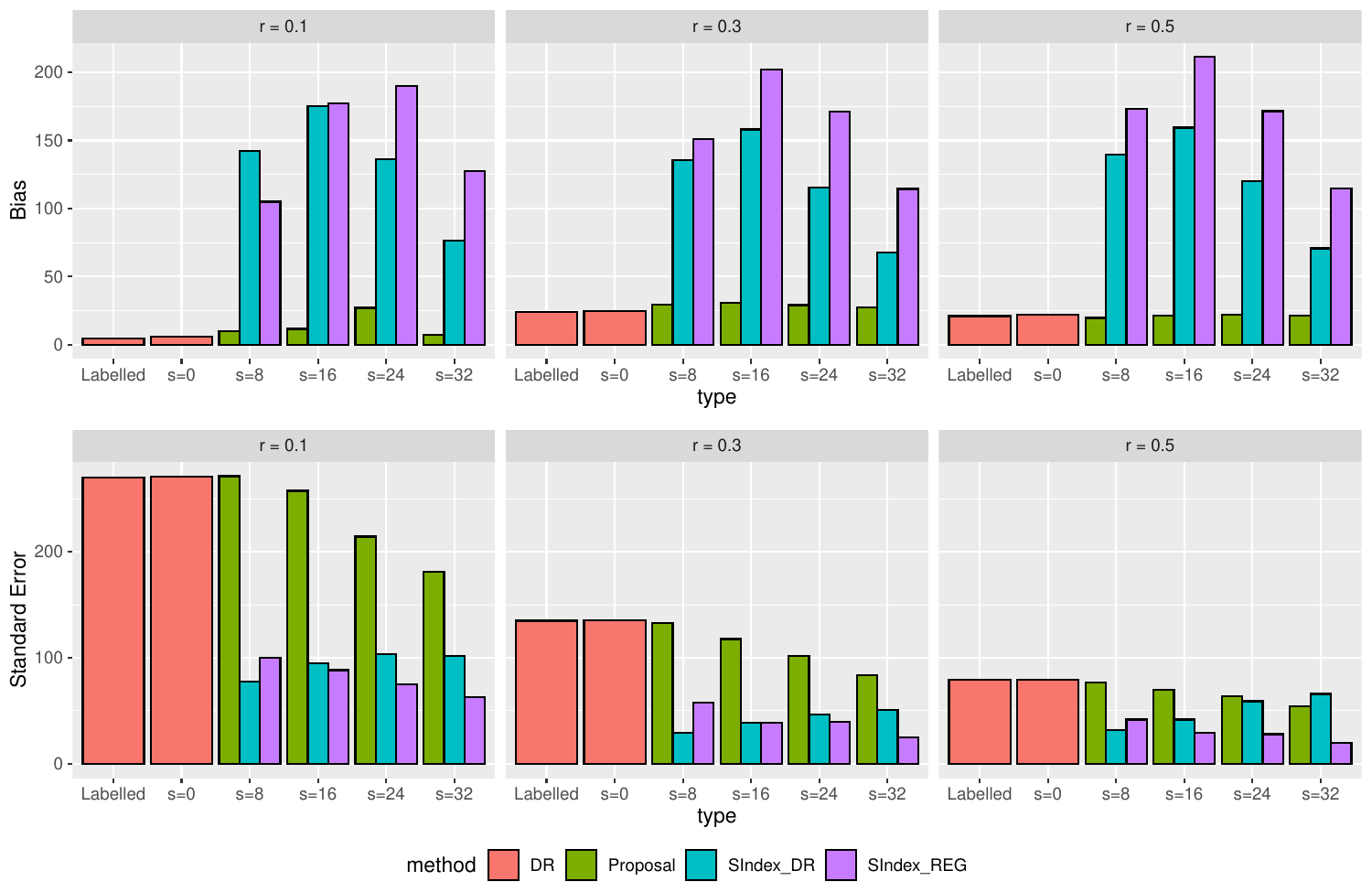}
         \caption{LASSO nuisance estimation}\label{fig: error-river-lasso}
 \end{subfigure}
\caption{Bias and standard error of different estimators over $120$ repetitions of experiments based on Riverside data.}\label{fig: error-river-extra}
\end{figure}

\begin{figure}[h]
\centering 
\begin{subfigure}[b]{\textwidth}
         \centering
         \includegraphics[width=0.9\textwidth]{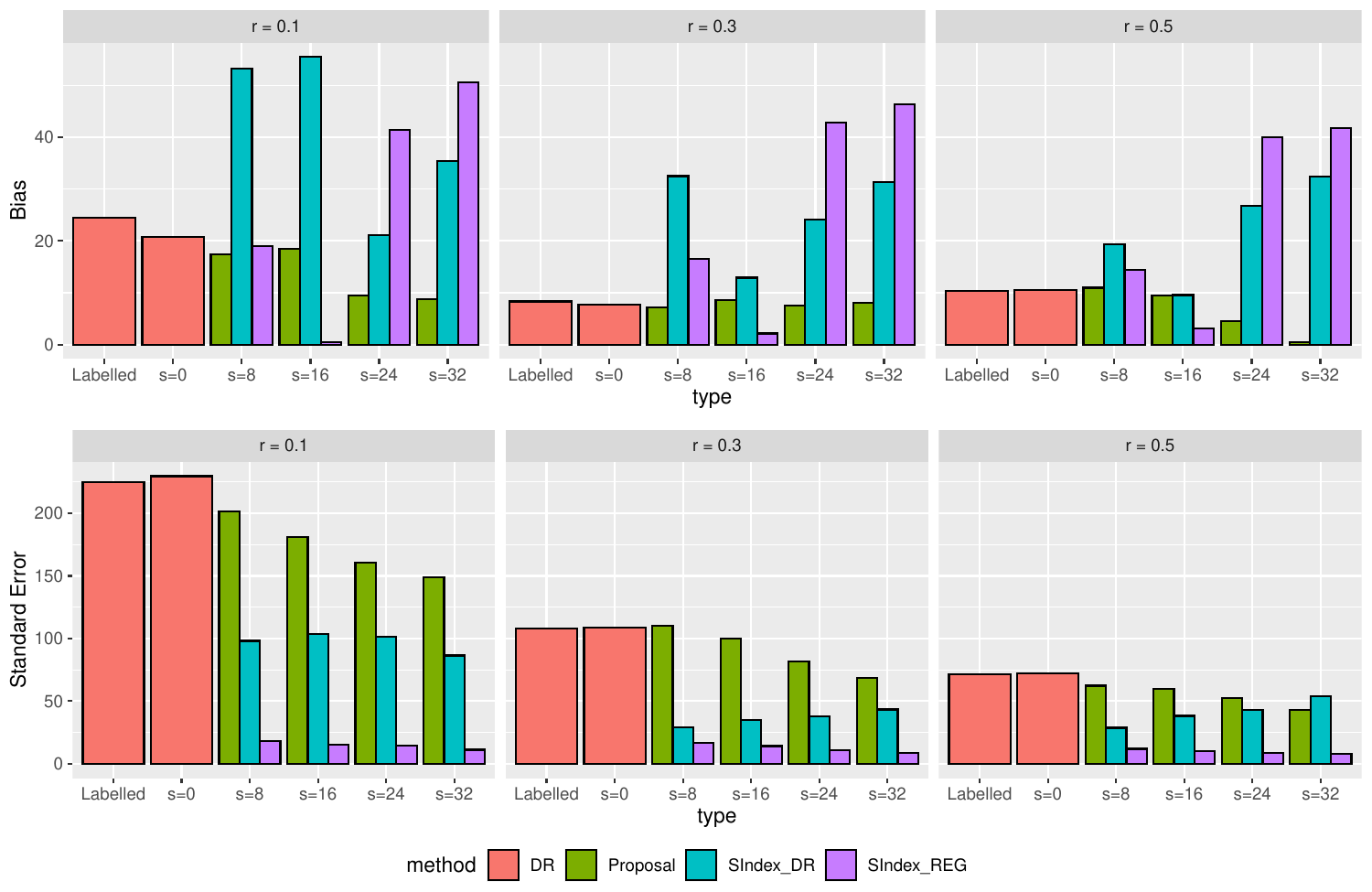}
         \caption{Los Angeles data.}\label{fig: error-la-rf}
 \end{subfigure}
 \begin{subfigure}[b]{\textwidth}
         \centering
         \includegraphics[width=0.9\textwidth]{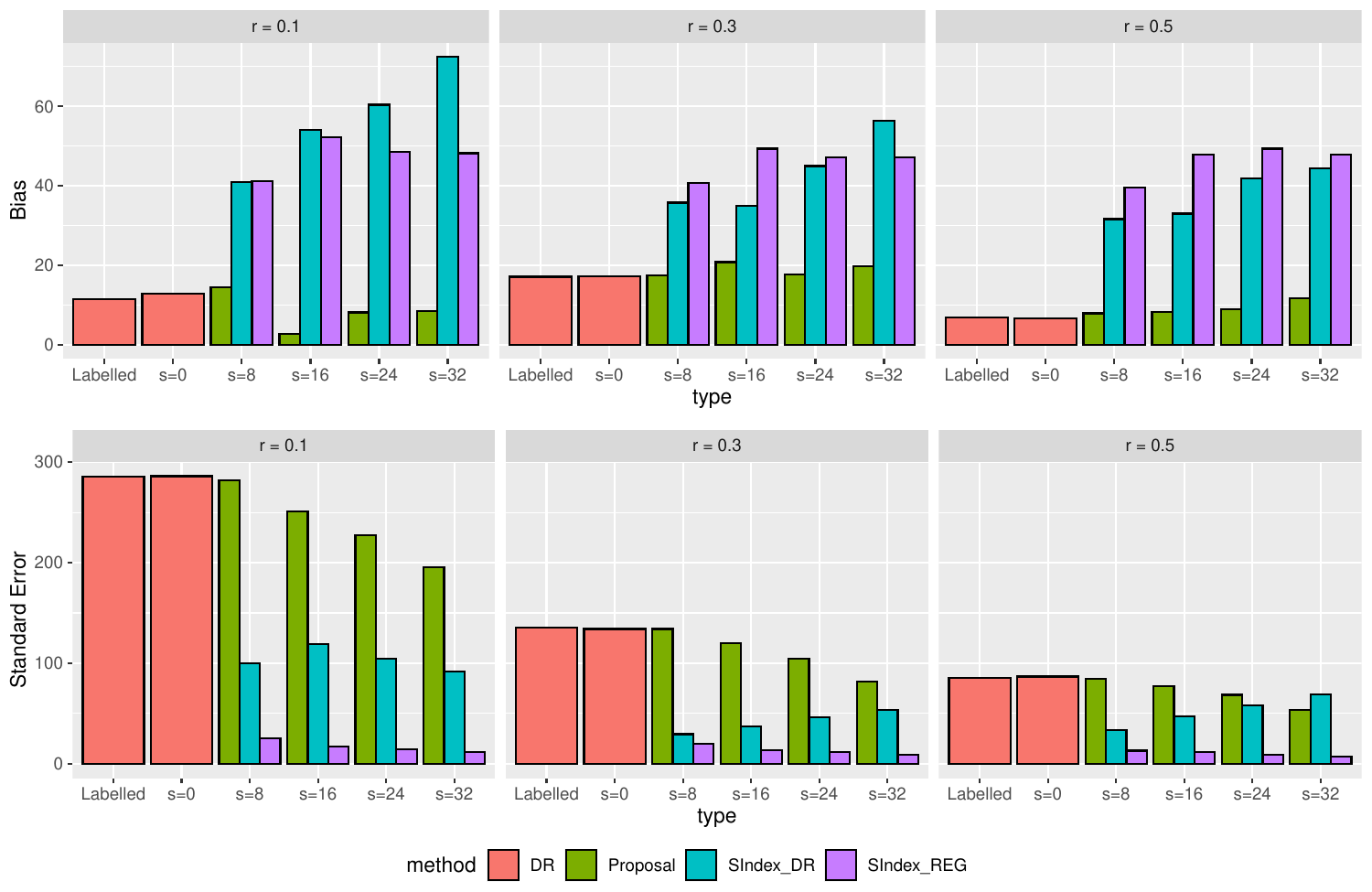}
         \caption{San Diego data.}\label{fig: error-sd-rf}
 \end{subfigure}
\caption{Bias and standard error of different estimators over $120$ repetitions of experiments based on Los Angeles data and San Diego data respectively. Nuisances are estimated by random forests.}\label{fig: error-lasd-rf}
\end{figure}

\begin{figure}[h]
\centering 
\begin{subfigure}[b]{\textwidth}
         \centering
         \includegraphics[width=0.9\textwidth]{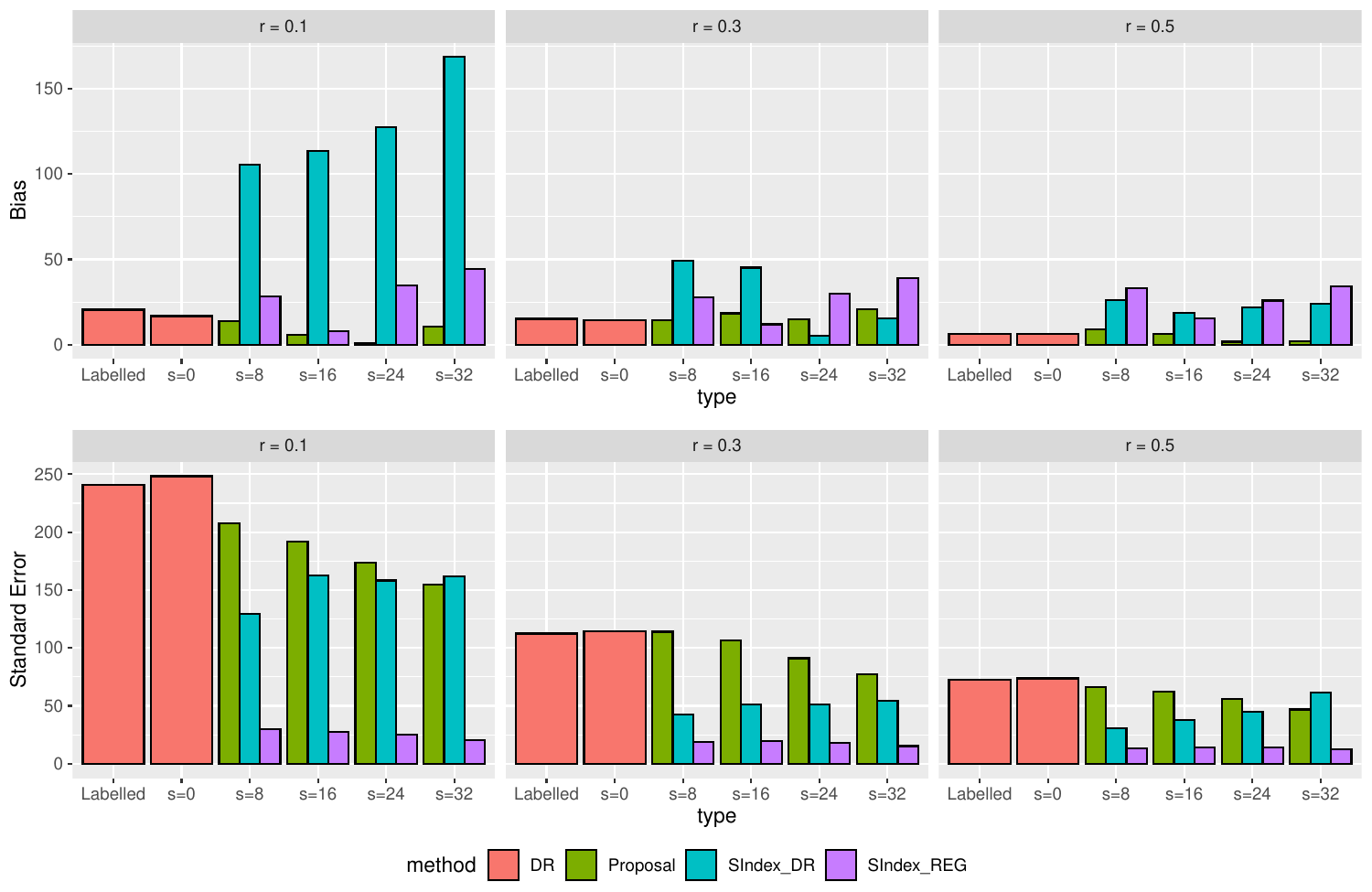}
         \caption{Los Angeles data.}\label{fig: error-la-gb}
 \end{subfigure}
 \begin{subfigure}[b]{\textwidth}
         \centering
         \includegraphics[width=0.9\textwidth]{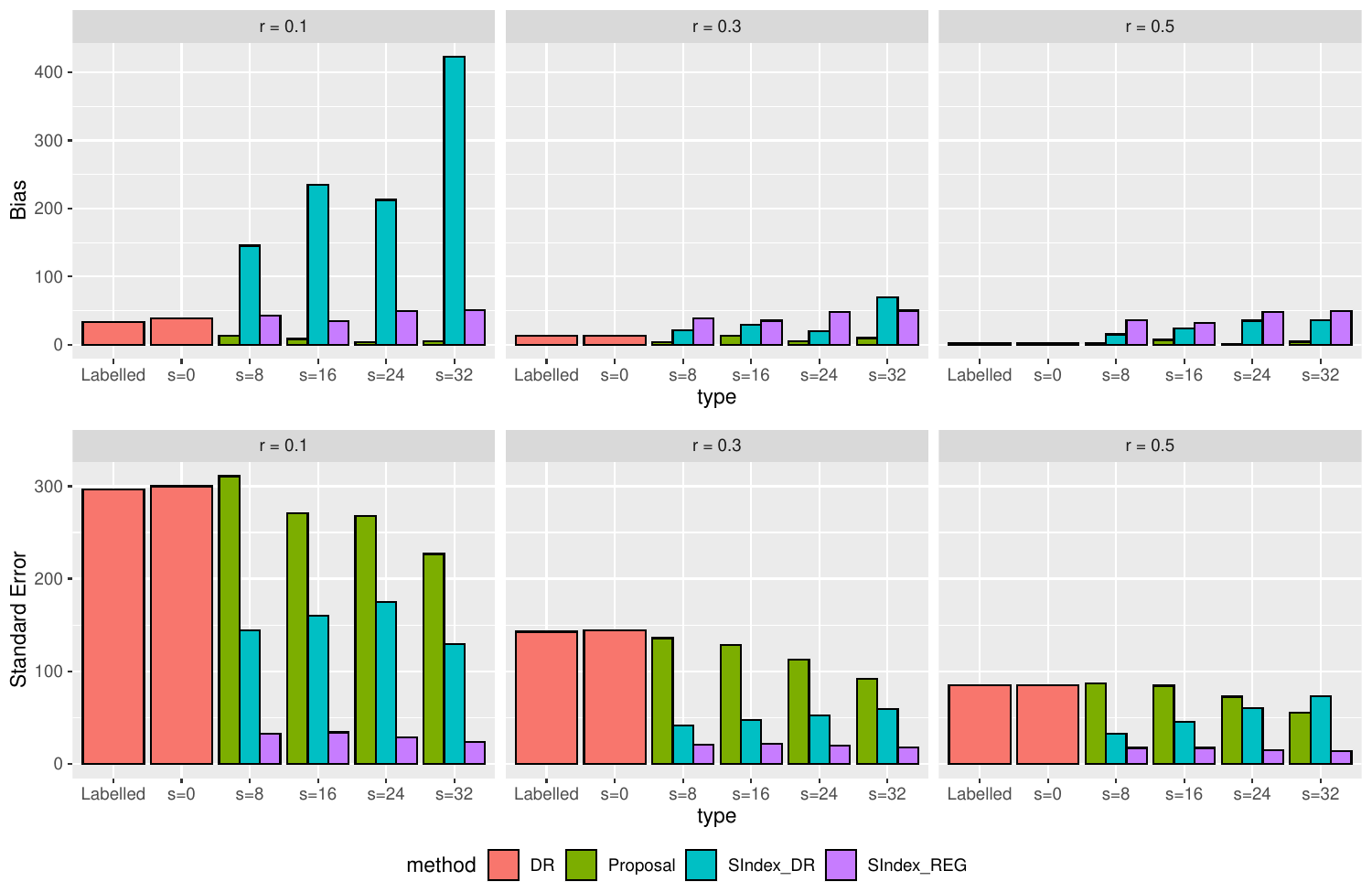}
         \caption{San Diego data.}\label{fig: error-sd-gb}
 \end{subfigure}
\caption{Bias and standard error of different estimators over $120$ repetitions of experiments based on Los Angeles data and San Diego data respectively. Nuisances are estimated by gradient boosting.}\label{fig: error-lasd-gb}
\end{figure}

\begin{figure}[h]
\centering 
\begin{subfigure}[b]{\textwidth}
         \centering
         \includegraphics[width=0.9\textwidth]{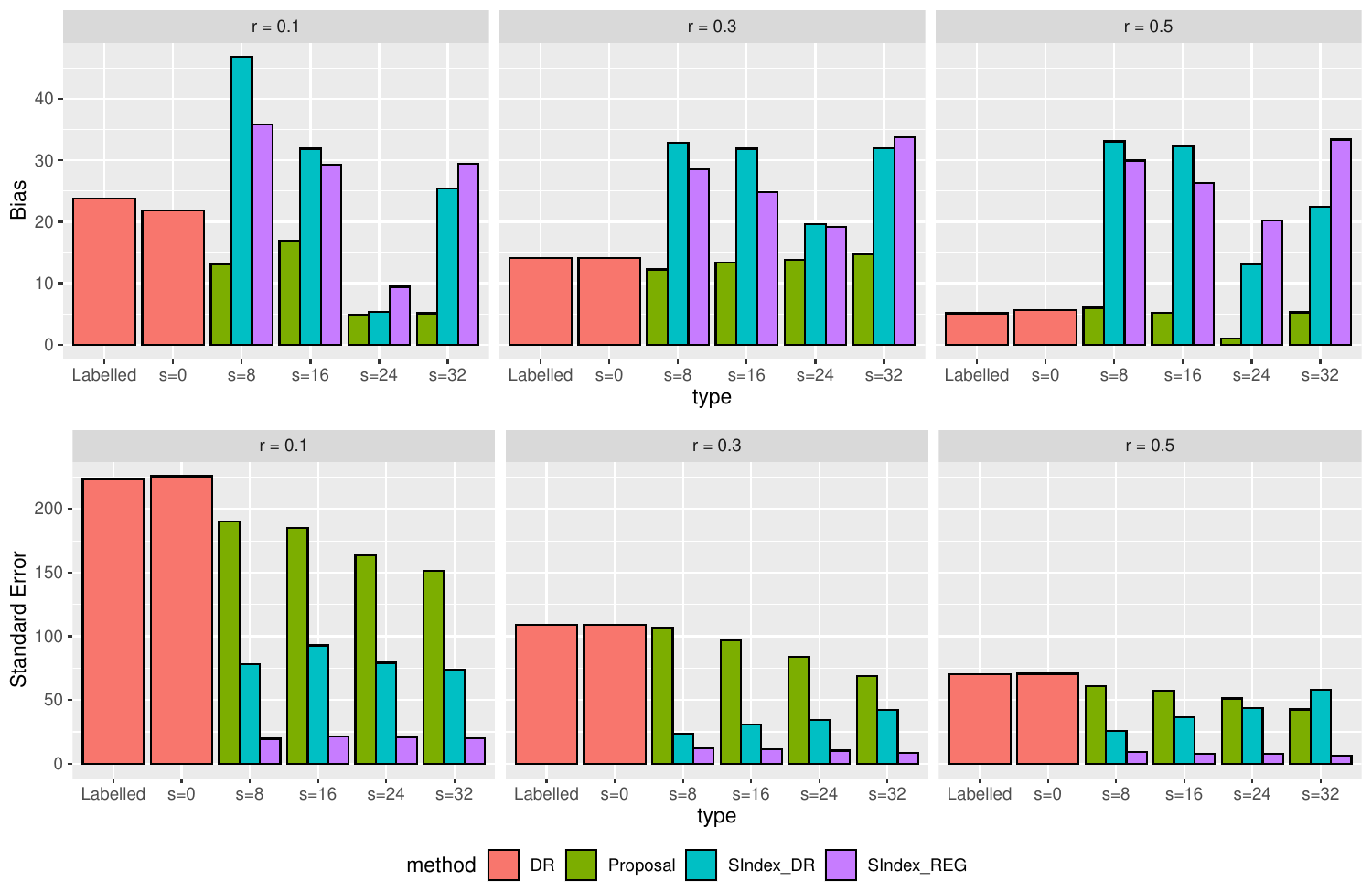}
         \caption{Los Angeles data.}\label{fig: error-la-lasso}
 \end{subfigure}
 \begin{subfigure}[b]{\textwidth}
         \centering
         \includegraphics[width=0.9\textwidth]{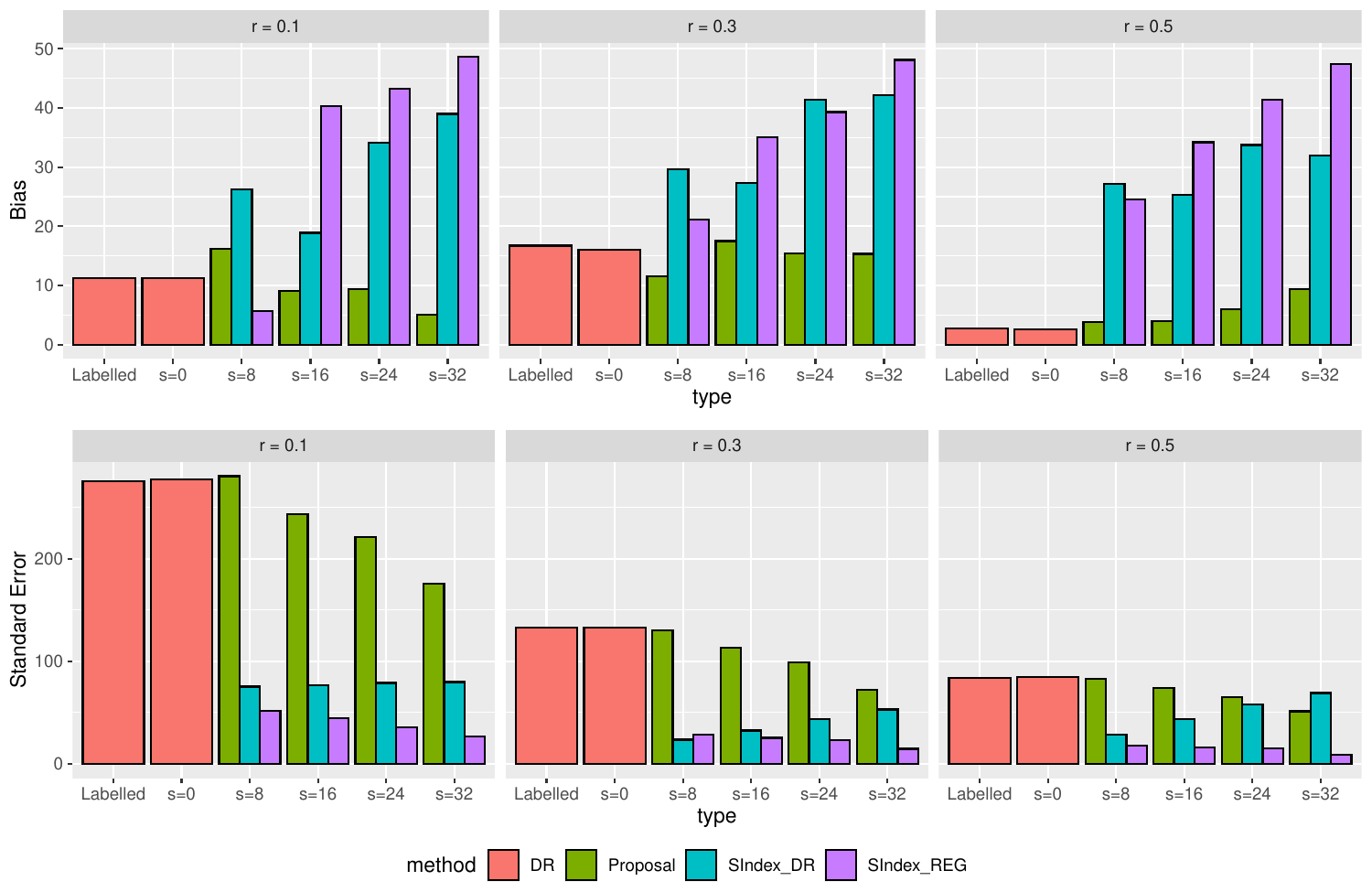}
         \caption{San Diego data.}\label{fig: error-sd-lasso}
 \end{subfigure}
\caption{Bias and standard error of different estimators over $120$ repetitions of experiments based on Los Angeles data and San Diego data respectively. Nuisances are estimated by LASSO.}\label{fig: error-lasd-lasso}
\end{figure}

\subsection{Simulation Experiment}

{In \Cref{sec: simulation} \Cref{table: ate-0.25}, we show the results of ATE estimation using the estimator $\hat\delta$ in \Cref{def: estimator} when the proportion of labeled data is $\pi_N = N^{-1/4}$. Here, in \Cref{table: ate-0.33} and \Cref{table: ate-0.5}, we also show the estimation results for $\pi_N = N^{-1/3}$ and $\pi_N = 2.5 N^{-1/2}$ respectively. These two settings correspond to smaller labeled data, so all methods tend to have worse performance (higher standard deviation, wider confidence intervals and lower confidence interval coverage). But  the  qualitative conclusions  in these two setting remain the same as those in \Cref{sec: simulation}}. 

\begin{table}
\centering 
\begin{tabular}{cc|cccccc}
\toprule 
\multirow{2}{*}{Measure}  & \multirow{2}{*}{Nuisance Est.} & \multicolumn{6}{c}{ $N$}  \\
& & $2000$ & $4000$ & $8000$ & $16000$ & $32000$ & $64000$ \\
\midrule 
\multirow{3}{*}{Bias} & Oracle & 0.0028 &  0.0123 & 0.0035 & 0.0069 & 0.0134 & 0.0025 \\
 & Parametric & 0.0237 &  0.0198 & 0.0081 & 0.0078 & 0.0143 & 0.0033 \\
& GB & 0.0335 & 0.0166 & 0.0195 & 0.0113 & 0.0203 & 0.0020 \\
\midrule 
\multirow{3}{*}{Standard  Deviation} &Oracle & 0.4148 & 0.3427 & 0.2672 & 0.2250 & 0.1825 & 0.1442 \\
 & Parametric & 0.5603 & 0.3937 & 0.2940 & 0.2395 & 0.1892 & 0.1480 \\
& GB &  1.0269 & 0.6488 & 0.4201 & 0.2946 & 0.2127 & 0.1618 \\
\midrule 
\multirow{3}{*}{CI Length} & Oracle & 1.5272 & 1.2711 & 1.0349 & 0.8465 & 0.6834 & 0.5540 \\
 & Parametric & 1.9737 & 1.4517 & 1.1290 & 0.8891 & 0.7044 & 0.5636 \\
& GB & 3.5349 & 2.3158 & 1.5425 & 1.0932 & 0.7828 & 0.5964 \\
\midrule 
\multirow{3}{*}{CI Coverage} & Oracle &  0.967 & 0.953 & 0.962 & 0.949 & 0.946 & 0.940 \\
& Parametric & 0.933 & 0.941 & 0.956 & 0.936 & 0.943 &  0.933 \\
& GB & 0.930 & 0.936 & 0.946 & 0.943 & 0.933 & 0.934 \\
\bottomrule
\end{tabular}
\caption{Results of ATE estimation  with true nuisance values (oracle) or nuisances estimated by parametric models (Parametric) and gradient boosting (GB) when $\pi_N = N^{-1/3}$.}\label{table: ate-0.33}
\end{table}

\begin{table}
\centering 
\begin{tabular}{cc|cccccc}
\toprule 
\multirow{2}{*}{Measure}  & \multirow{2}{*}{Nuisance Est.} & \multicolumn{6}{c}{ $N$}  \\
& & $2000$ & $4000$ & $8000$ & $16000$ & $32000$ & $64000$ \\
\midrule 
\multirow{3}{*}{Bias} & Oracle & 0.0203 & 0.0146 & 0.0095 & 0.0048 & 0.0052 & 0.0099 \\
 & Parametric & 0.0027 &  0.0185 & 0.0091 & 0.0082 & 0.0077 & 0.0108 \\
& GB & 0.0142 & 0.0120 & 0.0258 & 0.0098 & 0.0066 & 0.0041 \\
\midrule 
\multirow{3}{*}{Standard  Deviation} &Oracle & 0.4814 & 0.4097 & 0.3617 & 0.3080 & 0.2696 & 0.2246 \\
 & Parametric & 0.7673 & 0.5487 & 0.4310 & 0.3412 & 0.2895 & 0.2401 \\
& GB &  1.4401 & 1.0128 & 0.6461 & 0.5079 & 0.3825 & 0.2822 \\
\midrule 
\multirow{3}{*}{CI Length} & Oracle & 1.7612 & 1.5553 & 1.3684 & 1.1615 & 1.0141 & 0.8705 \\
 & Parametric & 2.5709 & 1.9734 & 1.6019 & 1.2799 & 1.0802 & 0.9072 \\
& GB & 4.7977 & 3.3593 & 2.4116 & 1.7775 & 1.3491 & 1.0334 \\
\midrule 
\multirow{3}{*}{CI Coverage} & Oracle &  0.965 & 0.961 & 0.967 & 0.950 & 0.943 & 0.959 \\
& Parametric & 0.916 & 0.939 & 0.936 & 0.946 & 0.940 &  0.939 \\
& GB & 0.922 & 0.931 & 0.941 & 0.947 & 0.935 & 0.931 \\
\bottomrule
\end{tabular}
\caption{Results of ATE estimation  with true nuisance values (oracle) or nuisances estimated by parametric models (Parametric) and gradient boosting (GB) when $\pi_N = 2.5 N^{-1/2}$.}\label{table: ate-0.5}
\end{table}

\end{document}